%% file: paper.tex
\documentclass[10pt]{article}

\usepackage[a4paper,margin=1in]{geometry}

\usepackage[utf8]{inputenc}
\usepackage[T1]{fontenc}

\usepackage[table]{xcolor}

\usepackage{graphicx}
\usepackage{booktabs}
\usepackage{multirow}
\usepackage{tabularx}
\usepackage{makecell}
\usepackage{bbding}
\usepackage{algorithm}
\usepackage{algorithmic}
\usepackage{amsmath}
\usepackage{amsthm}
\usepackage{amssymb}
\usepackage{amsfonts}
\usepackage{fancybox}
\usepackage{enumitem}
\usepackage{nicefrac}
\usepackage{microtype}
\usepackage{xcolor}
\usepackage{subcaption}

\usepackage[authoryear,round]{natbib}

\usepackage{url}
\usepackage[colorlinks=true,linkcolor=blue,citecolor=blue,urlcolor=blue]{hyperref}

\definecolor{DownGreen}{RGB}{34,139,34}
\definecolor{UpOrange}{RGB}{210,105,30}
\definecolor{NeutralGray}{RGB}{120,120,120}
\definecolor{rowblue}{RGB}{235,242,250}
\definecolor{customgray}{gray}{0.9}

\definecolor{MuonAccentBg}{RGB}{236,242,248}

\newcommand{\method}{\textbf{Muse}}

\newtheorem{lemma}{Lemma}[section]
\newtheorem{assumption}{Assumption}[section]
\newtheorem{theorem}{Theorem}[section]
\newtheorem{proposition}[theorem]{Proposition}
\newtheorem{corollary}[theorem]{Corollary}
\newtheorem{remark}{Remark}[section]

\title{Muse: Representation Geometry of Muon Beyond Normalized Momentum}

\author{
Da Chang$^{1,2,6}$\thanks{These authors contributed equally to this work.},
Qiankun Shi$^{1,3}$\footnotemark[1],
Lvgang Zhang$^{1,4}$\footnotemark[1], \\
Di He$^{1,2,6}$, Yaoshuai Ma$^{1,4}$,
Ganzhao Yuan$^{5}$\thanks{Corresponding authors: liuyx@pcl.ac.cn and yuanganzhao@foxmail.com.},
Yongxiang Liu$^{1}$\footnotemark[2] \\[2mm]
{\small $^{1}$Pengcheng Laboratory} \\
{\small $^{2}$Shenzhen Institute of Advanced Technology, Chinese Academy of Sciences}\\
{\small $^{3}$Sun Yat-sen University}
{\small $^{4}$Southern University of Science and Technology} \\
{\small $^{5}$Shenzhen University of Advanced Technology}\\
{\small $^{6}$University of Chinese Academy of Sciences} \\[1mm]
{\small Emails: \texttt{changda24@mails.ucas.ac.cn}}
}

\date{}
\begin{document}

\maketitle

\begin{abstract}
Muon-style optimizers apply a polar map to matrix momentum, but their updates also depend on the representation of each parameter block before orthogonalization. We study this representation choice as a form of optimizer geometry and introduce {\method}, a family of Muon-style optimizers that shares the same momentum rule and Newton--Schulz backend across native, nearest-square, skinny, and vector representations. Each Frobenius-isometric representation induces a distinct polar steepest-descent geometry, in which the shorter matrix dimension determines the number of supported singular channels, the pullback scaling, and the constants in stochastic nonconvex convergence bounds. In a teacher--student model, curvature collapse and an isotropic Marchenko--Pastur spectral profile connect early-stage dissipation to the represented nuclear-to-squared-Frobenius norm ratio. Pretraining experiments on LLaMA2-130M and LLaMA2-600M, together with fixed-momentum diagnostics, show that balanced non-native representations can match the performance of the native representation, whereas reducing the shorter dimension weakens the scaling and singular-channel support, leading to behavior that increasingly resembles normalized momentum.
\end{abstract}


\input{1_introduction}
\input{2_method_and_theory}
\input{4_experiment}

\section{Discussion}
\label{sec:conclusion}

This paper views the matrix representation supplied to Muon's polar map as a genuine optimizer choice. Each fixed Frobenius-isometric representation defines a distinct steepest-descent geometry under spectral and nuclear norms, and the temporary short side enters the rank, norm scaling, stationarity constants, and the curvature-collapse mechanism for early dissipation. Diagnostics on LLaMA2-130M and LLaMA2-600M test this representation axis under standard Muon scaling. Balanced non-native matrixizations can closely track the native layout, whereas reducing the short side changes nuclear support, pulled-back directions, and calibration. This also gives a geometric explanation for the observed Muon--nSGDM separation: the matrix polar step preserves multi-channel nuclear support, whereas the vector endpoint collapses it. Thus, the native layout used by Muon is one admissible geometry rather than a canonical necessity: changing the represented matrix changes the update geometry seen by the parameter block.

\section{Limitations}
\label{sec:limitations}
Representation-dependent dissipation is analyzed only locally. We work with a teacher--student model in which ReLU activation masks are fixed, optimization trajectories remain within a single empirical cell, and the curvature window collapses in the large-class and large-sample limits. In this regime, the nuclear-to-Frobenius ratio of the represented error explains early dissipation for fixed Frobenius-isometric matrixizations. The Marchenko--Pastur profile is used only as an isotropic spectral baseline for evaluating \(\|T(\mathbf E_0)\|_*/\|T(\mathbf E_0)\|_F^2\). The optimization result depends on this represented spectral quantity, so it can be applied beyond the MP model whenever the quantity is available.

In general nonconvex training, the block-Gershgorin endpoint \(\beta_{1,C,N}\), which describes the local fixed-mask curvature window, may become negative due to activation-pattern changes, cross-layer coupling, and off-block curvature. In this case, the scalar Frobenius-quadratic reduction no longer applies directly. A complete analysis of representation effects must also account for negative curvature, adaptive representation selection, and interactions with curvature preconditioners such as Shampoo, K-FAC, and SOAP. We leave this joint treatment to future work.

\bibliographystyle{plainnat}
\bibliography{manuscript}

\newpage
\appendix

\onecolumn
\section*{\LARGE Appendix}
\input{5_related_work}
\input{appendix_analysis}
\input{6_appendix}

\end{document}

%% file: 1_introduction.tex
\section{Introduction}
\label{sec:intro}

Muon-style optimizers replace coordinate-wise normalization with a polar update on matrix-valued momentum and have recently become competitive alternatives to AdamW in large language model training~\citep{jordan6muon,bernstein2024old,chen2025muon,amsel2025polar,lau2025polargrad}.
The polar map is evaluated after choosing a matrix representation of each parameter block.
For Euclidean gradient descent this choice is only a relabeling, but for Muon it is part of the optimizer geometry: if \(S\) denotes the native matrixization and \(T\) denotes another Frobenius-isometric matrixization, then typically
\[T^{-1}\operatorname{Orth}(T(\mathbf Z))\neq S^{-1}\operatorname{Orth}(S(\mathbf Z)).
\]
Thus Muon raises a representation question: \textbf{does its behavior rely on the layer's native row--column layout, or on a broader polar geometry induced by the representation supplied to the polar backend?}

This question is coupled to the temporary short side of the represented matrix.
After a fixed row-major vectorization of a \(d\)-dimensional parameter block, any divisor \(p\mid d\) with \(p\le d/p\) defines a Frobenius-isometric flat representation \(T_p:\mathbb R^d\to\mathbb R^{p\times (d/p)}\).
The short side \(p\) bounds the singular channels available to the polar update and the attainable nuclear support value.
Short-side reduction therefore changes the rank and norm constants governing the represented update.

We introduce \method{} (\textbf{MU}on under \textbf{S}tructured r\textbf{E}presentations) as a representation-indexed view of Muon that makes this design axis explicit.
Given a matrix block \(\mathbf G\in\mathbb R^{m\times n}\), \(d=mn\), and native short side \(s=\min(m,n)\), Muse applies a common momentum rule and Newton--Schulz polar backend to alternative fixed matrixizations of the same block.
The native representation is the canonical Muon reference.
The nearest-square flat representation tests whether a non-native representation with a large short side can match the native layout.
The skinny-\(\theta\) representations, with \(\theta\in\{1/8,1/4,1/2,3/4\}\), reduce the temporary short side and form a short-side ladder.
The vector endpoint collapses the polar map to normalized SGD with momentum (nSGDM)~\citep{cutkosky2020momentum,zhao2024stochastic}.
The family isolates representation as a controlled optimizer parameter.

Figure~\ref{fig:intro_representation_axis} summarizes this representation axis: matrix representations provide greater nuclear support than the vector endpoint, the MP shape factor is consistent with fixed LLaMA2 momentum measurements, and the selected validation losses exhibit the same ordering.

\begin{figure*}[!t]
\centering
\begin{tabular}{@{}c@{\hspace{0.3em}}c@{\hspace{0.3em}}c@{}}
\includegraphics[width=0.32\textwidth]{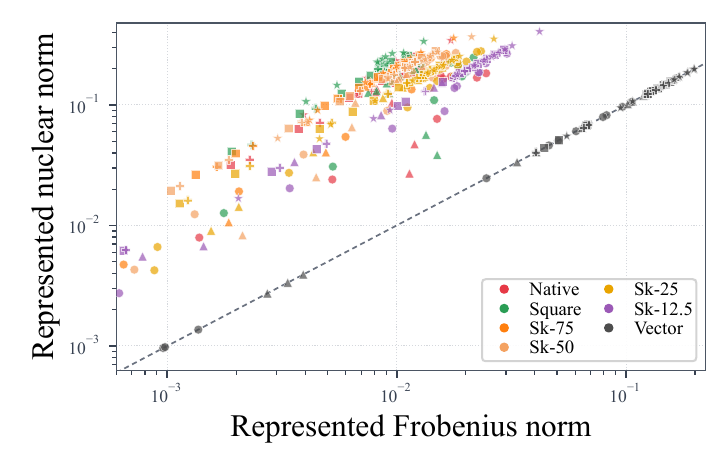} &
\includegraphics[width=0.32\textwidth]{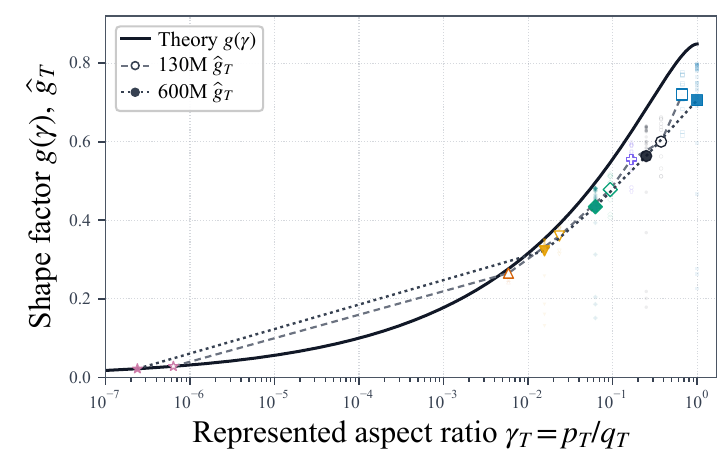} &
\includegraphics[width=0.32\textwidth]{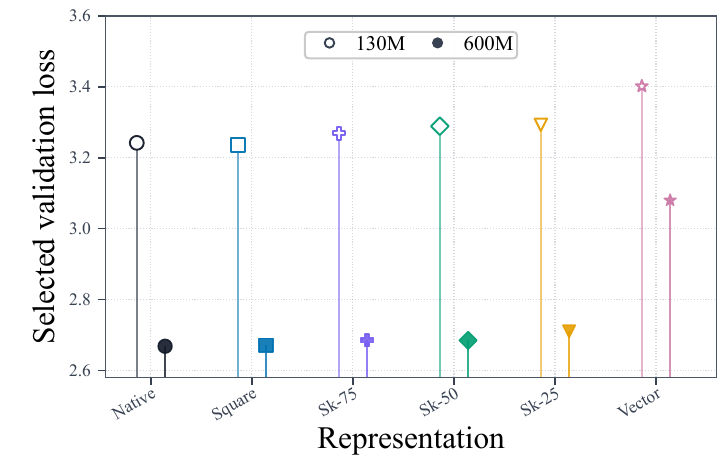} \\[-0.2em]
{\footnotesize (a) Represented norm geometry.} &
{\footnotesize (b) Aspect-ratio shape factor.} &
{\footnotesize (c) Best validation loss.}
\end{tabular}
\vspace*{-0.4em}
\caption{\textbf{Matrix representations create a nuclear-support gap relative to normalized SGD with momentum (nSGDM)~\citep{cutkosky2020momentum,zhao2024stochastic}.}
\textbf{(a)} For saved LLaMA2 momentum blocks, Vector (nSGDM) lies on the Frobenius--nuclear identity line, whereas matrix representations exhibit greater nuclear support at comparable Frobenius scales.
\textbf{(b)} Both the MP shape factor \(g(\gamma)=\beta(\gamma)\gamma^{1/4}\) and its empirical counterpart \(\widehat g_T:=d^{-1/4}\|T(\mathbf M)\|_*/\|T(\mathbf M)\|_F\), where \(\mathbf M\) is a saved momentum block with \(d\) entries, increase with the represented aspect ratio for LLaMA2 MLP-down/W2 blocks at the 130M and 600M scales. Small background points correspond to the actual \(\gamma_T\) of each target block, whereas connected markers indicate the corresponding weighted averages.
\textbf{(c)} Across the LLaMA2-130M and LLaMA2-600M sweeps, the reported validation losses are averages over three random seeds. Native (Muon) and Square achieve the lowest losses within the polar family, whereas representations with reduced short sides progressively approach Vector (nSGDM).}
\vspace*{-0.8em}
\label{fig:intro_representation_axis}
\end{figure*}

This representation-based perspective isolates a Muon-specific source of variation from coordinate-wise adaptivity and matrix-level preconditioning~\citep{duchi2011adaptive,kingma2014adam,ilya2019adamw,martens2015optimizing,gupta2018shampoo,vyas2024soap}. Recent analyses of Muon characterize polar updates through spectral-norm descent, norm-constrained linear minimization, spectral flattening, inexact polar computation, and nonconvex convergence guarantees. However, most existing variants retain the native matrix representation~\citep{chen2025muon,pethick2025training,sfyraki2025lions,lau2025polargrad,nguyen2026spectral,shulgin2025beyond,shen2025convergence,chang2025convergence,kim2026convergence,nagashima2026improved,shumaylov2026muon,zhang2026namo,si2025adamuon,Zhang2025AdaGradMM,ahn2025dion,khaled2025muonbp,gupta2025effective}. We study a complementary question: \textbf{how structured matrixization shapes the polar update of Muon.}

Our analysis first establishes the geometry induced by admissible representations and then isolates the underlying performance mechanism through curvature-collapse dissipation.
Section~\ref{subsec:fixed-representation-polar} proves that every fixed structured matrixization induces a polar steepest-descent direction under the spectral/nuclear-norm geometry used in recent analyses of Muon~\citep{bernstein2024old,chen2025muon,pethick2025training,sfyraki2025lions}:
\[
\langle \mathbf G,\mathbf D_T(\mathbf G)\rangle_F=\|T(\mathbf G)\|_*,
\qquad
-\mathbf D_T(\mathbf G)\in\arg\min_{\|T(\Delta\mathbf X)\|_2\le1}
\langle \mathbf G,\Delta\mathbf X\rangle_F .
\]
The analysis also establishes short-side rank and norm bounds that distinguish matrix-valued polar updates from the vector endpoint. Section~\ref{subsec:nonconvex-convergence-nonselective} provides a stochastic stationarity guarantee for fixed structured matrixizations, whereas Section~\ref{subsec:controlled-large-c-mp-polar} derives an aspect-ratio-dependent ordering from curvature collapse and an MP spectral profile. We then evaluate LLaMA2-130M and LLaMA2-600M training to determine whether the same representation-dependent axis remains observable in Transformer pretraining.

Our contributions are threefold.
\textit{(i)} \textbf{Representation-indexed polar geometry.} We formulate native, nearest-square, skinny, and vector matrixizations as fixed Frobenius-isometric representations of the same parameter block, and establish that the representation provided to the polar map is an integral component of the geometry of Muon.
\textit{(ii)} \textbf{Short-side stationarity constants.} The represented short side \(p\) determines the rank of the polar update, the Frobenius norm of the pulled-back direction, and the worst-case bound on the nuclear support of the represented update. For any fixed structured matrixization \(T\), the stochastic nonconvex stationarity guarantee therefore depends on the representation through its short side \(p_T\).
\textit{(iii)} \textbf{Representation-dependent mechanism and evidence.} In a controlled teacher--student model, curvature-window collapse together with an MP spectral profile induces an aspect-ratio-dependent ordering of early normalized dissipation.
LLaMA2-130M and LLaMA2-600M pretraining experiments, together with fixed-momentum diagnostics, evaluate the same representation axis at scale by distinguishing balanced reshaping, short-side reduction, and the vector endpoint.

%% file: 2_method_and_theory.tex
\section{A Representation Axis for Muon}
\label{sec:method}

\subsection{Polar Updates in Represented Coordinates}
\label{sec:setup}
For matrices, define \(\langle \mathbf A,\mathbf B\rangle_F=\operatorname{tr}(\mathbf A^\top\mathbf B)\), and let \(\|\cdot\|_F\), \(\|\cdot\|_2=\|\cdot\|_{\mathrm{op}}\), and \(\|\cdot\|_*\) denote the Frobenius, spectral/operator, and nuclear norms, respectively. For a symmetric matrix \(\mathbf H\), let \(\lambda_{\min}(\mathbf H)\) and \(\lambda_{\max}(\mathbf H)\) denote its smallest and largest eigenvalues, respectively. We use \(O(\cdot)\) and \(\widetilde{O}(\cdot)\) for deterministic asymptotic order, with the latter suppressing logarithmic factors. Vectorized parameter blocks are identified with \(\mathbb R^d\) whenever the corresponding vectorization preserves the Frobenius inner product. Given a compact SVD \(\mathbf A=\mathbf U\boldsymbol\Sigma\mathbf V^\top\), define \(\operatorname{Orth}(\mathbf A)=\mathbf U\mathbf V^\top\), with \(\operatorname{Orth}(\mathbf 0)=\mathbf 0\).

A fixed structured representation is either the native Muon matrixization of a layer block, oriented such that the short side corresponds to the rows, or a fixed row-major flat matrixization \(T_p\) defined in Section~\ref{subsec:nearest-square-muse-method}. For a representation \(T:\mathbb R^d\to\mathbb R^{p_T\times q_T}\), where \(p_T\le q_T\), we require that \(T\) preserve inner products:
\(\langle \mathbf u,\mathbf v\rangle_F
=\langle T(\mathbf u),T(\mathbf v)\rangle_F\) for all \(\mathbf u,\mathbf v\in\mathbb R^d.\)
The short side \(p_T\) is the maximum number of singular channels available to the polar step. The corresponding pulled-back polar direction is defined as
\[
\mathbf D_T(\mathbf Z)
:=
T^{-1}\bigl(\operatorname{Orth}(T(\mathbf Z))\bigr).
\]
This formulation makes explicit the matrixization supplied to the Muon polar map while leaving the optimizer state and polar computation unchanged~\citep{chen2025muon,pethick2025training,sfyraki2025lions}.

We consider the stochastic optimization problem
\[
\min_{\mathbf X\in\mathbb R^{m\times n}} f(\mathbf X),
\qquad
f(\mathbf X)=\mathbb E_{\xi\sim\mathcal D}\bigl[f(\mathbf X;\xi)\bigr],
\]
where \(f\) is possibly nonconvex. For stochastic iterations, let \(\mathcal F_{t-1}\) denote the history available before sampling \(\xi_t\), and define
\(\mathbf G_t=\nabla f(\mathbf X_t),
\widehat{\mathbf G}_t=\nabla f(\mathbf X_t;\xi_t)\). We adopt standard smoothness and stochastic-oracle assumptions~\citep{Cutkosky2019MomentumBasedVR,Fang2018SPIDERNN,zhou2020stochastic,DBLP:conf/nips/LiRJ23,DBLP:conf/nips/WangFZ0023}.

\begin{assumption}
\label{ass:smoothness}
The population objective \(f\) is differentiable and \(L\)-smooth on \(\mathbb R^{m\times n}\): \(\|\nabla f(\mathbf Y)-\nabla f(\mathbf X)\|_F\le L\|\mathbf Y-\mathbf X\|_F\) for all \(\mathbf X,\mathbf Y\in\mathbb R^{m\times n}\).
\end{assumption}

\begin{assumption}
\label{ass:stochastic-oracle}
Conditioned on \(\mathcal F_{t-1}\), the sample \(\xi_t\) is independent of the past and the stochastic gradient is unbiased: \(\mathbb E[\widehat{\mathbf G}_t\mid\mathcal F_{t-1}]=\mathbf G_t\).
\end{assumption}

\begin{assumption}
\label{ass:stoc-bound}
There exists \(\sigma^2<\infty\) such that, for every \(t\), \(\mathbb E[\|\widehat{\mathbf G}_t-\mathbf G_t\|_F^2\mid\mathcal F_{t-1}]\le \sigma^2\).
\end{assumption}

\subsection{The Short-Side Ladder from Muon to nSGDM}
\label{subsec:nearest-square-muse-method}

Fix a matrix block \(\mathbf G\in\mathbb R^{m\times n}\), let \(d=mn\), and set \(s=\min(m,n)\).
For any divisor \(p\mid d\) satisfying \(p\le d/p\), define the flat representation
\begin{equation}
\label{eq:flat-representation}
T_p(\mathbf G)
=
R_{p,d/p}\bigl(\operatorname{vec}(\mathbf G)\bigr)
\in\mathbb R^{p\times (d/p)},
\end{equation}
where \(\operatorname{vec}\) denotes the fixed row-major vectorization used throughout this work.
Let \(\mathbf D_{T_p}(\mathbf G)\) denote the associated polar direction.
Together with the native Muon matrixization, these flat representations define the structured representation class considered in this work, which includes native Muon, nearest-square flattening, skinny flattening, and the vector endpoint.
The nearest-square flat representation selects \(T_p\) with the largest admissible short side not exceeding \(\sqrt d\), namely \(p\in\arg\max_{a\mid d,\ a\le \sqrt d} a\).
The skinny-\(\theta\) flat representation selects \(T_p\) whose short side is closest to a prescribed fraction of the native short side, namely \(p\in\arg\min_{a\mid d,\ a\le s}|a-\theta s|\) for \(\theta\in\{1/8,1/4,1/2,3/4\}\).
Both choices preserve the Frobenius geometry of the block while changing the singular vectors and the spectral ball. Unlike coordinate scaling or tensor preconditioning, they introduce neither additional curvature information nor second-moment state~\citep{ruiz2001scaling,qu2025optimal,martens2015optimizing,gupta2018shampoo,vyas2024soap}.

\begin{algorithm}[H]
\caption{\method{}: Muon with structured representations}
\label{alg:muse}
\begin{algorithmic}[1]
\STATE \textbf{Input:} Initial parameters \(\mathbf X_1\in\mathbb R^{m\times n}\), fixed structured representation \(T\), total iterations \(K\), learning rates \(\eta_t>0\), and momentum coefficient \(\mu\in[0,1)\).
\STATE Initialize \(\mathbf M_0=\mathbf0\).
\FOR{\(t=1\) \textbf{to} \(K\)}
    \STATE Draw \(\xi_t\sim\mathcal D\), set \(\widehat{\mathbf G}_t=\nabla f(\mathbf X_t;\xi_t)\), and update \(\mathbf M_t=\mu\mathbf M_{t-1}+(1-\mu)\widehat{\mathbf G}_t\).
    \STATE \(\mathbf X_{t+1}=\mathbf X_t-\eta_t\,\mathbf D_T(\mathbf M_t)\).
\ENDFOR
\end{algorithmic}
\end{algorithm}

Algorithm~\ref{alg:muse} gives a unified formulation of these representation choices.
Native Muon is obtained by using the native matrixization, while nSGDM is obtained at the vector endpoint, for which the spectral, nuclear, and Frobenius norms coincide. In this representation-indexed optimizer family, the representation \(T\) specifies the optimizer geometry on which the polar map acts, while the Muon momentum rule, scaling convention, and Newton--Schulz backend are shared across variants~\citep{liu2025muon,ahn2025dion,khaled2025muonbp,si2025adamuon,Zhang2025AdaGradMM,chang2026muoneq}.

\section{Geometry and Dissipation of Represented Polar Updates}
\label{sec:analysis}

\subsection{Spectral-Ball Descent in Represented Coordinates}
\label{subsec:fixed-representation-polar}

For any fixed structured representation, the polar update is the steepest-descent direction over the spectral ball in the represented coordinates. This formulation includes native Muon, nearest-square flattening, skinny-\(\theta\) flattening, and the vector endpoint. The claim follows from spectral/nuclear-norm duality and standard polar-factor identities~\citep{bhatia2009positive}, and is consistent with spectral-ball and norm-constrained interpretations of Muon~\citep{bernstein2024old,chen2025muon,pethick2025training}. Theorem~\ref{thm:optimization-validity-reshape-polar} formalizes this representation-indexed descent identity and shows that, under smoothness, the decrease depends on the nuclear norm and rank of the represented update.

\begin{theorem}
\label{thm:optimization-validity-reshape-polar}
Let \(T\) be a fixed structured representation with shape \(p\times q\), \(p\le q\), and let \(\mathbf G=\nabla f(\mathbf X)\neq\mathbf0\). Then
\(\langle \mathbf G,\mathbf D_T(\mathbf G)\rangle_F=\|T(\mathbf G)\|_*\), and \(-\mathbf D_T(\mathbf G)\in\arg\min_{\|T(\Delta\mathbf X)\|_2\le1}\langle \mathbf G,\Delta\mathbf X\rangle_F\).
If Assumption~\ref{ass:smoothness} holds, then for every \(\eta>0\),
\begin{equation}
\label{eq:fixed-polar-smooth-descent}
f(\mathbf X-\eta\mathbf D_T(\mathbf G))\le f(\mathbf X)-\eta\|T(\mathbf G)\|_*+\frac{L\eta^2}{2}\operatorname{rank}(T(\mathbf G)).
\end{equation}
See Appendix~\ref{appendix:optimization-validity-reshape-polar} for details.
\end{theorem}

\begin{remark}
The same duality also yields a quadratic-model interpretation: the model \(f(\mathbf X)+\langle \mathbf G,\Delta\mathbf X\rangle_F+\frac{\lambda}{2}\|T(\Delta\mathbf X)\|_2^2\) over \(\Delta\mathbf X\) is minimized at \(-\lambda^{-1}\|T(\mathbf G)\|_*\mathbf D_T(\mathbf G)\), and the corresponding model decrease is \(\|T(\mathbf G)\|_*^2/(2\lambda)\). When \(T\) is the vector endpoint, \(\|T(\mathbf G)\|_*=\|T(\mathbf G)\|_2=\|\mathbf G\|_F\), and \(\mathbf D_T(\mathbf G)\) reduces to the normalized gradient. Thus, different choices of \(T\) satisfy the same descent identity but induce distinct update directions and optimization trajectories.
\end{remark}

\subsection{Stationarity Controlled by the Short Side}
\label{subsec:nonconvex-convergence-nonselective}

We next state a standard nonconvex stationarity guarantee for training with a fixed matrixization. The theorem analyzes a decaying-momentum variant, whereas Algorithm~\ref{alg:muse} gives the fixed-\(\mu\) implementation used in our experiments. The time-varying schedules are chosen to obtain an anytime stationarity guarantee without requiring the training horizon \(K\) to be specified in advance. The guarantee controls stationarity in the geometry induced by \(T\) and is consistent with convergence analyses of polar momentum and adaptive methods~\citep{Li2025ANO,shen2025convergence,Sato2025ConvergenceBA,chang2025convergence,kim2026convergence,nagashima2026improved,zhang2026convergence,duchi2011adaptive,kingma2014adam,Shazeer2018AdafactorAL,ilya2019adamw,gupta2018shampoo}.

\begin{theorem}
\label{thm:main-fixed-representation-stationarity-nonselective}
Let \(T\) be a fixed structured representation with short side \(p_T\).
Fix \(\eta>0\) and \(0<\alpha\le1\), set \(\eta_t=\eta t^{-3/4}\) and \(\alpha_t=1-\alpha t^{-1/2}\).
Suppose Assumptions~\ref{ass:smoothness}, \ref{ass:stochastic-oracle}, and~\ref{ass:stoc-bound} hold, \(f\) is lower bounded by \(f^\star\), \(\mathbb E[f(\mathbf X_1)]<\infty\), and \(\mathbb E\|\nabla f(\mathbf X_1)\|_F^2<\infty\).
Then, for every \(K\ge2\),
\begin{equation}
\label{eq:main-nonconvex-stationarity-summary}
\begin{aligned}
\min_{1\le t\le K}
\mathbb E\|T(\nabla f(\mathbf X_t))\|_*
&\le
\frac{1}{\eta K^{1/4}}
\Bigg[
\mathbb E[f(\mathbf X_1)]-f^\star
+
(1-\alpha)^2\mathbb E\|\nabla f(\mathbf X_1)\|_F^2
+
\alpha^2\sigma^2
\\
&\qquad
+
(1+\ln K)
\bigg(
\frac{2\sqrt2L^2p_T\eta^2}{\alpha}
+
\alpha^2\sigma^2
+
\frac{2p_T\eta^2}{\alpha}
+
\frac{Lp_T\eta^2}{2}
\bigg)
\Bigg].
\end{aligned}
\end{equation}
\end{theorem}
\begin{remark}

The quantity \(\|T(\nabla f(\mathbf X_t))\|_*\) is the \(T\)-nuclear stationarity measure. Since the theorem depends on \(T\) only through \(p_T\), representations with the same short side share the same worst-case constants and yield a \(\widetilde{O}(K^{-1/4})\) rate. This measure also controls native nuclear stationarity: for \(s=\min(m,n)\),
\[
\min_{1\le t\le K}
\mathbb E\|\nabla f(\mathbf X_t)\|_*
\le
\sqrt{s}
\min_{1\le t\le K}
\mathbb E\|T(\nabla f(\mathbf X_t))\|_*
\le
\widetilde{O}(K^{-1/4}).
\]

\end{remark}

Theorem~\ref{thm:main-fixed-representation-stationarity-nonselective} gives a unified first-order stationarity guarantee for any fixed representation, with the proof provided in Appendix~\ref{appendix:muse-nonconvex-stationarity}. Under the bounded-variance stochastic oracle in Assumption~\ref{ass:stoc-bound}, the analysis controls stochastic error through Frobenius moments and worst-case norm conversion. The result thus provides a common convergence baseline for the representation-indexed family. It complements the heavy-tailed analysis of \citet{hubler2026free}, which derives sharper stationarity-adjusted dimension dependence for Muon when noise moments are specified in the native non-Euclidean geometry. We use this fixed-geometry guarantee as a common baseline before turning, in Section~\ref{subsec:controlled-large-c-mp-polar}, to a local curvature-collapse mechanism that distinguishes among representations.


\subsection{Curvature Collapse and Aspect-Ratio Dissipation}
\label{subsec:controlled-large-c-mp-polar}
A local teacher--student calculation shows how curvature collapse induces representation-dependent polar dissipation. Motivated by the observed block structure and heterogeneity of neural-network curvature~\citep{zhang2024whytransformersneedadam,dong2025hessianstructure,tomihari2025gradientheterogeneity,deng2026rmnp}, we show that the empirical fixed-mask \(\mathbf W\)-block Hessian converges to a scalar Frobenius quadratic in the large-class and large-sample limit. Under an MP spectral profile, the induced one-step energy decrease is ordered by \(\beta(\gamma)\gamma^{1/4}\).

Let \(n\), \(m\), and \(C\) denote the input dimension, hidden width, and number of classes.
Let \(\mathbf x\sim \mathcal N(\mathbf0,\mathbf I_n)\), \(\mathbf W_\star\in\mathbb R^{m\times n}\), and \(\mathbf V_{\star,C}\in\mathbb R^{C\times m}\).
For a perturbation \(\mathbf E\in\mathbb R^{m\times n}\), the empirical fixed-mask loss is
\[
\widehat{\mathcal L}_{b,C}(\mathbf E)
:=
\frac12\mathbb E_N
\left\|
\mathbf V_{\star,C}
\left[
\operatorname{ReLU}((\mathbf W_\star+\mathbf E)\mathbf x)
-
\operatorname{ReLU}(\mathbf W_\star\mathbf x)
\right]
\right\|_2^2 .
\]
Here \(\mathbb E_N[\phi(\mathbf x)]:=N^{-1}\sum_{t=1}^N\phi(\mathbf x_t)\).
The class dimension \(C\) controls the coherence of \(\mathbf V_{\star,C}^{\top}\mathbf V_{\star,C}\), while \(N\) controls the empirical fixed-mask curvature estimates.

\begin{assumption}
\label{ass:main-local-row-nondegeneracy}
The reference hidden matrix \(\mathbf W_\star\in\mathbb R^{m\times n}\) satisfies \(\delta_\star:=\min_{1\le i\le m}\|(\mathbf W_\star)_{i,:}\|_2>0\).
Fix \(0<r<\delta_\star\), and let \(\mathcal U_r:=\{\mathbf E\in\mathbb R^{m\times n}:\|\mathbf E\|_F\le r\}\).
Then, for any \(\mathbf E\in\mathcal U_r\), \(\tau\in[0,1]\), and \(i\in[m]\), we have \((\mathbf W_\star+\tau\mathbf E)_{i,:}\ne0\).
\end{assumption}


\begin{assumption}
\label{ass:main-teacher-output-normalization-moments}

The hidden width \(m\) is fixed.
For each output dimension \(C\), let
\(\kappa_C>0\) and
\(\mathbf q_{i,C}\in\mathbb R^C\),
\(i=1,\ldots,m\),
be nonzero unnormalized teacher-output columns.
Define
\(\mathbf v_{i,C}
:=\sqrt{\kappa_C}\mathbf q_{i,C}/\|\mathbf q_{i,C}\|_2\)
and
\(\mathbf V_{\star,C}
:=[\mathbf v_{1,C},\ldots,\mathbf v_{m,C}]\).
For each fixed \(i\), there exists
\(b_i\in(0,\infty)\) such that
\(C^{-1}\sum_{a=1}^C
\mathbb E[(\mathbf q_{i,C})_a^2]\to b_i\);
moreover,
\(\operatorname{Var}
(C^{-1}\sum_{a=1}^C(\mathbf q_{i,C})_a^2)
\to0\). For each fixed \(i\neq k\), (i) \(\mathbb E[(\mathbf q_{i,C})_a(\mathbf q_{k,C})_a]=0\); (ii) \(\mathbb E[(\mathbf q_{i,C})_a
(\mathbf q_{k,C})_a
(\mathbf q_{i,C})_b
(\mathbf q_{k,C})_b]=0\)
for \(a\neq b\); (iii)
\(C^{-1}\sum_{a=1}^C
\mathbb E[(\mathbf q_{i,C})_a^2
(\mathbf q_{k,C})_a^2]
\le K_{\mathrm{out}}\),
where \(K_{\mathrm{out}}<\infty\) is independent of \(C\).
\end{assumption}

\begin{remark}
Assumption~\ref{ass:main-teacher-output-normalization-moments} imposes standard high-dimensional moment and incoherence conditions on the teacher-output columns \citep{vershynin2018high}. It ensures that \(\mathbf V_{\star,C}^{\top}\mathbf V_{\star,C}\) becomes asymptotically diagonal as \(C\to\infty\).
\end{remark}

For \(\mathbf E\in\mathcal U_r\) and \(\tau\in[0,1]\), the empirical fixed-mask \(\mathbf W\)-block Hessian has blocks
\[
\mathbf H^{WW}_{C,N,ik}(\tau\mathbf E)
=
\bigl(\mathbf V_{\star,C}^{\top}\mathbf V_{\star,C}\bigr)_{ik}\,
\mathbb E_N\!\left[
\mathbf 1_{\{(\mathbf W_\star+\tau\mathbf E)_{i,:}\mathbf x>0\}}
\mathbf 1_{\{(\mathbf W_\star+\tau\mathbf E)_{k,:}\mathbf x>0\}}
\mathbf x\mathbf x^\top
\right],
\]
and \(\mathbf H^{WW}_{C,N}(\tau\mathbf E):=[\mathbf H^{WW}_{C,N,ik}(\tau\mathbf E)]_{i,k=1}^m\).
For this empirical Hessian, set
\[
\begin{aligned}
\beta_{1,C,N}
&:=
\min_{\mathbf E\in\mathcal U_r,\ \tau\in[0,1],\ i\in[m]}
\left\{
\lambda_{\min}(\mathbf H^{WW}_{C,N,ii}(\tau\mathbf E))
-
\sum_{k\ne i}\|\mathbf H^{WW}_{C,N,ik}(\tau\mathbf E)\|_{\mathrm{op}}
\right\},
\\
\beta_{2,C,N}
&:=
\max_{\mathbf E\in\mathcal U_r,\ \tau\in[0,1],\ i\in[m]}
\left\{
\lambda_{\max}(\mathbf H^{WW}_{C,N,ii}(\tau\mathbf E))
+
\sum_{k\ne i}\|\mathbf H^{WW}_{C,N,ik}(\tau\mathbf E)\|_{\mathrm{op}}
\right\}.
\end{aligned}
\]
The block-Gershgorin bound gives, uniformly over the local fixed-mask paths used in the theorem and every \(\mathbf D\in\mathbb R^{m\times n}\),
\[
\beta_{1,C,N}\|\mathbf D\|_F^2
\le
\operatorname{vec}(\mathbf D)^\top
\mathbf H^{WW}_{C,N}(\tau\mathbf E)
\operatorname{vec}(\mathbf D)
\le
\beta_{2,C,N}\|\mathbf D\|_F^2,
\]
and, when the denominator is positive, define \(\rho_{C,N}:=(\beta_{2,C,N}-\beta_{1,C,N})/(\beta_{1,C,N}+\beta_{2,C,N})\).
Let \(\zeta_N\) be the local half-space empirical-moment error in Lemma~\ref{lem:appendix-uniform-local-halfspace-moments}.

\begin{remark}
The endpoint \(\beta_{1,C,N}\) is a local fixed-mask strong-convexity certificate for the \(W\)-block, and \(\rho_{C,N}\) is the relative width of the corresponding curvature window. Theorem~\ref{thm:main-curvature-window-collapse} shows that \(\beta_{1,C,N}>0\) becomes typical while \(\rho_{C,N}\to0\), so the local empirical curvature approaches a scalar Frobenius quadratic.
\end{remark}

\begin{theorem}
\label{thm:main-curvature-window-collapse}
Suppose Assumptions~\ref{ass:main-local-row-nondegeneracy} and~\ref{ass:main-teacher-output-normalization-moments} hold.
Fix \(m\) and \(n\).
If \(\mathbf x_1,\ldots,\mathbf x_N\overset{\mathrm{i.i.d.}}{\sim}\mathcal N(\mathbf0,\mathbf I_n)\), then for every \(\varepsilon>0\),
\[
\lim_{C,N\to\infty}\mathbb P(\beta_{1,C,N}>0)=1,
\qquad
\lim_{M\to\infty}\limsup_{C,N\to\infty}
\mathbb P\left(
\frac{\rho_{C,N}}{\zeta_N+C^{-1/2}}>M
\right)=0,
\qquad
\lim_{C,N\to\infty}\mathbb P(\rho_{C,N}>\varepsilon)=0.
\]
\end{theorem}

\begin{corollary}
\label{cor:main-fixed-mask-loss-equivalence}
Under the conditions of Theorem~\ref{thm:main-curvature-window-collapse}, fix \(\mathbf E\in\mathcal U_r\).
If the linear path \(\tau\mapsto \tau\mathbf E\) stays in one empirical ReLU cell, then, with probability tending to one as \(C,N\to\infty\),
\[
\frac{\beta_{1,C,N}}{2}\|\mathbf E\|_F^2
\le
\frac12\mathbb E_N
\left\|
\mathbf V_{\star,C}
\left[
\operatorname{ReLU}((\mathbf W_\star+\mathbf E)\mathbf x)
-
\operatorname{ReLU}(\mathbf W_\star\mathbf x)
\right]
\right\|_2^2
\le
\frac{\beta_{2,C,N}}{2}\|\mathbf E\|_F^2 .
\]
The relative-width limits in Theorem~\ref{thm:main-curvature-window-collapse} also hold.
\end{corollary}

On the curvature-collapse event, with \(\lambda_{C,N}:=(\beta_{1,C,N}+\beta_{2,C,N})/2\), the empirical gradient is approximately radial:
\(\nabla_{\mathbf E}\widehat{\mathcal L}_{b,C}(\mathbf E)\approx\lambda_{C,N}\mathbf E\).
For a Frobenius-isometric representation \(T\) of the \(W\)-block, the corresponding polar step has the leading expansion
\[
\frac{
\|T(\mathbf E-\eta\mathbf D_T(\nabla_{\mathbf E}\widehat{\mathcal L}_{b,C}(\mathbf E)))\|_F^2
-
\|T(\mathbf E)\|_F^2
}
{\eta\|T(\mathbf E)\|_F^2}
\approx
-2\frac{\|T(\mathbf E)\|_*}{\|T(\mathbf E)\|_F^2}
+
\eta
\frac{
\|\mathbf D_T(\nabla_{\mathbf E}\widehat{\mathcal L}_{b,C}(\mathbf E))\|_F^2
}
{\|T(\mathbf E)\|_F^2}.
\]
Theorem~\ref{thm:main-mp-representation-ordering} formalizes this one-step expansion under the spectral profile.
For \(T:\mathbb R^d\to\mathbb R^{p_T\times q_T}\), recall \(d=mn=p_Tq_T\), assume \(p_T\le q_T\), and set \(\gamma_T=p_T/q_T\); the limit first sends \(C,N\to\infty\) at fixed represented dimension and then sends \(d\to\infty\) along the MP representation sequence.

\begin{assumption}
\label{ass:main-mp-spectral-profile-finite-step}
For each representation \(T\), consider an initial perturbation \(\mathbf E_0\in\mathcal U_r\), where \(T(\mathbf E_0)\in\mathbb R^{p_T\times q_T}\), \(p_Tq_T=d\), \(p_T\le q_T\),  \(\gamma_T=p_T/q_T\), and \(\sigma_0>0\). Set \(\widetilde\sigma_i^{(T)}:=\sigma_i(T(\mathbf E_0))/(\sigma_0\sqrt{q_T})\).
The normalized singular values satisfy the Marchenko--Pastur spectral profile
\[
\frac1{p_T}\sum_i\delta_{\widetilde\sigma_i^{(T)}}\Rightarrow\mu_{\gamma_T},
\qquad
\frac1{p_T}\sum_i\widetilde\sigma_i^{(T)}\to\beta(\gamma_T),
\qquad
\frac1{p_T}\sum_i(\widetilde\sigma_i^{(T)})^2\to1,
\]
where \(\beta(\gamma)=\int x\,d\mu_\gamma(x)\), \(\int x^2\,d\mu_\gamma(x)=1\), and \(\beta(\gamma_T)\) is uniformly bounded away from zero along the representation sequence.
\end{assumption}
\begin{remark}
The Marchenko--Pastur spectral profile is used as the isotropic large-dimensional
spectral baseline for the represented perturbation
\citep{bai2010spectral}.
\end{remark}
\begin{theorem}
\label{thm:main-mp-representation-ordering}
Suppose Assumptions~\ref{ass:main-local-row-nondegeneracy}, \ref{ass:main-teacher-output-normalization-moments}, and~\ref{ass:main-mp-spectral-profile-finite-step} hold.
For a fixed learning rate \(\eta>0\), let \(\mathbf E_0\in\mathcal U_r\), assume that the path \(\tau\mapsto \tau\mathbf E_0\) remains in one empirical ReLU cell, and define
\(\mathbf E_1=\mathbf E_0-\eta\mathbf D_T(\nabla_{\mathbf E}\widehat{\mathcal L}_{b,C}(\mathbf E_0))\).
Under the local curvature-collapse event from Theorem~\ref{thm:main-curvature-window-collapse},
\[
\lim_{d\to\infty}
\frac{\sigma_0d^{1/4}}{\beta(\gamma_T)\gamma_T^{1/4}}
\frac{\|T(\mathbf E_0)\|_*}{\|T(\mathbf E_0)\|_F^2}
=1.
\]
Moreover, for every \(\varepsilon>0\),
\[
\lim_{d\to\infty}\lim_{C,N\to\infty}
\mathbb P\left(
\left|
\frac{\|T(\mathbf E_1)\|_F^2-\|T(\mathbf E_0)\|_F^2}
{\eta\|T(\mathbf E_0)\|_F^2}
+
\frac{2\beta(\gamma_T)\gamma_T^{1/4}}
{\sigma_0d^{1/4}}
\right|
>
\varepsilon d^{-1/4}
\right)=0.
\]
Thus the leading-order early normalized dissipation is ordered by \(\beta(\gamma)\gamma^{1/4}\).
\end{theorem}
\begin{remark}
The MP profile is used to calibrate the represented perturbation \(T(\mathbf E_0)\), whose singular-value profile may differ from that of the original error matrix. The optimization statement depends on \(\|T(\mathbf E_0)\|_*/\|T(\mathbf E_0)\|_F^2\), and the MP calculation evaluates this quantity under an isotropic spectral model. Under curvature-window collapse, this gives
\[
\|T(\mathbf E_0)\|_*/\|T(\mathbf E_0)\|_F^2
\sim\beta(\gamma_T)\gamma_T^{1/4}/(\sigma_0d^{1/4}).
\]
Thus, a larger value of \(\beta(\gamma_T)\gamma_T^{1/4}\) gives larger leading-order normalized dissipation at a fixed learning rate, and requires a smaller learning rate to achieve the same early relative decrease. This interpretation underlies Theorem~\ref{thm:main-mp-representation-ordering}, Corollary~\ref{cor:main-learning-rate-calibration}, Figure~\ref{fig:teacher_student_representation_sorting}, and the LLaMA2 learning-dynamics test in Figure~\ref{fig:llama_learning_dynamics}.
\end{remark}

\begin{corollary}
\label{cor:main-learning-rate-calibration}
Under the assumptions of Theorem~\ref{thm:main-mp-representation-ordering}, suppose that two representations \(T_1,T_2\) satisfy \(\beta(\gamma_{T_1})\gamma_{T_1}^{1/4}>\beta(\gamma_{T_2})\gamma_{T_2}^{1/4}\). Then \(T_1\) has larger leading-order early normalized dissipation.
To match a target infinitesimal relative-drop sequence \(\tau_d\to0\), defined by \((\|T(\mathbf E_0)\|_F^2-\|T(\mathbf E_1)\|_F^2)/\|T(\mathbf E_0)\|_F^2\sim\tau_d\),
the learning rate should satisfy \(2\eta_T\|T(\mathbf E_0)\|_*/\|T(\mathbf E_0)\|_F^2\sim\tau_d\), or equivalently, \(\eta_T\sim(\tau_d/2)\sigma_0\sqrt{q_T}/\beta(\gamma_T)\).
\end{corollary}

\begin{remark}
The square endpoint has gain \(\beta(1)(\sigma_0d^{1/4})^{-1}\), whereas the vector endpoint has gain \((\sigma_0\sqrt d)^{-1}\).
Under the standard Gaussian/MP normalization, \(\beta(1)=8/(3\pi)\), so the square-to-vector leading-order ratio is \(\Theta(d^{1/4})\).
On representation ranges over which \(\beta(\gamma)\gamma^{1/4}\) is increasing, reducing the represented short side lowers the leading-order normalized dissipation.
\end{remark}

Proofs of Theorem~\ref{thm:main-curvature-window-collapse}, Corollary~\ref{cor:main-fixed-mask-loss-equivalence}, Theorem~\ref{thm:main-mp-representation-ordering}, and Corollary~\ref{cor:main-learning-rate-calibration} are given in Appendices~\ref{appendix:controlled-shrinkage}--\ref{appendix:representation-ordering-curvature-collapse}.

\textbf{Controlled diagnostic.}
Figures~\ref{fig:teacher_student_case_study} and~\ref{fig:teacher_student_representation_sorting} test hidden-unit block dominance, representation-dependent learning-rate calibration, and the resulting depth-dependent final-loss separation.

\begin{figure*}[!t]\centering
\subfloat[Hidden-unit Hessian blocks.]{\includegraphics[width=0.31\linewidth]{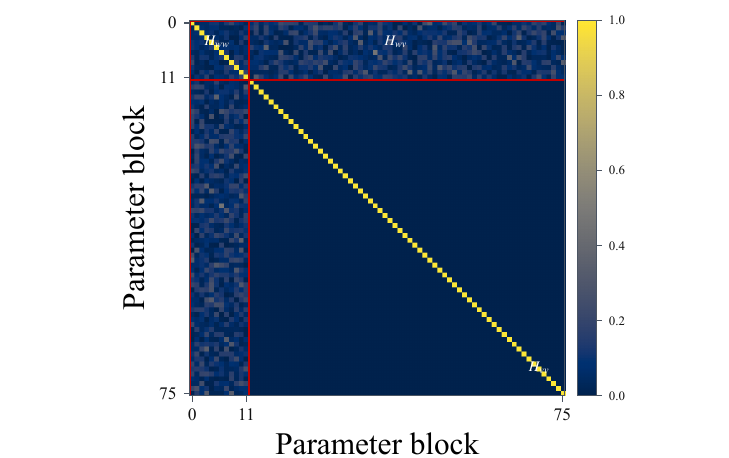}}\hfill \subfloat[Block dominance ratio.]{\includegraphics[width=0.31\linewidth]{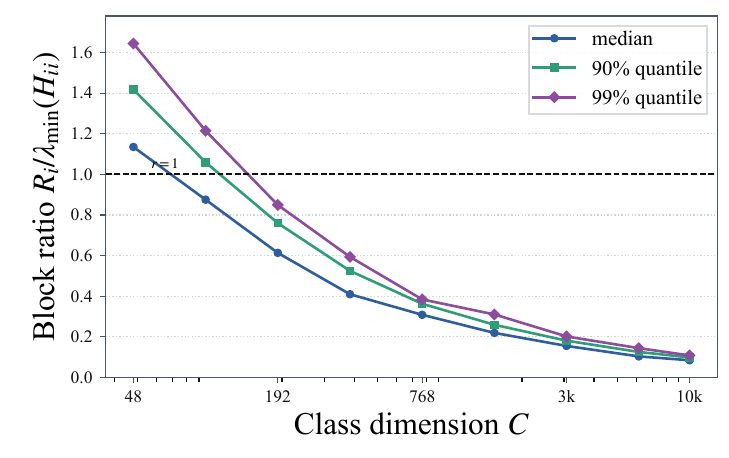}}\hfill \subfloat[Curvature lower envelopes.]{\includegraphics[width=0.31\linewidth]{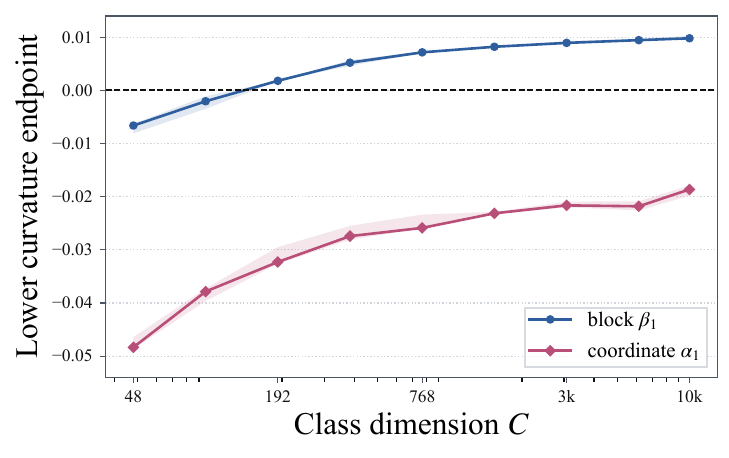}}
\caption{\textbf{The controlled teacher--student model exhibits hidden-unit block dominance in the local \(\mathbf W\)-curvature.}
\textbf{(a)} Empirical \(\mathbf W\)-block Hessian arranged by hidden unit.
\textbf{(b)} Quantiles of \(r_i^{WW}:=R_i/\lambda_{\min}(\widehat{\mathbf H}^{WW}_{ii})\), where \(R_i=\sum_{k\ne i}\|\widehat{\mathbf H}^{WW}_{ik}\|_{\mathrm{op}}\); values below one indicate diagonal hidden-unit blocks dominating their off-block coupling.
\textbf{(c)} Coordinate lower envelope \(\alpha_1\) and hidden-unit block lower endpoint \(\beta_1=\min_i\{\lambda_{\min}(\widehat{\mathbf H}^{WW}_{ii})-R_i\}\) on the same scale.}
\label{fig:teacher_student_case_study}
\end{figure*}
\begin{figure*}[!t]\centering
\subfloat[Single-layer LR sweep.]{\includegraphics[width=0.31\linewidth]{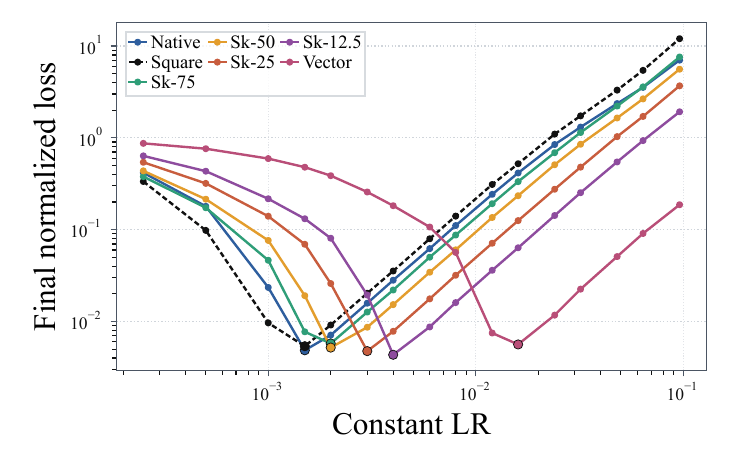}}\hfill
\subfloat[Tuned loss versus depth.]{\includegraphics[width=0.31\linewidth]{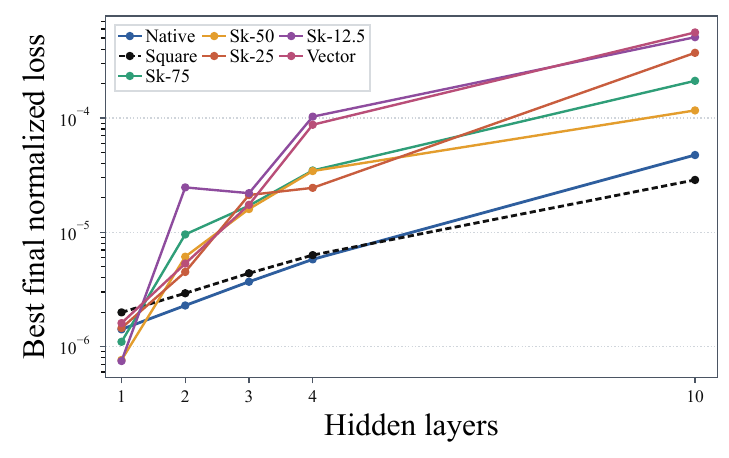}}\hfill
\subfloat[Ten-layer selected curves.]{\includegraphics[width=0.31\linewidth]{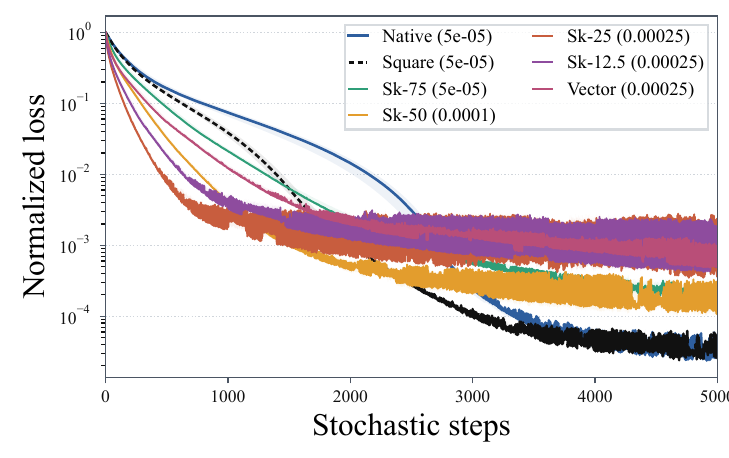}}
\caption{\textbf{Short-side reduction changes the learning-rate calibration and becomes more costly with depth.}
\textbf{(a)} Largest single-hidden-layer dimension sweep \((m,n,C)=(96,192,768)\); selected points move to larger learning rates for narrower or vectorized representations even when the best final losses are comparable.
\textbf{(b)} Best final normalized loss across hidden depth after tuning the constant learning rate for each representation.
\textbf{(c)} Ten-hidden-layer training curves at the selected learning rates.
The remaining single-hidden-layer dimension sweeps are in Figure~\ref{fig:appendix_teacher_student_dimension_sweeps}.}
\label{fig:teacher_student_representation_sorting}
\end{figure*}

%% file: 4_experiment.tex
\section{Representation Geometry at Language-Model Scale}
\label{sec:exp}

\setlength{\textfloatsep}{8pt plus 2pt minus 2pt}
\setlength{\floatsep}{8pt plus 2pt minus 2pt}
\setlength{\intextsep}{8pt plus 2pt minus 2pt}
\setlength{\abovecaptionskip}{3pt}
\setlength{\belowcaptionskip}{0pt}

\subsection{Pretraining Along the Representation Axis}
\label{subsec:exp-learning-dynamics}

We use LLaMA2-style pretraining as an empirical scale test of the representation axis~\citep{Touvron2023Llama2O}.
Across the LLaMA2-130M and LLaMA2-600M runs, all polar variants share the same momentum rule and Newton--Schulz polar backend, and differ only in the matrixization supplied to the polar map; the 130M sweep additionally includes AdamW as a coordinate-wise adaptive reference~\citep{kingma2014adam,ilya2019adamw}.
For both model scales, reported losses and trajectories are averaged over three random seeds.
The 130M experiments use four NVIDIA RTX PRO 6000 Blackwell GPUs (96GB), and the 600M experiments use eight 910C NPUs (64GB).
Detailed experimental settings are provided in Appendix~\ref{app:experimental-details}.

Figure~\ref{fig:llama_learning_dynamics} shows a consistent ordering at both scales: Native and Square remain at the low-loss end of the polar family, short-side reductions form a higher-loss ladder, and Vector is clearly separated. At 600M, Square is within \(0.003\) validation loss of Native.

\begin{figure}[!ht]
\centering
\subfloat[130M final validation sweep.]{%
  \begin{minipage}[t]{0.31\linewidth}
  \centering
  \includegraphics[width=\linewidth]{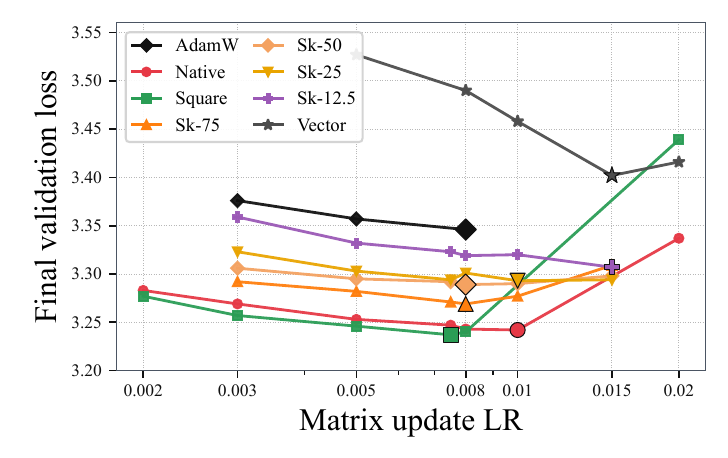}
  \end{minipage}
}
\hfill
\subfloat[130M selected training loss.]{%
  \begin{minipage}[t]{0.31\linewidth}
  \centering
  \includegraphics[width=\linewidth]{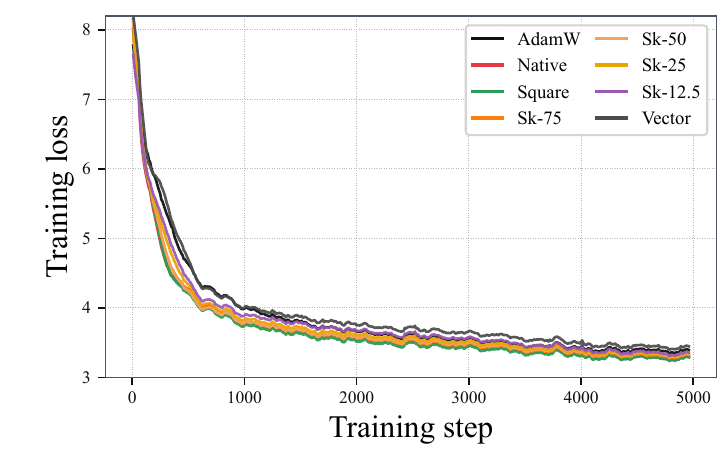}
  \end{minipage}
}
\hfill
\subfloat[130M selected validation loss.]{%
  \begin{minipage}[t]{0.31\linewidth}
  \centering
  \includegraphics[width=\linewidth]{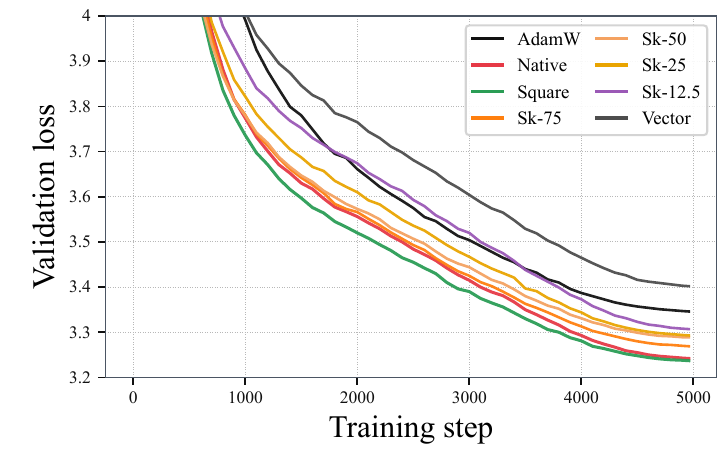}
  \end{minipage}
}
\vspace{0.45em}
\subfloat[600M final validation sweep.]{%
  \begin{minipage}[t]{0.31\linewidth}
  \centering
  \includegraphics[width=\linewidth]{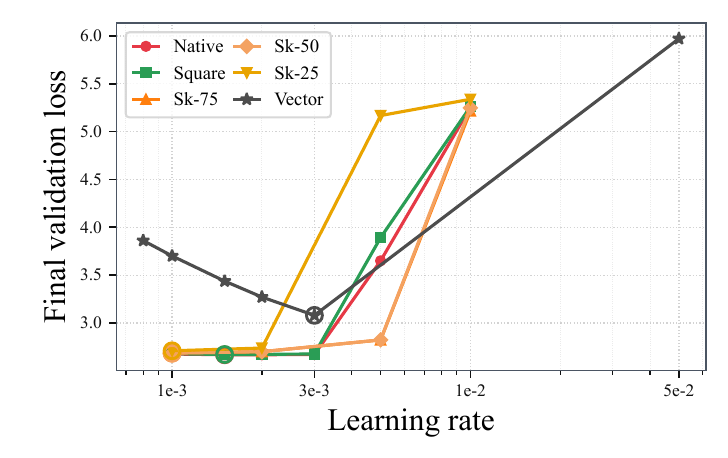}
  \end{minipage}
}
\hfill
\subfloat[600M selected training loss.]{%
  \begin{minipage}[t]{0.31\linewidth}
  \centering
  \includegraphics[width=\linewidth]{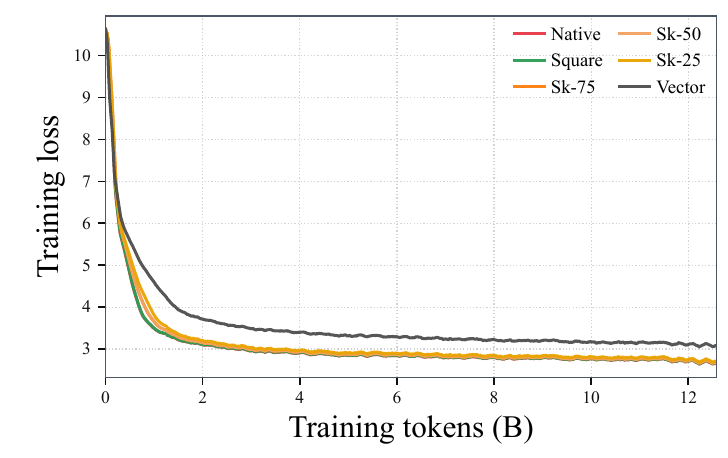}
  \end{minipage}
}
\hfill
\subfloat[600M selected validation loss.]{%
  \begin{minipage}[t]{0.31\linewidth}
  \centering
  \includegraphics[width=\linewidth]{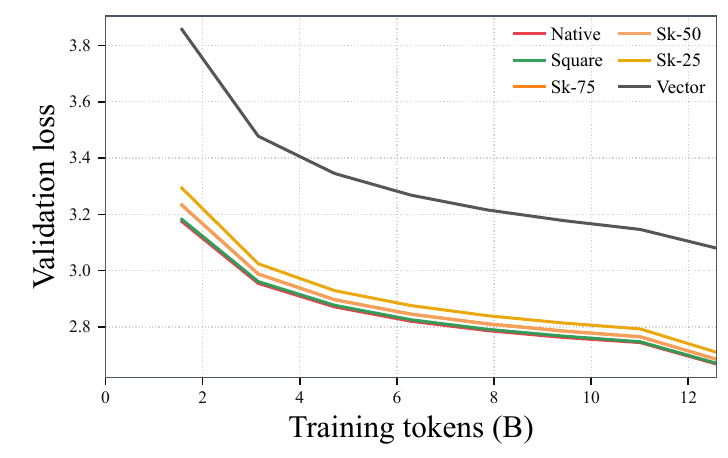}
  \end{minipage}
}
\caption{\textbf{LLaMA2-130M and LLaMA2-600M learning dynamics exhibit the same representation-axis ordering.}
The top row reports the \(2.60\)B-token LLaMA2-130M sweep, and the bottom row reports the \(12.58\)B-token LLaMA2-600M sweep.
\textbf{(a,d)} Final validation loss across learning-rate sweeps; highlighted markers indicate the selected learning rates, and the 130M row includes AdamW as a coordinate-wise adaptive reference.
\textbf{(b,c,e,f)} Selected training and validation trajectories show that Native and Square remain at the low-loss end, short-side reductions incur higher losses, and the vector endpoint is clearly separated.
All displayed losses and trajectories are averaged over three random seeds.}
\label{fig:llama_learning_dynamics}
\label{fig:llama130m_learning_dynamics}
\label{fig:llama0p6b_learning_dynamics}
\end{figure}

\subsection{Fixed-Momentum Geometry Across Representations}
\label{subsec:exp-counterfactual-geometry}

To examine the geometry underlying Figure~\ref{fig:llama130m_learning_dynamics}, we apply alternative representations to saved momentum blocks and measure represented nuclear support, native-direction disagreement, and the \(90\%\) singular-channel count on the same tensors. Appendix~\ref{app:experimental-details} specifies the module set and checkpoint grid; heterogeneity across Transformer modules makes these fixed-momentum diagnostics informative~\citep{zhang2024whytransformersneedadam,dong2025hessianstructure,tomihari2025gradientheterogeneity}.

The checkpoint-resolved appendix summaries in Figures~\ref{fig:llama130m_layerwise_nuclear_norm}--\ref{fig:appendix_llama130m_module_nuclear_heatmaps} and Figures~\ref{fig:appendix_llama600m_layerwise_nuclear_norm}--\ref{fig:appendix_llama600m_module_nuclear_heatmaps} localize this nuclear-support separation across layers and modules at both scales.

\subsection{Module-Level Nuclear Support}
\label{subsec:exp-representation-geometry}

Figure~\ref{fig:llama_module_geometry_scatter} summarizes module geometry at the final checkpoint without duplicating the LLaMA2-130M Frobenius--nuclear panel in Figure~\ref{fig:intro_representation_axis}(a).
This combined view separates the vector endpoint, which remains close to the Frobenius identity, from matrix representations that exhibit broader nuclear support and greater singular-channel usage. Appendices~\ref{app:additional-representation-geometry} and~\ref{app:theory-aligned-early-diagnostics} provide trajectory-level, module-resolved, and early-dissipation diagnostics.

\begin{figure}[!ht]
\centering
\subfloat[130M channel usage.]{%
  \includegraphics[width=0.31\linewidth]{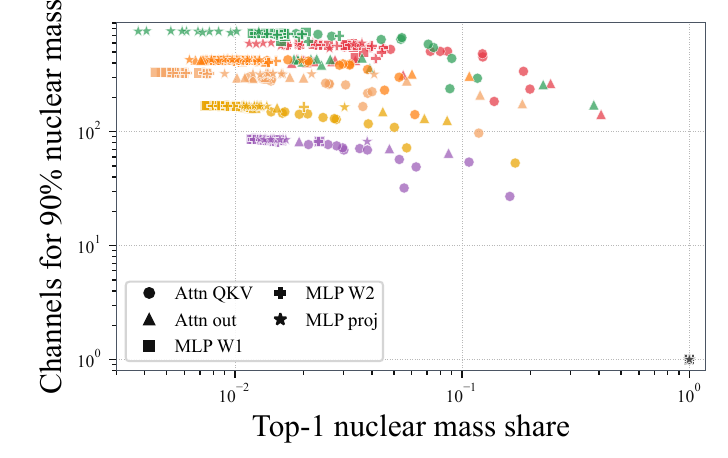}
}%
\hfill
\subfloat[600M Frobenius-nuclear scales.]{%
  \includegraphics[width=0.31\linewidth]{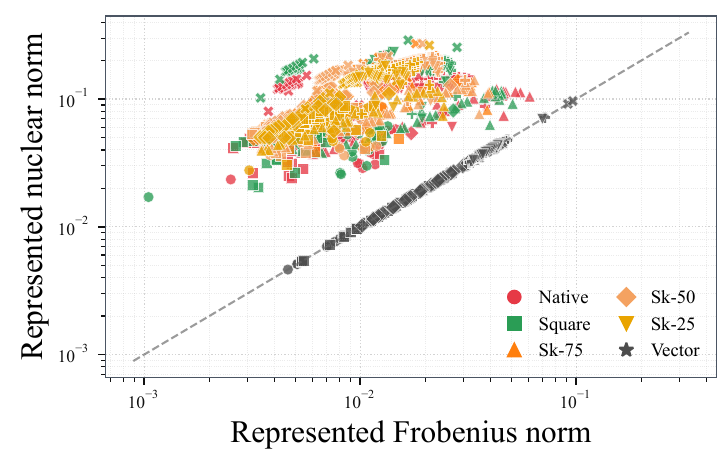}
}%
\hfill
\subfloat[600M channel usage.]{%
  \includegraphics[width=0.31\linewidth]{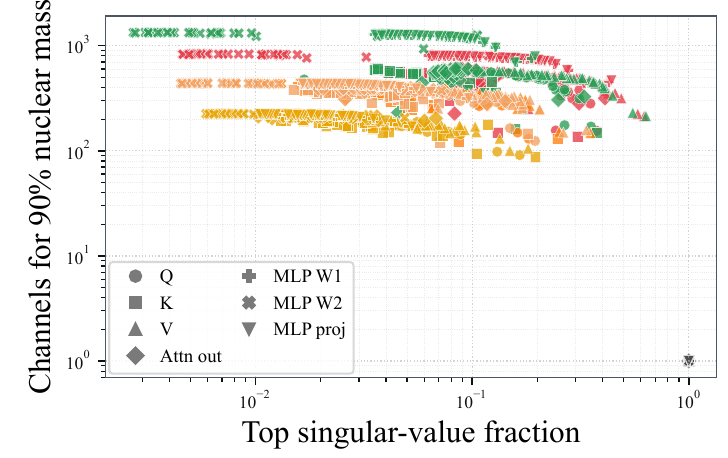}
}
\caption{\textbf{Module-level geometry localizes the same nuclear-support separation at both scales.}
Each point corresponds to a transformer matrix module at the final checkpoint; color indicates the trained representation, and marker shape indicates the module type.
\textbf{(a)} For 130M, the channel-usage view separates concentrated spectra from broader nuclear support; the corresponding Frobenius--nuclear view is shown in Figure~\ref{fig:intro_representation_axis}(a).
\textbf{(b,c)} For 600M, the represented Frobenius--nuclear scales and \(90\%\) channel counts place Vector near the identity-line regime, whereas matrix representations occupy higher-nuclear-support regimes and use substantially more singular channels.}
\label{fig:llama_module_geometry_scatter}
\label{fig:llama130m_module_geometry_scatter}
\label{fig:llama0p6b_module_geometry_scatter}
\end{figure}

%% file: 5_related_work.tex
\section{Related Work}
\label{sec:rw}

\noindent\(\blacktriangleright\) \textbf{Adaptive and structured preconditioning.} Coordinate-wise adaptive methods estimate diagonal or factored geometry from gradient history, whereas natural-gradient methods, K-FAC, Shampoo, SOAP, and related tensor preconditioners organize geometry at the layer, matrix, or tensor level~\citep{duchi2011adaptive,kingma2014adam,ilya2019adamw,Shazeer2018AdafactorAL,martens2015optimizing,gupta2018shampoo,shi2023distributed,eschenhagen2023kronecker,vyas2024soap}. Subsequent work improves momentum design, variance control, sketching, Kronecker structure, whitening, and adaptive structured preconditioning~\citep{xie2024adan,yuan2024mars,Huang2021SUPERADAMFA,chang2026mgup,feinberg2023sketchy,an2025asgo,wen2025foam,Ma2024SWANSW,frans2025stable}. Heavy-tail diagnostics further guide layerwise learning rates and module-wise weight decay in LLMs~\citep{he2026one,he2026alphadecay}. This paper studies a complementary degree of freedom: the first-order polar direction induced by a selected Frobenius-isometric representation.

\noindent\(\blacktriangleright\) \textbf{Muon and polar-update variants.} Muon applies Newton--Schulz iterations to approximate the polar factor of each matrix momentum block, making polar-normalized momentum practical at scale~\citep{jordan6muon,bernstein2024old,liu2025muon,shah2025practical,deepseekai2026deepseekv4}. Existing analyses connect Muon to spectral-norm steepest descent, norm-constrained linear minimization, stochastic Frank--Wolfe perspectives, orthogonalization, spectral flattening, inexact polar updates, matrix-gradient preconditioning, and nonconvex convergence~\citep{chen2025muon,pethick2025training,sfyraki2025lions,boreiko2025towards,nguyen2026spectral,shulgin2025beyond,lau2025polargrad,Li2025ANO,shen2025convergence,Sato2025ConvergenceBA,chang2025convergence,chang2026note,kim2026convergence,nagashima2026improved,zhang2026convergence}. Related variants improve polar computation, communication, quantization, adaptation, balancing, hybridization, and width scaling~\citep{kexuefm-10922,amsel2025polar,boissin2025turbo,grishina2025accelerating,ahn2025dion,khaled2025muonbp,gruntkowska2025error,gupta2025effective,si2025adamuon,zhang2026namo,Zhang2025AdaGradMM,chang2026muoneq,Xu2026moga}. \method{} studies Frobenius-isometric representation as the optimizer geometry supplied to the polar map: the parameter block and Frobenius norm are fixed, while the represented matrix is varied.

\noindent\(\blacktriangleright\) \textbf{Spectral targets and representation order.} Recent work has examined which components of the spectral geometry of Muon drive its behavior. Spectral-flattening interpretations emphasize the learning-rate and convergence effects of orthogonalization~\citep{nguyen2026spectral}, while random or inverted target spectra probe singular-value targets in the native matrix representation~\citep{shumaylov2026muon}. Classical random-matrix laws provide baselines for unstructured spectra, but they keep the task-aligned row--column representation fixed before applying the polar map~\citep{marvcenko1967distribution}. Our results identify representation order as another determinant of polar-update behavior: fixed permutations and reshapes alter singular vectors and nuclear support, whereas fresh strong mixing averages persistent row--column alignment toward normalized momentum.

\noindent\(\blacktriangleright\) \textbf{Matrix analysis.} Our analysis uses spectral--nuclear norm duality and classical facts about polar factors and matrix sign maps~\citep{bhatia2009positive}, together with norm-constrained optimization perspectives on linear minimization and steepest-descent directions~\citep{pethick2025training}. Matrix balancing and diagonal equilibration rescale rows, columns, or coordinates to improve conditioning~\citep{ruiz2001scaling,qu2025optimal}. \method{} instead uses Frobenius-isometric representation changes, spanning native, nearest-square, skinny, and vector views; this places Muon, normalized momentum, and balanced or short-side-reduced reshaping within a single representation-indexed polar-descent framework.

%% file: appendix_analysis.tex
\section{Proof of Theorem~\ref{thm:optimization-validity-reshape-polar}}
\label{appendix:optimization-validity-reshape-polar}

\begin{proof}
Let \(\mathbf A=T(\mathbf G)\) and let \(\mathbf P=\operatorname{Orth}(\mathbf A)\).
Since \(T\) is a bijection and \(\mathbf G\neq\mathbf0\), \(\mathbf A\neq\mathbf0\).
Let \(\mathbf A=\mathbf U\boldsymbol\Sigma\mathbf V^\top\) be a compact SVD.
Then \(\mathbf P=\mathbf U\mathbf V^\top\) and \(\|\mathbf P\|_2=1\). Set \(\mathbf D=T^{-1}(\mathbf P)=\mathbf D_T(\mathbf G)\).
Frobenius isometry gives, for every \(\Delta\mathbf X\),
\begin{equation}
\label{eq:proof-rppd-isometry}
\langle \mathbf G,\Delta\mathbf X\rangle_F
=
\langle T(\mathbf G),T(\Delta\mathbf X)\rangle_F
=
\langle \mathbf A,T(\Delta\mathbf X)\rangle_F .
\end{equation}
Consequently,
\begin{equation}
\label{eq:proof-rppd-alignment}
\langle \mathbf G,\mathbf D\rangle_F
=
\langle \mathbf A,\mathbf P\rangle_F
=
\operatorname{tr}(\boldsymbol\Sigma)
=
\|\mathbf A\|_*
=
\|T(\mathbf G)\|_* .
\end{equation}

Let \(\|T(\Delta\mathbf X)\|_2\le1\) and set \(\mathbf H=T(\Delta\mathbf X)\).
By \eqref{eq:proof-rppd-isometry} and spectral--nuclear duality,
\begin{equation}
\label{eq:proof-rppd-linear-lower}
\langle \mathbf G,\Delta\mathbf X\rangle_F
=
\langle \mathbf A,\mathbf H\rangle_F
\ge
-\|\mathbf A\|_*\|\mathbf H\|_2
\ge
-\|\mathbf A\|_* .
\end{equation}
The choice \(\Delta\mathbf X=-\mathbf D\) satisfies \(\|T(\Delta\mathbf X)\|_2=\|\mathbf P\|_2=1\) and attains equality in \eqref{eq:proof-rppd-linear-lower} by \eqref{eq:proof-rppd-alignment}.
This proves the steepest-unit-step statement.

The smooth descent bound follows from Assumption~\ref{ass:smoothness}.
Since \(\operatorname{Orth}(\mathbf A)\) has Frobenius norm squared equal to \(\operatorname{rank}(\mathbf A)\), Frobenius isometry gives
\[
\|\mathbf D\|_F^2
=
\|\operatorname{Orth}(T(\mathbf G))\|_F^2
=
\operatorname{rank}(T(\mathbf G)).
\]
Substituting \(\Delta\mathbf X=-\eta\mathbf D\) into the smoothness upper bound proves \eqref{eq:fixed-polar-smooth-descent}.

It remains to minimize the spectral quadratic model.
Define
\[
m_{T,\lambda}(\Delta\mathbf X)
=
f(\mathbf X)+\langle \mathbf G,\Delta\mathbf X\rangle_F+\frac{\lambda}{2}\|T(\Delta\mathbf X)\|_2^2 .
\]
For arbitrary \(\Delta\mathbf X\), set \(\mathbf H=T(\Delta\mathbf X)\).
Using \eqref{eq:proof-rppd-isometry},
\begin{equation}
\label{eq:proof-rppd-model-lower}
m_{T,\lambda}(\Delta\mathbf X)-f(\mathbf X)
=
\langle \mathbf A,\mathbf H\rangle_F
+
\frac{\lambda}{2}\|\mathbf H\|_2^2
\ge
-\|\mathbf A\|_*\|\mathbf H\|_2
+
\frac{\lambda}{2}\|\mathbf H\|_2^2 .
\end{equation}
The scalar lower bound is minimized at \(\|\mathbf H\|_2=\|\mathbf A\|_*/\lambda\), with value \(-\|\mathbf A\|_*^2/(2\lambda)\).
Equality is attained by \(\mathbf H^\star=-(\|\mathbf A\|_*/\lambda)\mathbf P\).
Thus
\[
T^{-1}(\mathbf H^\star)
=-\frac{\|T(\mathbf G)\|_*}{\lambda}\mathbf D
=-\frac{\|T(\mathbf G)\|_*}{\lambda}\mathbf D_T(\mathbf G)
\in\arg\min_{\Delta\mathbf X}m_{T,\lambda}(\Delta\mathbf X),
\qquad
\min_{\Delta\mathbf X}m_{T,\lambda}(\Delta\mathbf X)
=
f(\mathbf X)-\frac{\|T(\mathbf G)\|_*^2}{2\lambda}.
\]
\end{proof}

\section{Momentum Alignment}
\label{appendix:usable-polar-drive}

\begin{proposition}
\label{prop:usable-polar-drive}
Let \(T\) be a fixed structured representation with short side \(p_T\).
For any current gradient \(\mathbf G\) and momentum state \(\mathbf M\),
\[
\left|
\langle \mathbf G,\mathbf D_T(\mathbf M)\rangle_F
-
\|T(\mathbf M)\|_*
\right|
\le
\|T(\mathbf M-\mathbf G)\|_*
\le
\sqrt{p_T}\|\mathbf M-\mathbf G\|_F .
\]
Consequently,
\[
\langle \mathbf G,\mathbf D_T(\mathbf M)\rangle_F
\ge
\|T(\mathbf G)\|_*
-
2\sqrt{p_T}\|\mathbf M-\mathbf G\|_F .
\]
\end{proposition}

\begin{proof}
By Frobenius isometry and the definition of \(\mathbf D_T\),
\[
\begin{aligned}
\langle \mathbf G,\mathbf D_T(\mathbf M)\rangle_F
&=
\langle T(\mathbf G),\operatorname{Orth}(T(\mathbf M))\rangle_F
\\
&=
\langle T(\mathbf M),\operatorname{Orth}(T(\mathbf M))\rangle_F
-
\langle T(\mathbf M-\mathbf G),\operatorname{Orth}(T(\mathbf M))\rangle_F
\\
&=
\|T(\mathbf M)\|_*
-
\langle T(\mathbf M-\mathbf G),\operatorname{Orth}(T(\mathbf M))\rangle_F .
\end{aligned}
\]
Therefore
\[
\begin{aligned}
\left|
\langle \mathbf G,\mathbf D_T(\mathbf M)\rangle_F
-
\|T(\mathbf M)\|_*
\right|
&=
\left|
\langle T(\mathbf M-\mathbf G),\operatorname{Orth}(T(\mathbf M))\rangle_F
\right|
\\
&\stackrel{(\circ)}{\le}
\|T(\mathbf M-\mathbf G)\|_*
\\
&\stackrel{(\star)}{\le}
\sqrt{p_T}\|\mathbf M-\mathbf G\|_F .
\end{aligned}
\]
\((\circ)\) uses spectral--nuclear duality and
\(\|\operatorname{Orth}(T(\mathbf M))\|_2\le1\);
\((\star)\) uses the short-side bound and Frobenius isometry.
The lower alignment bound follows from
\[
\begin{aligned}
\langle \mathbf G,\mathbf D_T(\mathbf M)\rangle_F
&\ge
\|T(\mathbf M)\|_*
-
\|T(\mathbf M-\mathbf G)\|_*
\\
&\stackrel{(\circ)}{\ge}
\|T(\mathbf G)\|_*
-
2\|T(\mathbf M-\mathbf G)\|_*
\\
&\stackrel{(\star)}{\ge}
\|T(\mathbf G)\|_*
-
2\sqrt{p_T}\|\mathbf M-\mathbf G\|_F .
\end{aligned}
\]
\((\circ)\) uses the reverse triangle inequality for the nuclear norm, and
\((\star)\) uses \(\|T(\mathbf M-\mathbf G)\|_*\le\sqrt{p_T}\|\mathbf M-\mathbf G\|_F\).
This is the alignment estimate used in the proof of Theorem~\ref{thm:main-fixed-representation-stationarity-nonselective}.
\end{proof}

\section{Proof of Theorem~\ref{thm:main-fixed-representation-stationarity-nonselective}}
\label{appendix:muse-nonconvex-stationarity}

Throughout this section, \(T\) is the fixed structured representation from Theorem~\ref{thm:main-fixed-representation-stationarity-nonselective}, with short side \(p_T\), and \(\mathbf D_T\) is the pulled-back polar direction defined in Section~\ref{sec:setup}.
The proof uses the fact that \(T\) remains a single Frobenius-isometric map across \(t\).

\begin{lemma}
\label{lem:muse-sequence}
Suppose that \(\{e_t\}_{t\ge1}\) and \(\{a_t\}_{t\ge1}\) are nonnegative sequences and
\[
e_{t+1}\le (1-r_{t+1})e_t+a_{t+1},
\qquad
r_t=r t^{-p},\quad r\in(0,1],\quad p\in(0,1].
\]
Then \(r_t e_t\le 2(e_t-e_{t+1}+a_{t+1})\).
\end{lemma}

\begin{proof}
Using the recursion with \(c=2\),
\[
\begin{aligned}
r_t e_t-2(e_t-e_{t+1}+a_{t+1})
&\stackrel{(\circ)}{\le}
r_t e_t-2(e_t+a_{t+1})
+2\bigl((1-r_{t+1})e_t+a_{t+1}\bigr)
\\
&=
e_t(r_t-2r_{t+1})
\\
&\stackrel{(\star)}{\le}
0 .
\end{aligned}
\]
\((\circ)\) uses \(e_{t+1}\le (1-r_{t+1})e_t+a_{t+1}\);
\((\star)\) follows from \(r_t/r_{t+1}=((t+1)/t)^p\le2^p\le2\).
\end{proof}

\begin{lemma}
\label{lem:muse-momentum-tracking}
Under the assumptions of Theorem~\ref{thm:main-fixed-representation-stationarity-nonselective}, define
\[
\mathbf S_t:=\mathbf M_t-\nabla f(\mathbf X_t),
\qquad
E_t:=\mathbb E\|\mathbf S_t\|_F^2,
\qquad
r_t:=1-\alpha_t=\alpha t^{-1/2}.
\]
With \(\eta_t=\eta t^{-3/4}\) and \(\alpha_t=1-r_t\),
\begin{equation}
\label{eq:muse-tracking-recursion}
E_{t+1}
\le
(1-r_{t+1})E_t
+
\left(\frac{2\sqrt2L^2p_T\eta^2}{\alpha}+\alpha^2\sigma^2\right)(t+1)^{-1}.
\end{equation}
Consequently, for every \(K\ge2\),
\begin{equation}
\label{eq:muse-weighted-tracking}
\sum_{t=1}^K r_t E_t
\le
2E_1+2\left(\frac{2\sqrt2L^2p_T\eta^2}{\alpha}+\alpha^2\sigma^2\right)(1+\ln K).
\end{equation}
\end{lemma}

\begin{proof}
Let \(\mathbf G_t=\nabla f(\mathbf X_t)\) and \(\widehat{\mathbf G}_t=\nabla f(\mathbf X_t;\xi_t)\).
The tracking error satisfies
\[
\begin{aligned}
\mathbf S_{t+1}
&=
\mathbf M_{t+1}-\mathbf G_{t+1}
\\
&=
\alpha_{t+1}(\mathbf M_t-\mathbf G_t)
+
\alpha_{t+1}(\mathbf G_t-\mathbf G_{t+1})
+
(1-\alpha_{t+1})(\widehat{\mathbf G}_{t+1}-\mathbf G_{t+1}).
\end{aligned}
\]
The first equality is the definition of \(\mathbf S_{t+1}\), and the second expands the momentum recursion.
Taking conditional expectation, the cross terms with
\(\widehat{\mathbf G}_{t+1}-\mathbf G_{t+1}\) vanish by unbiasedness and conditional independence. Hence
\[
\begin{aligned}
E_{t+1}
&\stackrel{(\circ)}{\le}
\alpha_{t+1}E_t
+
\frac{\alpha_{t+1}^2}{1-\alpha_{t+1}}
L^2\mathbb E\|\mathbf X_{t+1}-\mathbf X_t\|_F^2
+
(1-\alpha_{t+1})^2\sigma^2
\\
&\le
(1-r_{t+1})E_t
+
\frac{L^2p_T\eta_t^2}{r_{t+1}}
+
r_{t+1}^2\sigma^2.
\end{aligned}
\]
\((\circ)\) uses Young's inequality with parameter
\((1-\alpha_{t+1})/\alpha_{t+1}\), Lipschitz gradients, and the variance bound.
The second line uses \(\alpha_{t+1}=1-r_{t+1}\),
\(\mathbf X_{t+1}-\mathbf X_t=-\eta_t\mathbf D_T(\mathbf M_t)\), and
\(\|\mathbf D_T(\mathbf M_t)\|_F^2\le p_T\).
For \(\eta_t=\eta t^{-3/4}\) and \(r_{t+1}=\alpha(t+1)^{-1/2}\),
\[
\begin{aligned}
\frac{\eta_t^2}{r_{t+1}}
&=
\frac{\eta^2}{\alpha}t^{-3/2}(t+1)^{1/2}
\\
&\stackrel{(\circ)}{\le}
\frac{2\sqrt2\eta^2}{\alpha}(t+1)^{-1},
\\
r_{t+1}^2
&=
\alpha^2(t+1)^{-1}.
\end{aligned}
\]
\((\circ)\) uses \(t^{-3/2}(t+1)^{1/2}\le 2\sqrt2(t+1)^{-1}\) for \(t\ge1\).
This gives \eqref{eq:muse-tracking-recursion}.
Summing Lemma~\ref{lem:muse-sequence} yields
\[
\begin{aligned}
\sum_{t=1}^K r_t E_t
&\stackrel{(\circ)}{\le}
2\sum_{t=1}^K(E_t-E_{t+1})
+
2\left(\frac{2\sqrt2L^2p_T\eta^2}{\alpha}+\alpha^2\sigma^2\right)
\sum_{t=1}^K(t+1)^{-1}
\\
&\stackrel{(\star)}{\le}
2E_1
+
2\left(\frac{2\sqrt2L^2p_T\eta^2}{\alpha}+\alpha^2\sigma^2\right)(1+\ln K).
\end{aligned}
\]
\((\circ)\) applies Lemma~\ref{lem:muse-sequence} to \eqref{eq:muse-tracking-recursion}, and
\((\star)\) telescopes the first sum and uses the harmonic bound.
This proves \eqref{eq:muse-weighted-tracking}.
\end{proof}

\begin{proof}[Proof of Theorem~\ref{thm:main-fixed-representation-stationarity-nonselective}]
Let \(\mathbf G_t=\nabla f(\mathbf X_t)\), \(\mathbf S_t=\mathbf M_t-\mathbf G_t\), \(r_t=1-\alpha_t=\alpha t^{-1/2}\), and \(\mathbf D_t=\mathbf D_T(\mathbf M_t)\).
The one-step smoothness bound is
\begin{equation}
\label{eq:muse-one-step-smooth}
\begin{aligned}
f(\mathbf X_{t+1})
&\le
f(\mathbf X_t)-\eta_t\langle \mathbf G_t,\mathbf D_t\rangle_F
+
\frac{L\eta_t^2}{2}\|\mathbf D_t\|_F^2 .
\end{aligned}
\end{equation}
This uses Assumption~\ref{ass:smoothness} and
\(\mathbf X_{t+1}-\mathbf X_t=-\eta_t\mathbf D_t\).
The direction bounds are
\[
\begin{aligned}
\|\mathbf D_t\|_F^2
&\le
p_T
\\
\langle \mathbf G_t,\mathbf D_t\rangle_F
&\ge
\|T(\mathbf G_t)\|_*
-
2\sqrt{p_T}\|\mathbf S_t\|_F .
\end{aligned}
\]
The first line uses the polar identities and Frobenius isometry; the second applies Proposition~\ref{prop:usable-polar-drive} with
\(\mathbf G=\mathbf G_t\) and \(\mathbf M=\mathbf M_t\).
Substituting these estimates into \eqref{eq:muse-one-step-smooth} gives
\begin{equation}
\label{eq:muse-descent-before-young}
\begin{aligned}
f(\mathbf X_{t+1})
&\le
f(\mathbf X_t)
-
\eta_t\|T(\mathbf G_t)\|_*
+
2\sqrt{p_T}\eta_t\|\mathbf S_t\|_F
+
\frac{Lp_T}{2}\eta_t^2.
\end{aligned}
\end{equation}
The tracking term is controlled by
\[
\begin{aligned}
2\sqrt{p_T}\eta_t\|\mathbf S_t\|_F
&\stackrel{(\circ)}{\le}
\frac12 r_t\|\mathbf S_t\|_F^2
+
\frac{2p_T\eta_t^2}{r_t}
\\
&=
\frac12 r_t\|\mathbf S_t\|_F^2
+
\frac{2p_T\eta^2}{\alpha}t^{-1}.
\end{aligned}
\]
\((\circ)\) is Young's inequality with parameter \(r_t\).
Also, \(\eta_t^2=\eta^2t^{-3/2}\le\eta^2t^{-1}\).
Taking expectations in \eqref{eq:muse-descent-before-young} and summing gives
\begin{equation}
\label{eq:muse-sum-gradient}
\begin{aligned}
\sum_{t=1}^K\eta_t\mathbb E\|T(\mathbf G_t)\|_*
&\stackrel{(\circ)}{\le}
\mathbb E[f(\mathbf X_1)-f(\mathbf X_{K+1})]
\\
&\quad
+
\frac12\sum_{t=1}^K r_t E_t
+
\frac{2p_T\eta^2}{\alpha}\sum_{t=1}^K t^{-1}
+
\frac{Lp_T\eta^2}{2}\sum_{t=1}^K t^{-3/2}
\\
&\le
\Delta_1
+
\frac12\sum_{t=1}^K r_t E_t
+
\frac{2p_T\eta^2}{\alpha}\sum_{t=1}^K t^{-1}
+
\frac{Lp_T\eta^2}{2}\sum_{t=1}^K t^{-3/2},
\end{aligned}
\end{equation}
where \(\Delta_1=\mathbb E[f(\mathbf X_1)]-f^\star\).
\((\circ)\) telescopes the descent inequality after expectation, and the last line uses \(f(\mathbf X_{K+1})\ge f^\star\).
Now Lemma~\ref{lem:muse-momentum-tracking} and the harmonic bounds imply
\[
\begin{aligned}
\sum_{t=1}^K\eta_t\mathbb E\|T(\mathbf G_t)\|_*
&\stackrel{(\circ)}{\le}
\Delta_1+E_1
\\
&\quad
+
(1+\ln K)
\left(
\frac{2\sqrt2L^2p_T\eta^2}{\alpha}
+
\alpha^2\sigma^2
+
\frac{2p_T\eta^2}{\alpha}
+
\frac{Lp_T\eta^2}{2}
\right).
\end{aligned}
\]
\((\circ)\) applies \eqref{eq:muse-weighted-tracking},
\(\sum_{t=1}^Kt^{-1}\le1+\ln K\), and
\(\sum_{t=1}^Kt^{-3/2}\le1+\ln K\).
Finally, \(\mathbf M_1=\alpha\widehat{\mathbf G}_1\), and
\[
\begin{aligned}
E_1
&=
\mathbb E\|\alpha(\widehat{\mathbf G}_1-\mathbf G_1)-(1-\alpha)\mathbf G_1\|_F^2
\\
&\stackrel{(\circ)}{\le}
\alpha^2\sigma^2+(1-\alpha)^2\mathbb E\|\mathbf G_1\|_F^2.
\end{aligned}
\]
\((\circ)\) uses unbiasedness to remove the cross term and then applies the variance bound.
Substituting this estimate into the weighted-sum bound gives
\[
\begin{aligned}
\sum_{t=1}^K\eta_t\mathbb E\|T(\mathbf G_t)\|_*
&\le
\mathbb E[f(\mathbf X_1)]-f^\star
+
(1-\alpha)^2\mathbb E\|\mathbf G_1\|_F^2
+
\alpha^2\sigma^2
\\
&\quad
+
(1+\ln K)
\left(
\frac{2\sqrt2L^2p_T\eta^2}{\alpha}
+
\alpha^2\sigma^2
+
\frac{2p_T\eta^2}{\alpha}
+
\frac{Lp_T\eta^2}{2}
\right).
\end{aligned}
\]
Finally, since \(\sum_{t=1}^K\eta_t
=\eta\sum_{t=1}^Kt^{-3/4}\ge\eta K^{1/4}\),
\[
\begin{aligned}
\min_{1\le t\le K}
\mathbb E\|T(\nabla f(\mathbf X_t))\|_*
&\le
\frac{\sum_{t=1}^K\eta_t\mathbb E\|T(\mathbf G_t)\|_*}
{\sum_{t=1}^K\eta_t}
\\
&\le
\frac{1}{\eta K^{1/4}}
\Bigg[
\mathbb E[f(\mathbf X_1)]-f^\star
+
(1-\alpha)^2\mathbb E\|\nabla f(\mathbf X_1)\|_F^2
+
\alpha^2\sigma^2
\\
&\qquad
+
(1+\ln K)
\bigg(
\frac{2\sqrt2L^2p_T\eta^2}{\alpha}
+
\alpha^2\sigma^2
+
\frac{2p_T\eta^2}{\alpha}
+
\frac{Lp_T\eta^2}{2}
\bigg)
\Bigg].
\end{aligned}
\]
This is \eqref{eq:main-nonconvex-stationarity-summary}.
The representation-dependent inputs are Frobenius isometry and \(\|\mathbf D_T(\mathbf M)\|_F^2\le p_T\).
Hence the argument applies to any fixed structured representation \(T\) in the setup of Theorem~\ref{thm:main-fixed-representation-stationarity-nonselective}.
Because the weighted-sum bound is established before taking the minimum, the same argument is uniform over fixed structured representations with the same short side.
\end{proof}

\section{Fixed-Mask Curvature Collapse}
\label{appendix:controlled-shrinkage}

Assumptions~\ref{ass:main-local-row-nondegeneracy} and~\ref{ass:main-teacher-output-normalization-moments} are stated in the main text.
This appendix uses the local window \(\mathcal U_r\) from Assumption~\ref{ass:main-local-row-nondegeneracy}, with \(\mathbf W\)-block perturbations denoted by \(\mathbf E_W\).
Within this appendix, write \(\mathbf K_C:=\mathbf V_{\star,C}^{\top}\mathbf V_{\star,C}\).
For a symmetric block matrix \(\mathbf H=[\mathbf H_{ik}]_{i,k=1}^m\in\mathbb R^{mn\times mn}\), where \(\mathbf H_{ik}\in\mathbb R^{n\times n}\) and \(\mathbf H_{ki}=\mathbf H_{ik}^{\top}\), set
\[
R_i(\mathbf H):=\sum_{k\ne i}\|\mathbf H_{ik}\|_{\mathrm{op}}.
\]

\begin{lemma}
\label{lem:appendix-block-gershgorin-quadratic-form}

Let \(\mathbf H\) be a symmetric block matrix as above. Then for any \(\mathbf E_W\in\mathbb R^{m\times n}\),
\[
\min_i\{\lambda_{\min}(\mathbf H_{ii})-R_i(\mathbf H)\}\|\mathbf E_W\|_F^2
\le
\operatorname{vec}(\mathbf E_W)^\top \mathbf H\operatorname{vec}(\mathbf E_W)
\le
\max_i\{\lambda_{\max}(\mathbf H_{ii})+R_i(\mathbf H)\}\|\mathbf E_W\|_F^2 .
\]

\begin{proof}

Write $\mathbf E_W=[\mathbf e_1^\top;\ldots;\mathbf e_m^\top]$. By block expansion,
\[
\begin{aligned}
\operatorname{vec}(\mathbf E_W)^\top \mathbf H\operatorname{vec}(\mathbf E_W)
&\stackrel{(\circ)}{=}
\sum_{i=1}^m\sum_{k=1}^m \mathbf e_i^\top \mathbf H_{ik}\mathbf e_k
=
\sum_i \mathbf e_i^\top \mathbf H_{ii}\mathbf e_i
+
\sum_{i\ne k}\mathbf e_i^\top \mathbf H_{ik}\mathbf e_k .
\end{aligned}
\]
Here, $(\circ)$ is the quadratic-form expansion with blocks indexed by hidden units.

We first prove the upper bound:
\[
\begin{aligned}
\operatorname{vec}(\mathbf E_W)^\top \mathbf H\operatorname{vec}(\mathbf E_W)
&\stackrel{(\circ)}{=}
\sum_i \mathbf e_i^\top \mathbf H_{ii}\mathbf e_i
+
\sum_{i\ne k}\mathbf e_i^\top \mathbf H_{ik}\mathbf e_k
\\
&\stackrel{(\star)}{\le}
\sum_i
\lambda_{\max}(\mathbf H_{ii})\|\mathbf e_i\|_2^2
+
\sum_{i\ne k}
\|\mathbf H_{ik}\|_{\mathrm{op}}\|\mathbf e_i\|_2\|\mathbf e_k\|_2
\\
&\stackrel{(\bullet)}{\le}
\sum_i
\lambda_{\max}(\mathbf H_{ii})\|\mathbf e_i\|_2^2
+
\frac12
\sum_{i\ne k}
\|\mathbf H_{ik}\|_{\mathrm{op}}
\left(
\|\mathbf e_i\|_2^2+\|\mathbf e_k\|_2^2
\right)
\\
&\stackrel{(\diamond)}{=}
\sum_i
\left[
\lambda_{\max}(\mathbf H_{ii})
+
\sum_{k\ne i}\|\mathbf H_{ik}\|_{\mathrm{op}}
\right]
\|\mathbf e_i\|_2^2
\\
&\stackrel{(\heartsuit)}{\le}
\max_i\{\lambda_{\max}(\mathbf H_{ii})+R_i(\mathbf H)\}
\sum_i\|\mathbf e_i\|_2^2
\\
&=
\max_i\{\lambda_{\max}(\mathbf H_{ii})+R_i(\mathbf H)\}\|\mathbf E_W\|_F^2 .
\end{aligned}
\]
Here, $(\circ)$ is the block expansion; $(\star)$ uses the spectral upper bound for diagonal blocks and the operator-norm bound for off-diagonal blocks; $(\bullet)$ uses $2ab\le a^2+b^2$; $(\diamond)$ uses \(\mathbf H_{ki}=\mathbf H_{ik}^\top\) to combine the coefficients of the same \(\|\mathbf e_i\|_2^2\); $(\heartsuit)$ uses the definition of \(R_i(\mathbf H)\).

The lower bound is analogous:
\[
\begin{aligned}
\operatorname{vec}(\mathbf E_W)^\top \mathbf H\operatorname{vec}(\mathbf E_W)
&\stackrel{(\circ)}{=}
\sum_i \mathbf e_i^\top \mathbf H_{ii}\mathbf e_i
+
\sum_{i\ne k}\mathbf e_i^\top \mathbf H_{ik}\mathbf e_k
\\
&\stackrel{(\star)}{\ge}
\sum_i
\lambda_{\min}(\mathbf H_{ii})\|\mathbf e_i\|_2^2
-
\sum_{i\ne k}
\|\mathbf H_{ik}\|_{\mathrm{op}}\|\mathbf e_i\|_2\|\mathbf e_k\|_2
\\
&\stackrel{(\bullet)}{\ge}
\sum_i
\lambda_{\min}(\mathbf H_{ii})\|\mathbf e_i\|_2^2
-
\frac12
\sum_{i\ne k}
\|\mathbf H_{ik}\|_{\mathrm{op}}
\left(
\|\mathbf e_i\|_2^2+\|\mathbf e_k\|_2^2
\right)
\\
&\stackrel{(\diamond)}{=}
\sum_i
\left[
\lambda_{\min}(\mathbf H_{ii})
-
\sum_{k\ne i}\|\mathbf H_{ik}\|_{\mathrm{op}}
\right]
\|\mathbf e_i\|_2^2
\\
&\stackrel{(\heartsuit)}{\ge}
\min_i\{\lambda_{\min}(\mathbf H_{ii})-R_i(\mathbf H)\}
\sum_i\|\mathbf e_i\|_2^2
\\
&=
\min_i\{\lambda_{\min}(\mathbf H_{ii})-R_i(\mathbf H)\}\|\mathbf E_W\|_F^2 .
\end{aligned}
\]
Here, $(\circ)$ is the block expansion; $(\star)$ uses the spectral lower bound for diagonal blocks and the operator-norm lower bound for off-diagonal blocks; $(\bullet)$ uses \(2ab\le a^2+b^2\); $(\diamond)$ uses \(\mathbf H_{ki}=\mathbf H_{ik}^\top\) to combine the coefficients of the same \(\|\mathbf e_i\|_2^2\); $(\heartsuit)$ uses the definition of \(R_i(\mathbf H)\). This proves the lemma.

\end{proof}
\end{lemma}
\medskip\noindent\textbf{Notation.}

Use \(\mathbb E_N\) from the main text and the appendix shorthand \(\mathbf K_C=\mathbf V_{\star,C}^{\top}\mathbf V_{\star,C}\). In the realizable teacher--student setting, define the empirical \(\mathbf W\)-block loss
\begin{equation}
\label{eq:appendix-fixed-mask-loss}
\widehat{\mathcal L}_{b,C}(\mathbf E_W)
:=
\frac12
\mathbb E_N
\left\|
\mathbf V_{\star,C}
\left[
\sigma((\mathbf W_\star+\mathbf E_W)\mathbf x)-\sigma(\mathbf W_\star \mathbf x)
\right]
\right\|_2^2 .
\end{equation}

For $\mathbf E\in\mathcal U_r$ and $\tau\in[0,1]$, define
\begin{equation}
\label{eq:appendix-fixed-mask-curvature}
\begin{aligned}
\mathbf 1_i(\tau,\mathbf E;\mathbf x)
&:=
\mathbf 1_{\{(\mathbf W_\star+\tau\mathbf E)_{i,:}\mathbf x>0\}},
\\
\mathbf H^{WW}_{C,N,ik}(\tau\mathbf E)
&:=
(\mathbf K_C)_{ik}
\mathbb E_N[
\mathbf 1_i(\tau,\mathbf E;\mathbf x)\mathbf 1_k(\tau,\mathbf E;\mathbf x)\mathbf x\mathbf x^\top],
\\
\mathbf H^{WW}_{C,N}(\tau\mathbf E)
&:=
[\mathbf H^{WW}_{C,N,ik}(\tau\mathbf E)]_{i,k=1}^m .
\end{aligned}
\end{equation}

\medskip\noindent\textbf{Teacher-output normalization.}

We use the teacher-output construction from Assumption~\ref{ass:main-teacher-output-normalization-moments}; in particular, \((\mathbf K_C)_{ii}=\kappa_C\). The Gram coherence error is
\begin{equation}
\label{eq:appendix-gram-coherence-error}
\epsilon_C:=\kappa_C^{-1}\max_i\sum_{k\ne i}|(\mathbf K_C)_{ik}|.
\end{equation}

\begin{lemma}
\label{lem:appendix-teacher-output-gram-coherence-rate}

Under Assumption~\ref{ass:main-teacher-output-normalization-moments},
\[
\lim_{M\to\infty}\limsup_{C\to\infty}
\mathbb P(C^{1/2}\epsilon_C>M)=0.
\]

\begin{proof}

For each fixed $i$, let $Y_{t,C}^{(i)}:=(\mathbf q_{i,C})_t^2$. For each fixed $i\ne k$, let
\[
X_{t,C}^{(i,k)}
:=
(\mathbf q_{i,C})_t(\mathbf q_{k,C})_t .
\]

We first control the unnormalized column norms. By Chebyshev's inequality,
\[
\begin{aligned}
\frac1C\|\mathbf q_{i,C}\|_2^2
=
\frac1C\sum_{t=1}^{C}Y_{t,C}^{(i)}
\xrightarrow{p}
b_i .
\end{aligned}
\]
In particular, since $b_i>0$,
\[
\left(
\frac1C\|\mathbf q_{i,C}\|_2^2
\right)^{-1/2}
\]
is tight.

Next control the cross inner products. Fix $i\ne k$ and, to simplify notation, write
\[
X_{t,C}:=X_{t,C}^{(i,k)} .
\]
By Cauchy--Schwarz, uncorrelatedness, and the second-moment scaling,
\[
\begin{aligned}
\mathbb E\left|
\frac1C
\sum_{t=1}^{C}X_{t,C}
\right|
&\stackrel{(\circ)}{\le}
\left(
\mathbb E
\left|
\frac1C
\sum_{t=1}^{C}X_{t,C}
\right|^2
\right)^{1/2}
\\
&=
\left(
\frac1{C^2}
\mathbb E
\left|
\sum_{t=1}^{C}X_{t,C}
\right|^2
\right)^{1/2}
\\
&\stackrel{(\star)}{=}
\left(
\frac1{C^2}
\sum_{t=1}^{C}
\mathbb E[X_{t,C}^2]
\right)^{1/2}
\\
&\stackrel{(\bullet)}{\le}
\left(\frac{K_{\mathrm{out}}}{C}\right)^{1/2}=
O(C^{-1/2}).
\end{aligned}
\]
Here, $(\circ)$ uses Cauchy--Schwarz; $(\star)$ uses $\mathbb E[X_{t,C}]=0$ and $\mathbb E[X_{\tau,C}X_{t,C}]=0$ $(\tau\ne t)$; $(\bullet)$ uses $\frac1C\sum_{t=1}^{C}\mathbb E[X_{t,C}^2]\le K_{\mathrm{out}}$.

Therefore, by Markov's inequality, for any $M>0$,
\[
\begin{aligned}
\mathbb P
\left(
\left|
\frac1C
\sum_{t=1}^{C}X_{t,C}
\right|
>
M C^{-1/2}
\right)
\le
\frac{
\mathbb E\left|
C^{-1}\sum_{t=1}^{C}X_{t,C}
\right|
}{
M C^{-1/2}
}
\le
\frac{K_{\mathrm{out}}^{1/2}}{M}.
\end{aligned}
\]
Thus
\(C^{1/2}C^{-1}\langle \mathbf q_{i,C},\mathbf q_{k,C}\rangle\) is tight.

Now return to the normalized Gram off-diagonal. For $i\ne k$,
\[
\begin{aligned}
C^{1/2}\frac{|(\mathbf K_C)_{ik}|}{\kappa_C}
&=
C^{1/2}\frac{|\langle\mathbf v_{i,C},\mathbf v_{k,C}\rangle|}
{\kappa_C}
\\
&=
C^{1/2}\frac{|\langle\mathbf q_{i,C},\mathbf q_{k,C}\rangle|}
{\|\mathbf q_{i,C}\|_2\|\mathbf q_{k,C}\|_2}
\\
&=
\frac{
\left|
C^{1/2}C^{-1}
\langle\mathbf q_{i,C},\mathbf q_{k,C}\rangle
\right|
}{
\left(
C^{-1}\|\mathbf q_{i,C}\|_2^2
\right)^{1/2}
\left(
C^{-1}\|\mathbf q_{k,C}\|_2^2
\right)^{1/2}
}
\end{aligned}
\]
The numerator is tight by the preceding cross-inner-product estimate.
The denominator is bounded away from zero in probability because \(C^{-1}\|\mathbf q_{i,C}\|_2^2\xrightarrow{p}b_i>0\) and \(C^{-1}\|\mathbf q_{k,C}\|_2^2\xrightarrow{p}b_k>0\).
Hence \(C^{1/2}|(\mathbf K_C)_{ik}|/\kappa_C\) is tight.

Finally, use that $m$ is fixed. Let
\[
M_C
:=
\max_{i\ne k}
\frac{|(\mathbf K_C)_{ik}|}{\kappa_C}.
\]
Since there are only finitely many off-diagonal pairs, for any $\delta>0$, we can choose $A<\infty$ such that
\[
\begin{aligned}
\mathbb P(M_C>A C^{-1/2})
\le
\sum_{i\ne k}
\mathbb P
\left(
\frac{|(\mathbf K_C)_{ik}|}{\kappa_C}
>
A C^{-1/2}
\right)
\le
\delta
\end{aligned}
\]
for all sufficiently large $C$. Therefore
\(C^{1/2}M_C\) is tight.
Thus
\[
\begin{aligned}
\epsilon_C
&=
\frac1{\kappa_C}
\max_i\sum_{k\ne i}|(\mathbf K_C)_{ik}|
\\
&=
\max_i\sum_{k\ne i}
\frac{|(\mathbf K_C)_{ik}|}{\kappa_C}
\\
&\le
(m-1)M_C .
\end{aligned}
\]
Since $m$ is fixed and \(C^{1/2}M_C\) is tight, \(C^{1/2}\epsilon_C\) is tight. This proves the lemma.

\end{proof}
\end{lemma}
\begin{lemma}
\label{lem:appendix-uniform-local-halfspace-moments}

Fix $(m,n)$ and the local window $\mathcal U_r$ from Assumption~\ref{ass:main-local-row-nondegeneracy}. Let $\mathbf x_1,\ldots,\mathbf x_N\overset{\mathrm{i.i.d.}}{\sim}\mathcal N(\mathbf0,\mathbf I_n)$. Define
\[
\zeta_N
:=
\sup_{\substack{\mathbf E_W\in\mathcal U_r\\ \tau\in[0,1]\\ i\in[m]}}
\left\|
\mathbb E_N[
\mathbf 1_{\{(\mathbf W_\star+\tau\mathbf E_W)_{i,:}\mathbf x>0\}}\mathbf x\mathbf x^\top
]
-\frac12 \mathbf I
\right\|_{\mathrm{op}} .
\]
Then, for every \(\varepsilon>0\),
\[
\lim_{N\to\infty}\mathbb P(\zeta_N>\varepsilon)=0.
\]

\begin{proof}
By Assumption~\ref{ass:main-local-row-nondegeneracy}, for any $\mathbf E_W\in\mathcal U_r$, $\tau\in[0,1]$, and $i\in[m]$, we have $(\mathbf W_\star+\tau\mathbf E_W)_{i,:}\ne0$. Hence we can define
\[
\mathbf u_i(\tau,\mathbf E_W):=
\frac{(\mathbf W_\star+\tau\mathbf E_W)_{i,:}}
{\|(\mathbf W_\star+\tau\mathbf E_W)_{i,:}\|_2}
\in S^{n-1},
\qquad
\mathbf 1_{\{(\mathbf W_\star+\tau\mathbf E_W)_{i,:}\mathbf x>0\}}
=
\mathbf 1_{\{\mathbf u_i(\tau,\mathbf E_W)^\top \mathbf x>0\}} .
\]
By Gaussian symmetry, for any $\mathbf u\in S^{n-1}$,
\[
\mathbb E[
\mathbf 1_{\{\mathbf u^\top \mathbf x>0\}}\mathbf x\mathbf x^\top
]
=
\frac12 \mathbf I .
\]
Thus
\[
\begin{aligned}
\zeta_N
&\le
\sup_{\mathbf u\in S^{n-1}}
\left\|
\mathbb E_N[
\mathbf 1_{\{\mathbf u^\top \mathbf x>0\}}\mathbf x\mathbf x^\top
]
-
\mathbb E[
\mathbf 1_{\{\mathbf u^\top \mathbf x>0\}}\mathbf x\mathbf x^\top
]
\right\|_{\mathrm{op}} .
\end{aligned}
\]

We now prove that the right-hand side converges to $0$. Fix coordinates $1\le a,b\le n$, and write
\[
g_{ab}(\mathbf x):=x_a x_b .
\]
We first prove
\[
\sup_{\mathbf u\in S^{n-1}}
\left|
(\mathbb E_N-\mathbb E)
\left[
\mathbf 1_{\{\mathbf u^\top \mathbf x>0\}}g_{ab}(\mathbf x)
\right]
\right|
\xrightarrow{p}0 .
\]

Take any $R>0$, and decompose
\[
g_{ab}(\mathbf x)
=
g_{ab}(\mathbf x)\mathbf 1_{\{\|\mathbf x\|_2\le R\}}
+
g_{ab}(\mathbf x)\mathbf 1_{\{\|\mathbf x\|_2>R\}} .
\]
The tail term satisfies
\[
\begin{aligned}
\sup_{\mathbf u\in S^{n-1}}
\left|
(\mathbb E_N-\mathbb E)
\left[
\mathbf 1_{\{\mathbf u^\top \mathbf x>0\}}
g_{ab}(\mathbf x)\mathbf 1_{\{\|\mathbf x\|_2>R\}}
\right]
\right|\le
\mathbb E_N[
|g_{ab}(\mathbf x)|\mathbf 1_{\{\|\mathbf x\|_2>R\}}
]
+
\mathbb E[
|g_{ab}(\mathbf x)|\mathbf 1_{\{\|\mathbf x\|_2>R\}}
] .
\end{aligned}
\]
Since $|g_{ab}(\mathbf x)|\le \|\mathbf x\|_2^2$ and $\mathbf x\sim \mathcal N(\mathbf0,\mathbf I_n)$, we have
\[
\mathbb E|g_{ab}(\mathbf x)|<\infty .
\]
Therefore, by the ordinary law of large numbers,
\[
\mathbb E_N[
|g_{ab}(\mathbf x)|\mathbf 1_{\{\|\mathbf x\|_2>R\}}
]
\stackrel{p}{\to}
\mathbb E[
|g_{ab}(\mathbf x)|\mathbf 1_{\{\|\mathbf x\|_2>R\}}
] ,
\]
Moreover,
\[
\mathbb E[
|g_{ab}(\mathbf x)|\mathbf 1_{\{\|\mathbf x\|_2>R\}}
]
\to0
\qquad (R\to\infty).
\]
Thus the tail term can be made arbitrarily small by first taking $R$ large. This smallness is in probability: given any $\delta,\tau>0$, one can first take $R$ sufficiently large so that the probability that the tail upper bound above exceeds $\tau$ is, once $N$ is sufficiently large, less than $\delta$.

Next handle the bounded part. Fix $R$, and define
\[
f_{\mathbf u}^{(R)}(\mathbf x)
:=
\mathbf 1_{\{\mathbf u^\top \mathbf x>0\}}g_{ab}(\mathbf x)\mathbf 1_{\{\|\mathbf x\|_2\le R\}} .
\]
Since $n$ is fixed, the unit sphere $S^{n-1}$ is compact. Therefore, for any $h>0$, there exist finitely many points
\[
\mathbf v_1,\ldots,\mathbf v_K\in S^{n-1}
\]
such that for any $\mathbf u\in S^{n-1}$, there exists some $\mathbf v_\ell$ satisfying
\[
\|\mathbf u-\mathbf v_\ell\|_2\le h .
\]

If $\|\mathbf u-\mathbf v_\ell\|_2\le h$ and $\|\mathbf x\|_2\le R$, then whenever
\[
\mathbf 1_{\{\mathbf u^\top \mathbf x>0\}}
\ne
\mathbf 1_{\{\mathbf v_\ell^\top \mathbf x>0\}}
\]
we must have
\[
|\mathbf v_\ell^\top \mathbf x|
\le
|(\mathbf v_\ell-\mathbf u)^\top \mathbf x|
\le
h \|\mathbf x\|_2
\le
h R .
\]
Therefore
\[
|f_{\mathbf u}^{(R)}(\mathbf x)-f_{\mathbf v_\ell}^{(R)}(\mathbf x)|
\le
R^2\mathbf 1_{\{|\mathbf v_\ell^\top \mathbf x|\le h R\}} .
\]
Thus
\[
\begin{aligned}
&\sup_{\mathbf u\in S^{n-1}}
\left|
(\mathbb E_N-\mathbb E)f_{\mathbf u}^{(R)}
\right|
\\
&\le
\max_{1\le \ell\le K}
\left|
(\mathbb E_N-\mathbb E)f_{\mathbf v_\ell}^{(R)}
\right|
+
R^2
\max_{1\le \ell\le K}
\mathbb E_N[
\mathbf 1_{\{|\mathbf v_\ell^\top \mathbf x|\le h R\}}
]
+
R^2
\max_{1\le \ell\le K}
\mathbb E[
\mathbf 1_{\{|\mathbf v_\ell^\top \mathbf x|\le h R\}}
] .
\end{aligned}
\]
Since $K<\infty$, for each fixed $\mathbf v_\ell$, the ordinary law of large numbers gives
\[
(\mathbb E_N-\mathbb E)f_{\mathbf v_\ell}^{(R)}
\stackrel{p}{\to}0 .
\]
Therefore
\[
\max_{1\le \ell\le K}
\left|
(\mathbb E_N-\mathbb E)f_{\mathbf v_\ell}^{(R)}
\right|
\stackrel{p}{\to}0 .
\]
Similarly,
\[
\max_{1\le \ell\le K}
\left|
\mathbb E_N[
\mathbf 1_{\{|\mathbf v_\ell^\top \mathbf x|\le h R\}}
]
-
\mathbb E[
\mathbf 1_{\{|\mathbf v_\ell^\top \mathbf x|\le h R\}}
]
\right|
\stackrel{p}{\to}0 .
\]
Moreover, because for any unit vector $\mathbf v_\ell$, $\mathbf v_\ell^\top \mathbf x\sim \mathcal N(0,1)$,
\[
\mathbb E[
\mathbf 1_{\{|\mathbf v_\ell^\top \mathbf x|\le h R\}}
]
=
\mathbb P(|Z|\le h R),
\qquad Z\sim \mathcal N(0,1).
\]
Therefore
\[
\max_{1\le \ell\le K}
\mathbb E[
\mathbf 1_{\{|\mathbf v_\ell^\top \mathbf x|\le h R\}}
]
=
\mathbb P(|Z|\le h R)
\to0
\qquad (h\to0).
\]
Combining the above, for any fixed $R$ and any $\varepsilon>0$, first take $h>0$ so that the boundary-band probability term is sufficiently small, and then let $N\to\infty$ to obtain
\[
\sup_{\mathbf u\in S^{n-1}}
\left|
(\mathbb E_N-\mathbb E)f_{\mathbf u}^{(R)}
\right|
\xrightarrow{p}0 .
\]
Combining this with the tail control above and then letting $R\to\infty$, we obtain
\[
\sup_{\mathbf u\in S^{n-1}}
\left|
(\mathbb E_N-\mathbb E)
\left[
\mathbf 1_{\{\mathbf u^\top \mathbf x>0\}}x_a x_b
\right]
\right|
\xrightarrow{p}0 .
\]

Since $n$ is fixed, there are only finitely many coordinate pairs $(a,b)$, so
\[
\max_{1\le a,b\le n}
\sup_{\mathbf u\in S^{n-1}}
\left|
(\mathbb E_N-\mathbb E)
\left[
\mathbf 1_{\{\mathbf u^\top \mathbf x>0\}}x_a x_b
\right]
\right|
\xrightarrow{p}0 .
\]
Also, for any $\mathbf A\in\mathbb R^{n\times n}$,
\[
\|\mathbf A\|_{\mathrm{op}}
\le
n\max_{1\le a,b\le n}|(\mathbf A)_{ab}|,
\]
Thus
\[
\sup_{\mathbf u\in S^{n-1}}
\left\|
\mathbb E_N[
\mathbf 1_{\{\mathbf u^\top \mathbf x>0\}}\mathbf x\mathbf x^\top
]
-
\mathbb E[
\mathbf 1_{\{\mathbf u^\top \mathbf x>0\}}\mathbf x\mathbf x^\top
]
\right\|_{\mathrm{op}}
\xrightarrow{p}0 .
\]
Therefore, for every \(\varepsilon>0\),
\[
\lim_{N\to\infty}\mathbb P(\zeta_N>\varepsilon)=0.
\]

\end{proof}
\end{lemma}
\medskip\noindent\textbf{Definition 2.}\label{def:appendix-empirical-curvature-endpoints-relative-width}

Use the empirical curvature block \(\mathbf H^{WW}_{C,N}(\tau\mathbf E)\) in \eqref{eq:appendix-fixed-mask-curvature}. Define
\begin{equation}
\label{eq:appendix-empirical-curvature-window}
\begin{aligned}
R_{C,N,i}(\tau\mathbf E)
&:=
\sum_{k\ne i}
\|\mathbf H^{WW}_{C,N,ik}(\tau\mathbf E)\|_{\mathrm{op}},
\\
\beta_{1,C,N}
&:=
\min_{\substack{\mathbf E\in\mathcal U_r\\ \tau\in[0,1]\\ i\in[m]}}
\{\lambda_{\min}(\mathbf H^{WW}_{C,N,ii}(\tau\mathbf E))-R_{C,N,i}(\tau\mathbf E)\},
\\
\beta_{2,C,N}
&:=
\max_{\substack{\mathbf E\in\mathcal U_r\\ \tau\in[0,1]\\ i\in[m]}}
\{\lambda_{\max}(\mathbf H^{WW}_{C,N,ii}(\tau\mathbf E))+R_{C,N,i}(\tau\mathbf E)\},
\\
B_N
&:=
\max_{\substack{\mathbf E\in\mathcal U_r\\ \tau\in[0,1]\\ i\ne k}}
\left\|
\mathbb E_N[
\mathbf 1_i(\tau,\mathbf E;\mathbf x)\mathbf 1_k(\tau,\mathbf E;\mathbf x)\mathbf x\mathbf x^\top]
\right\|_{\mathrm{op}} .
\end{aligned}
\end{equation}
When the denominator is positive, define the empirical curvature-window relative width
\begin{equation}
\label{eq:appendix-empirical-curvature-relative-width}
\rho_{C,N}
:=
\frac{\beta_{2,C,N}-\beta_{1,C,N}}
{\beta_{1,C,N}+\beta_{2,C,N}} .
\end{equation}

\begin{proposition}
\label{prop:appendix-finite-cn-endpoint-control}

Fix \((m,n)\) and the local window \(\mathcal U_r\). Use \(\epsilon_C\) from \eqref{eq:appendix-gram-coherence-error}, and suppose \((\mathbf K_C)_{ii}=\kappa_C>0\). Let \(\zeta_N\) be the empirical half-space moment error in Lemma~\ref{lem:appendix-uniform-local-halfspace-moments}. Then, for all finite \((C,N)\),
\[
\beta_{1,C,N}\ge\kappa_C\left(\frac12-\zeta_N-\epsilon_C B_N\right),
\qquad
\beta_{2,C,N}\le\kappa_C\left(\frac12+\zeta_N+\epsilon_C B_N\right),
\qquad
B_N\le\frac12+\zeta_N .
\]
Therefore, for any $\mathbf E_W\in\mathcal U_r$ and $\tau\in[0,1]$, and on the event $\zeta_N+\epsilon_C B_N\le1/4$, we have
\[
\beta_{1,C,N}\|\mathbf E_W\|_F^2
\le
\operatorname{vec}(\mathbf E_W)^\top \mathbf H^{WW}_{C,N}(\tau\mathbf E_W)\operatorname{vec}(\mathbf E_W)
\le
\beta_{2,C,N}\|\mathbf E_W\|_F^2,
\quad
\beta_{1,C,N}\ge\frac{\kappa_C}{4},
\quad
\rho_{C,N}\le4(\zeta_N+\epsilon_C B_N).
\]

\begin{proof}

We first control the diagonal block. By definition,
\[
\begin{aligned}
\mathbf H^{WW}_{C,N,ii}(\tau\mathbf E)
&\stackrel{(\circ)}{=}
(\mathbf K_C)_{ii}
\mathbb E_N[
\mathbf 1_i(\tau,\mathbf E;\mathbf x)^2\mathbf x\mathbf x^\top]
\\
&\stackrel{(\star)}{=}
\kappa_C
\mathbb E_N[
\mathbf 1_i(\tau,\mathbf E;\mathbf x)\mathbf x\mathbf x^\top],
\end{aligned}
\]
Here, $(\circ)$ uses the definition of $\mathbf H^{WW}_{C,N,ii}$; $(\star)$ uses $\mathbf 1_i^2=\mathbf 1_i$ and $(\mathbf K_C)_{ii}=\kappa_C$. By the definition of $\zeta_N$,
\[
\begin{aligned}
\lambda_{\min}
\left(
\mathbb E_N[
\mathbf 1_i(\tau,\mathbf E;\mathbf x)\mathbf x\mathbf x^\top]
\right)
&\stackrel{(\circ)}{\ge}
\frac12-\zeta_N,
\\
\lambda_{\max}
\left(
\mathbb E_N[
\mathbf 1_i(\tau,\mathbf E;\mathbf x)\mathbf x\mathbf x^\top]
\right)
&\stackrel{(\star)}{\le}
\frac12+\zeta_N .
\end{aligned}
\]
Therefore
\[
\begin{aligned}
\lambda_{\min}(\mathbf H^{WW}_{C,N,ii}(\tau\mathbf E))
&\stackrel{(\circ)}{\ge}
\kappa_C\left(\frac12-\zeta_N\right),
\\
\lambda_{\max}(\mathbf H^{WW}_{C,N,ii}(\tau\mathbf E))
&\stackrel{(\star)}{\le}
\kappa_C\left(\frac12+\zeta_N\right).
\end{aligned}
\]
Here, $(\circ)$ and $(\star)$ use $\kappa_C>0$.

Next control the off-block leakage. By the definitions of $B_N$ and $\epsilon_C$,
\[
\begin{aligned}
R_{C,N,i}(\tau\mathbf E)
&=
\sum_{k\ne i}
\|\mathbf H^{WW}_{C,N,ik}(\tau\mathbf E)\|_{\mathrm{op}}
\\
&\stackrel{(\circ)}{\le}
\sum_{k\ne i}
|(\mathbf K_C)_{ik}|
\left\|
\mathbb E_N[
\mathbf 1_i(\tau,\mathbf E;\mathbf x)\mathbf 1_k(\tau,\mathbf E;\mathbf x)\mathbf x\mathbf x^\top]
\right\|_{\mathrm{op}}
\\
&\stackrel{(\star)}{\le}
B_N\sum_{k\ne i}|(\mathbf K_C)_{ik}|
\\
&\stackrel{(\diamond)}{\le}
\kappa_C\epsilon_C B_N .
\end{aligned}
\]
Here, $(\circ)$ uses the definition of $\mathbf H^{WW}_{C,N,ik}$; $(\star)$ uses the definition of $B_N$; $(\diamond)$ uses the definition of $\epsilon_C$.

Thus
\[
\begin{aligned}
\beta_{1,C,N}
&=
\min_{\mathbf E,\tau,i}
\left[
\lambda_{\min}(\mathbf H^{WW}_{C,N,ii}(\tau\mathbf E))
-
R_{C,N,i}(\tau\mathbf E)
\right]
\\
&\stackrel{(\circ)}{\ge}
\kappa_C\left(\frac12-\zeta_N\right)
-
\kappa_C\epsilon_C B_N
\\
&=
\kappa_C
\left(
\frac12-\zeta_N-\epsilon_C B_N
\right),
\end{aligned}
\]
Here, $(\circ)$ uses the diagonal lower bound and the off-block leakage bound. Similarly,
\[
\begin{aligned}
\beta_{2,C,N}
&=
\max_{\mathbf E,\tau,i}
\left[
\lambda_{\max}(\mathbf H^{WW}_{C,N,ii}(\tau\mathbf E))
+
R_{C,N,i}(\tau\mathbf E)
\right]
\\
&\stackrel{(\circ)}{\le}
\kappa_C\left(\frac12+\zeta_N\right)
+
\kappa_C\epsilon_C B_N
\\
&=
\kappa_C
\left(
\frac12+\zeta_N+\epsilon_C B_N
\right).
\end{aligned}
\]

We now prove $B_N\le1/2+\zeta_N$. For any $i\ne k$,
\[
\begin{aligned}
0
&\preceq
\mathbb E_N[
\mathbf 1_i(\tau,\mathbf E;\mathbf x)\mathbf 1_k(\tau,\mathbf E;\mathbf x)\mathbf x\mathbf x^\top]
\\
&\preceq
\mathbb E_N[
\mathbf 1_i(\tau,\mathbf E;\mathbf x)\mathbf x\mathbf x^\top].
\end{aligned}
\]
Therefore
\[
\begin{aligned}
B_N
&\stackrel{(\circ)}{\le}
\max_{\mathbf E,\tau,i}
\left\|
\mathbb E_N[
\mathbf 1_i(\tau,\mathbf E;\mathbf x)\mathbf x\mathbf x^\top]
\right\|_{\mathrm{op}}
\\
&\stackrel{(\star)}{\le}
\left\|\frac12 \mathbf I\right\|_{\mathrm{op}}+\zeta_N
\\
&=
\frac12+\zeta_N .
\end{aligned}
\]
Here, $(\circ)$ uses PSD domination; $(\star)$ uses the definition of $\zeta_N$.

By Lemma~\ref{lem:appendix-block-gershgorin-quadratic-form}, set
\[
\mathbf H=\mathbf H^{WW}_{C,N}(\tau\mathbf E_W).
\]
By the definitions of \(\beta_{1,C,N}\) and \(\beta_{2,C,N}\), the lemma gives
\[
\begin{aligned}
\beta_{1,C,N}\|\mathbf E_W\|_F^2
&\stackrel{(\circ)}{\le}
\operatorname{vec}(\mathbf E_W)^\top
\mathbf H^{WW}_{C,N}(\tau\mathbf E_W)
\operatorname{vec}(\mathbf E_W)
\\
&\stackrel{(\star)}{\le}
\beta_{2,C,N}\|\mathbf E_W\|_F^2 .
\end{aligned}
\]
Here, $(\circ)$ and $(\star)$ use the block-Gershgorin quadratic-form bound from Lemma~\ref{lem:appendix-block-gershgorin-quadratic-form}.

If
\[
\zeta_N+\epsilon_C B_N\le\frac14,
\]
then
\[
\begin{aligned}
\beta_{1,C,N}
&\stackrel{(\circ)}{\ge}
\kappa_C
\left(
\frac12-\zeta_N-\epsilon_C B_N
\right)
\\
&\stackrel{(\star)}{\ge}
\frac{\kappa_C}{4}
>0 .
\end{aligned}
\]
Here, $(\circ)$ uses the lower endpoint bound above; $(\star)$ uses the small-error event.

Finally, control the relative width. By the upper and lower endpoint bounds,
\[
\begin{aligned}
\beta_{2,C,N}-\beta_{1,C,N}
&\stackrel{(\circ)}{\le}
2\kappa_C(\zeta_N+\epsilon_C B_N).
\end{aligned}
\]
On the other hand,
\[
\begin{aligned}
\beta_{2,C,N}
&\stackrel{(\circ)}{\ge}
\kappa_C\left(\frac12-\zeta_N\right),
\end{aligned}
\]
Here, $(\circ)$ uses $\lambda_{\max}(\mathbf H^{WW}_{C,N,ii})\ge\lambda_{\min}(\mathbf H^{WW}_{C,N,ii})$ and $R_{C,N,i}\ge0$. Therefore
\[
\begin{aligned}
\beta_{1,C,N}+\beta_{2,C,N}
&\stackrel{(\circ)}{\ge}
\kappa_C
\left(
\frac12-\zeta_N-\epsilon_C B_N
\right)
+
\kappa_C
\left(
\frac12-\zeta_N
\right)
\\
&=
\kappa_C
\left(
1-2\zeta_N-\epsilon_C B_N
\right).
\end{aligned}
\]
On the event $\zeta_N+\epsilon_C B_N\le1/4$,
\[
1-2\zeta_N-\epsilon_C B_N
\ge
\frac12 .
\]
Thus
\[
\begin{aligned}
\rho_{C,N}
&=
\frac{\beta_{2,C,N}-\beta_{1,C,N}}
{\beta_{1,C,N}+\beta_{2,C,N}}
\\
&\stackrel{(\circ)}{\le}
\frac{
2\kappa_C(\zeta_N+\epsilon_C B_N)
}{
\kappa_C(1-2\zeta_N-\epsilon_C B_N)
}
\\
&\stackrel{(\star)}{\le}
4(\zeta_N+\epsilon_C B_N).
\end{aligned}
\]
Here, $(\circ)$ uses the bounds on the numerator and denominator; $(\star)$ uses the small-error event.

This proves the proposition.

\end{proof}
\end{proposition}
\begin{theorem}
\label{thm:appendix-joint-cn-curvature-window-collapse}

Suppose Assumptions~\ref{ass:main-local-row-nondegeneracy} and~\ref{ass:main-teacher-output-normalization-moments} hold.
Fix \((m,n)\) and the local window \(\mathcal U_r\) from Assumption~\ref{ass:main-local-row-nondegeneracy}.
Let \(\mathbf x_1,\ldots,\mathbf x_N\overset{\mathrm{i.i.d.}}{\sim}\mathcal N(\mathbf0,\mathbf I_n)\).
Let \(\beta_{1,C,N},\beta_{2,C,N}\) and \(\rho_{C,N}\) be as in \eqref{eq:appendix-empirical-curvature-window}--\eqref{eq:appendix-empirical-curvature-relative-width}. Then
\[
\lim_{C,N\to\infty}\mathbb P(\beta_{1,C,N}>0)=1,
\qquad
\lim_{C,N\to\infty}\mathbb P(\rho_{C,N}>\varepsilon)=0
\quad\text{for every }\varepsilon>0,
\]
and
\[
\lim_{M\to\infty}\limsup_{C,N\to\infty}
\mathbb P\left(
\frac{\rho_{C,N}}{\zeta_N+C^{-1/2}}>M
\right)=0.
\]

\begin{proof}

By Proposition~\ref{prop:appendix-finite-cn-endpoint-control}, $B_N\le 1/2+\zeta_N$.
Lemma~\ref{lem:appendix-uniform-local-halfspace-moments} gives \(\zeta_N\to0\) in probability, so \(B_N\) is tight.
Lemma~\ref{lem:appendix-teacher-output-gram-coherence-rate} gives tightness of \(C^{1/2}\epsilon_C\).
Consequently \(C^{1/2}\epsilon_CB_N\) is tight and \(\epsilon_CB_N\to0\) in probability.
Therefore \(\zeta_N+\epsilon_CB_N\to0\) in probability.
Hence the event
\[
\mathcal E_{C,N} :=\{\zeta_N+\epsilon_C B_N\le1/4\}
\]
holds with probability tending to one. On this event, Proposition~\ref{prop:appendix-finite-cn-endpoint-control} gives $\beta_{1,C,N}\ge\kappa_C/4>0$, so $\rho_{C,N}$ is well-defined, and
\[
\rho_{C,N}
\le
4(\zeta_N+\epsilon_C B_N).
\]
Thus, for every \(\varepsilon>0\),
\[
\lim_{C,N\to\infty}\mathbb P(\rho_{C,N}>\varepsilon)=0 .
\]
Also, on \(\mathcal E_{C,N}\),
\[
\frac{\rho_{C,N}}{\zeta_N+C^{-1/2}}
\le
4\frac{\zeta_N+\epsilon_C B_N}{\zeta_N+C^{-1/2}}
\le
4(1+C^{1/2}\epsilon_CB_N).
\]
Since \(C^{1/2}\epsilon_CB_N\) is tight and \(\mathbb P(\mathcal E_{C,N})\to1\), the displayed ratio is tight in the explicit sense stated above. This proves the theorem.

\end{proof}
\end{theorem}
\begin{lemma}
\label{lem:appendix-empirical-fixed-mask-loss-equivalence}

Fix $\mathbf E_W\in\mathcal U_r$. Assume that the linear path $\tau\mapsto \tau\mathbf E_W$ satisfies the empirical fixed-mask condition, namely $(\mathbf W_\star+\tau\mathbf E_W)_{i,:}\mathbf x_t\ne0$ for all $\tau\in[0,1]$, $i\in[m]$, and $t\in[N]$. If, for all $\tau\in[0,1]$,
\[
\beta_{1,C,N}\|\mathbf E_W\|_F^2
\le
\operatorname{vec}(\mathbf E_W)^\top
\mathbf H^{WW}_{C,N}(\tau\mathbf E_W)
\operatorname{vec}(\mathbf E_W)
\le
\beta_{2,C,N}\|\mathbf E_W\|_F^2 ,
\]
then
\[
\frac{\beta_{1,C,N}}{2}\|\mathbf E_W\|_F^2
\le
\widehat{\mathcal L}_{b,C}(\mathbf E_W)
\le
\frac{\beta_{2,C,N}}{2}\|\mathbf E_W\|_F^2 .
\]

\begin{proof}

Write
\[
\Delta h(\mathbf x):=\sigma((\mathbf W_\star+\mathbf E_W)\mathbf x)-\sigma(\mathbf W_\star \mathbf x).
\]
By the definition of the \(\mathbf W\)-block empirical loss,
\[
\begin{aligned}
\widehat{\mathcal L}_{b,C}(\mathbf E_W)
&\stackrel{(\circ)}{=}
\frac12
\mathbb E_N
\|\mathbf V_{\star,C}\Delta h(\mathbf x)\|_2^2
\\
&\stackrel{(\star)}{=}
\frac12
\mathbb E_N[
\Delta h(\mathbf x)^\top \mathbf K_C\Delta h(\mathbf x)] .
\end{aligned}
\]
Here, $(\circ)$ uses \eqref{eq:appendix-fixed-mask-loss}; $(\star)$ uses $\mathbf K_C=\mathbf V_{\star,C}^\top \mathbf V_{\star,C}$.

Define the one-dimensional path function
\[
\widehat f(\tau):=\widehat{\mathcal L}_{b,C}(\tau\mathbf E_W),
\qquad
\tau\in[0,1].
\]
The fixed-mask condition ensures that no pre-activation crosses a ReLU kink along the path. Therefore, for each $\mathbf x_t$ and $i$, the indicator
\[
\mathbf 1_i(\tau,\mathbf E_W;\mathbf x_t)
=
\mathbf 1_i(0,0;\mathbf x_t)
\]
is constant on $\tau\in[0,1]$, and
\[
\begin{aligned}
\sigma((\mathbf W_\star+\tau\mathbf E_W)_{i,:}\mathbf x_t)
-\sigma((\mathbf W_\star)_{i,:}\mathbf x_t)
&\stackrel{(\circ)}{=}
\tau\,\mathbf 1_i(0,0;\mathbf x_t)(\mathbf E_W)_{i,:}\mathbf x_t .
\end{aligned}
\]
Here, $(\circ)$ uses the first-order expression of ReLU within the same linear region under the fixed-mask condition; equivalently, $\mathbf 1_i(0,0;\mathbf x_t)$ can be written as any $\mathbf 1_i(\tau,\mathbf E_W;\mathbf x_t)$ along the path. Thus $\widehat f(\tau)$ is a quadratic function of $\tau$, and within this empirical fixed-mask region,
\[
\begin{aligned}
\widehat f''(\tau)
&\stackrel{(\circ)}{=}
\operatorname{vec}(\mathbf E_W)^\top
\mathbf H^{WW}_{C,N}(\tau\mathbf E_W)
\operatorname{vec}(\mathbf E_W),
\\
\mathbf H^{WW}_{C,N}(\tau\mathbf E_W)
&\stackrel{(\star)}{=}
\nabla^2_{\mathbf E_W}\widehat{\mathcal L}_{b,C}(\tau\mathbf E_W).
\end{aligned}
\]
Here, $(\circ)$ is the second directional derivative along the direction $\mathbf E_W$; within the empirical fixed-mask cell, $(\star)$ gives the exact Hessian identity of the \(\mathbf W\)-block restriction with respect to the $\mathbf E_W$ variable, rather than the full \(\mathbf W,\mathbf V\)-Hessian.

The realizable setting gives
\[
\widehat f(0)=0,
\qquad
\widehat f'(0)=0 .
\]
By the one-dimensional integral form of Taylor's formula,
\[
\begin{aligned}
\widehat{\mathcal L}_{b,C}(\mathbf E_W)
&=
\widehat f(1)-\widehat f(0)-\widehat f'(0)
\\
&\stackrel{(\circ)}{=}
\int_0^1(1-\tau)\widehat f''(\tau)\,d\tau
\\
&\stackrel{(\star)}{=}
\int_0^1(1-\tau)
\operatorname{vec}(\mathbf E_W)^\top
\mathbf H^{WW}_{C,N}(\tau\mathbf E_W)
\operatorname{vec}(\mathbf E_W)\,d\tau .
\end{aligned}
\]
Here, $(\circ)$ is the integral form of Taylor's formula; $(\star)$ uses the empirical fixed-mask exact Hessian identity.

Substituting the assumed pathwise quadratic-form window gives
\[
\begin{aligned}
\widehat{\mathcal L}_{b,C}(\mathbf E_W)
&\stackrel{(\circ)}{\le}
\int_0^1(1-\tau)\beta_{2,C,N}\|\mathbf E_W\|_F^2\,d\tau
\\
&\stackrel{(\star)}{=}
\frac{\beta_{2,C,N}}{2}\|\mathbf E_W\|_F^2 ,
\end{aligned}
\]
Moreover,
\[
\begin{aligned}
\widehat{\mathcal L}_{b,C}(\mathbf E_W)
&\stackrel{(\circ)}{\ge}
\int_0^1(1-\tau)\beta_{1,C,N}\|\mathbf E_W\|_F^2\,d\tau
\\
&\stackrel{(\star)}{=}
\frac{\beta_{1,C,N}}{2}\|\mathbf E_W\|_F^2 .
\end{aligned}
\]
Here, $(\circ)$ uses the curvature quadratic-form bound along the path; $(\star)$ uses $\int_0^1(1-\tau)\,d\tau=1/2$. This proves the lemma.
\end{proof}
\end{lemma}

Therefore, the loss equivalence is stated only on empirical fixed-mask paths; once outside this cell, Theorem~\ref{thm:appendix-joint-cn-curvature-window-collapse} uses only pointwise curvature quadratic-form bounds.

\begin{corollary}
\label{cor:appendix-empirical-fixed-mask-loss-equivalence-collapsed-window}

Under the conditions of Theorem~\ref{thm:appendix-joint-cn-curvature-window-collapse}, fix $\mathbf E_W\in\mathcal U_r$. If the linear path $\tau\mapsto \tau\mathbf E_W$ satisfies the empirical fixed-mask condition, then on the event $\zeta_N+\epsilon_C B_N\le1/4$,
\[
\frac{\beta_{1,C,N}}{2}\|\mathbf E_W\|_F^2
\le
\widehat{\mathcal L}_{b,C}(\mathbf E_W)
\le
\frac{\beta_{2,C,N}}{2}\|\mathbf E_W\|_F^2 .
\]
The relative-width limits in Theorem~\ref{thm:appendix-joint-cn-curvature-window-collapse} also hold.

\begin{proof}

Proposition~\ref{prop:appendix-finite-cn-endpoint-control} gives the pathwise quadratic-form window required by Lemma~\ref{lem:appendix-empirical-fixed-mask-loss-equivalence}. Substituting this window into Lemma~\ref{lem:appendix-empirical-fixed-mask-loss-equivalence} gives the empirical fixed-mask loss equivalence. The relative-width conclusion is exactly Theorem~\ref{thm:appendix-joint-cn-curvature-window-collapse}. This proves the corollary.

\end{proof}
\end{corollary}

\begin{remark}
The argument uses three curvature objects. First, $\mathbf H^{WW}_{C,N}(\tau\mathbf E)$ is the exact Hessian of the empirical fixed-mask \(\mathbf W\)-block loss within a fixed ReLU cell, so Lemma~\ref{lem:appendix-empirical-fixed-mask-loss-equivalence} converts the curvature window into empirical loss equivalence. Second, the population analogue
\[
\mathbf H^{WW}_{C,ik}(\tau\mathbf E)
:=
(\mathbf K_C)_{ik}
\mathbb E_{\mathbf x}[
\mathbf 1_i(\tau,\mathbf E;\mathbf x)\mathbf 1_k(\tau,\mathbf E;\mathbf x)\mathbf x\mathbf x^\top],
\qquad
\mathbf H^{WW}_{C}(\tau\mathbf E):=[\mathbf H^{WW}_{C,ik}(\tau\mathbf E)]_{i,k=1}^m
\]
serves as a population Gaussian Gauss--Newton curvature proxy. The exact Hessian statement is the empirical fixed-mask identity above. Third, Theorem~\ref{thm:appendix-joint-cn-curvature-window-collapse} uses the effective Gram matrix \(\mathbf K_C=\mathbf V_{\star,C}^{\top}\mathbf V_{\star,C}\) without \(C\)-normalization; hence \((\mathbf K_C)_{ii}=\kappa_C\) corresponds to a curvature center of \(\kappa_C/2\), while the relative window collapse is controlled by \(\epsilon_C+\zeta_N\).
\end{remark}

\section{Representation Ordering under Curvature Collapse}
\label{appendix:representation-ordering-curvature-collapse}

This section derives representation ordering from the fixed-mask curvature collapse in Theorem~\ref{thm:appendix-joint-cn-curvature-window-collapse}. The first step gives a deterministic transfer result: curvature collapse implies a radial gradient perturbation, a support-deficit bound, and an exact normalized energy identity. The second step turns this deterministic identity, under an MP-style spectral profile, into an ordering of early normalized energy dissipation across representation shapes.

\medskip\noindent\textbf{Conventions for This Section.}

Here \(b\) denotes the parameter block under study; in the main setting of this paper, it is the \(\mathbf W\)-block. All empirical quantities use the full-batch empirical average \(\mathbb E_N\), and \(\widehat{\mathcal L}_{b,C}\) is the fixed-mask loss in \eqref{eq:appendix-fixed-mask-loss}. Let \(\mathcal B_b\) be the Euclidean parameter space of this parameter block, and let its dimension be \(d=p_Tq_T\).

This section involves two kinds of asymptotic quantities: $(C,N)$ and $d$. The pair $(C,N)$ corresponds to the relative collapse of the empirical fixed-mask curvature window in Theorem~\ref{thm:appendix-joint-cn-curvature-window-collapse}, while $d=p_Tq_T\to\infty$ corresponds to the spectral profile / MP-style ordering of different matrix representations.
We reserve \(\mathbf D_T\) for the pulled-back polar direction and write the scalar nuclear-to-Frobenius gain explicitly as \(\|\mathbf A\|_*/\|\mathbf A\|_F^2\).

Fix \(K<\infty\). All asymptotic statements below are understood with \(K\) fixed. For each fixed \(k=0,\ldots,K-1\), as long as \(\mathbf E_k\in\mathcal U_r\) and the empirical fixed-mask path condition from \(0\) to \(\mathbf E_k\) holds, Theorem~\ref{thm:appendix-joint-cn-curvature-window-collapse} gives the corresponding local window \(\beta_{1,C,N,k}\), \(\beta_{2,C,N,k}\), \(\lambda_{C,N,k}\), and \(\rho_{C,N,k}\). Since \(K\) is fixed, finitely many high-probability events can be intersected directly; in particular, for every \(\varepsilon>0\),
\[
\lim_{C,N\to\infty}
\mathbb P\left(\max_{0\le k<K}\rho_{C,N,k}>\varepsilon\right)=0 .
\]

\begin{lemma}
\label{lem:appendix-support-deficit-radial-perturbation}

Let $\mathbf A\in\mathbb R^{p\times q}$ and $\mathbf G=\lambda \mathbf A+\boldsymbol\Delta$, where $\lambda>0$. Take $\mathbf P\in\arg\max_{\|\mathbf Z\|_{\mathrm{op}}\le1}\langle \mathbf G,\mathbf Z\rangle$, and write $\delta(\mathbf A,\mathbf G)=\|\mathbf A\|_*-\langle \mathbf A,\mathbf P\rangle$. Let $\mathcal K=\{\mathbf Z:\|\mathbf Z\|_{\mathrm{op}}\le1\}$, $\mathcal S(\mathbf A)=\partial\|\mathbf A\|_*=\arg\max_{\mathbf Z\in\mathcal K}\langle \mathbf A,\mathbf Z\rangle$, and $h_{\mathcal S(\mathbf A)}(\boldsymbol\Delta)=\sup_{\mathbf Q\in\mathcal S(\mathbf A)}\langle\boldsymbol\Delta,\mathbf Q\rangle$. Then
\[
0\le \delta(\mathbf A,\mathbf G)
\le
\frac{1}{\lambda}
\left(
\|\boldsymbol\Delta\|_*-
h_{\mathcal S(\mathbf A)}(\boldsymbol\Delta)
\right)
\le
\frac{2\|\boldsymbol\Delta\|_*}{\lambda}.
\]

\begin{proof}

By the support-function representation of the nuclear norm,
\[
\|\mathbf X\|_*
=
\max_{\mathbf Z\in\mathcal K}\langle \mathbf X,\mathbf Z\rangle .
\]
Since $\mathbf P\in\mathcal K$,
\[
\begin{aligned}
\langle \mathbf A,\mathbf P\rangle
\stackrel{(\circ)}{\le}
\max_{\mathbf Z\in\mathcal K}\langle \mathbf A,\mathbf Z\rangle
\stackrel{(\star)}{=}
\|\mathbf A\|_* .
\end{aligned}
\]
Therefore $\delta(\mathbf A,\mathbf G)\ge0$.

On the other hand, for any $\mathbf Q\in\mathcal S(\mathbf A)$, since $\mathbf P$ is a maximizer for $\mathbf G$,
\[
\langle \mathbf G,\mathbf P\rangle\ge \langle \mathbf G,\mathbf Q\rangle .
\]
Substituting $\mathbf G=\lambda \mathbf A+\boldsymbol\Delta$ and rearranging gives
\[
\begin{aligned}
\lambda\bigl(\|\mathbf A\|_*-\langle \mathbf A,\mathbf P\rangle\bigr)
&\stackrel{(\circ)}{=}
\lambda\langle \mathbf A,\mathbf Q-\mathbf P\rangle
\\
&\stackrel{(\star)}{\le}
\langle\boldsymbol\Delta,\mathbf P-\mathbf Q\rangle
\\
&\stackrel{(\bullet)}{=}
\langle\boldsymbol\Delta,\mathbf P\rangle-\langle\boldsymbol\Delta,\mathbf Q\rangle .
\end{aligned}
\]
Here we used $\mathbf Q\in\mathcal S(\mathbf A)$, namely $\langle \mathbf A,\mathbf Q\rangle=\|\mathbf A\|_*$, and then expanded and rearranged the maximizer inequality.

Dividing both sides by $\lambda>0$ and taking the supremum over $\mathbf Q\in\mathcal S(\mathbf A)$, we obtain
\[
\begin{aligned}
\delta(\mathbf A,\mathbf G)
&\le
\frac1\lambda
\left(
\langle\boldsymbol\Delta,\mathbf P\rangle
-
h_{\mathcal S(\mathbf A)}(\boldsymbol\Delta)
\right)
\\
&\stackrel{(\circ)}{\le}
\frac1\lambda
\left(
\|\boldsymbol\Delta\|_*
-
h_{\mathcal S(\mathbf A)}(\boldsymbol\Delta)
\right)
\\
&\stackrel{(\star)}{\le}
\frac{2\|\boldsymbol\Delta\|_*}{\lambda}.
\end{aligned}
\]
Here we used the support-function upper bound from $\mathbf P\in\mathcal K$, and the bound $h_{\mathcal S(\mathbf A)}(\boldsymbol\Delta)\ge-\|\boldsymbol\Delta\|_*$ from $\mathcal S(\mathbf A)\subseteq\mathcal K$. This proves the lemma.

\end{proof}
\end{lemma}
\begin{proposition}
\label{prop:appendix-curvature-collapse-normalized-spectral-energy}

Under the assumptions of Theorem~\ref{thm:appendix-joint-cn-curvature-window-collapse}, fix $K<\infty$. For each $k=0,\ldots,K-1$, assume that $\mathbf E_k\in\mathcal U_r$ and that the path $\tau\mapsto \tau\mathbf E_k$ satisfies the empirical fixed-mask path condition. Let $T:\mathcal B_b\to\mathbb R^{p_T\times q_T}$ be a linear Frobenius-isometric representation, where $p_Tq_T=d$ and $p_T\le q_T$. Set $\mathbf A_k=T(\mathbf E_k)$. When $\mathbf A_k\ne0$, take
\[
\mathbf P_k\in\arg\max_{\|\mathbf Z\|_{\mathrm{op}}\le1}
\langle T(\nabla\widehat{\mathcal L}_{b,C}(\mathbf E_k)),\mathbf Z\rangle,
\]
let $\mathbf A_{k+1}=\mathbf A_k-\eta \mathbf P_k$, and write $\delta_k=\|\mathbf A_k\|_*-\langle \mathbf A_k,\mathbf P_k\rangle$.

On the local high-probability event at step $k$ in Theorem~\ref{thm:appendix-joint-cn-curvature-window-collapse}, we have
\[
T(\nabla\widehat{\mathcal L}_{b,C}(\mathbf E_k))=\lambda_{C,N,k}\mathbf A_k+\boldsymbol\Delta_k,\qquad
\|\boldsymbol\Delta_k\|_F\le \lambda_{C,N,k}\rho_{C,N,k}\|\mathbf A_k\|_F,\qquad
0\le \frac{\delta_k}{\|\mathbf A_k\|_*}\le
2\rho_{C,N,k}\left(\frac{r_T}{r_{\mathrm{eff},*}(\mathbf A_k)}\right)^{1/2},
\]
where $r_T=p_T=\min(p_T,q_T)$ and $r_{\mathrm{eff},*}(\mathbf A_k)=(\|\mathbf A_k\|_*/\|\mathbf A_k\|_F)^2$. Moreover, the normalized energy dissipation satisfies the exact identity
\[
\frac{\|\mathbf A_{k+1}\|_F^2-\|\mathbf A_k\|_F^2}{\eta\|\mathbf A_k\|_F^2}
=
-2\frac{\|\mathbf A_k\|_*}{\|\mathbf A_k\|_F^2}
+
2\frac{\delta_k}{\|\mathbf A_k\|_F^2}
+
\eta\frac{\|\mathbf P_k\|_F^2}{\|\mathbf A_k\|_F^2}.
\]
Consequently, if the support-deficit and finite-step errors both vanish in probability relative to \(\|\mathbf A_k\|_*/\|\mathbf A_k\|_F^2\), then the normalized dissipation divided by this scalar gain converges in probability to \(-2\).

\begin{proof}

The proof consists of four steps. First, we convert the curvature window in Theorem~\ref{thm:appendix-joint-cn-curvature-window-collapse} into a radial gradient perturbation. Then we push this perturbation through the representation $T$. Next, we use Lemma~\ref{lem:appendix-support-deficit-radial-perturbation} to control the support deficit. Finally, we expand the Frobenius energy after one update.

\medskip\noindent\textbf{Step 1: Obtaining a radial perturbation from the curvature window.}

At step $k$, let $\widehat{\mathbf H}_{C,N,k}$ be the exact empirical Hessian on the current empirical fixed-mask linear cell. By the empirical block-level window in Theorem~\ref{thm:appendix-joint-cn-curvature-window-collapse}, on the corresponding high-probability event there exist window constants
\[
0<\beta_{1,C,N,k}\le \beta_{2,C,N,k}
\]
such that $\widehat{\mathbf H}_{C,N,k}$, viewed as a self-adjoint operator on this block space, satisfies
\[
\beta_{1,C,N,k}\mathbf I
\preceq
\widehat{\mathbf H}_{C,N,k}
\preceq
\beta_{2,C,N,k}\mathbf I .
\]
Here $\preceq$ is the Loewner order induced by the Frobenius inner product.

Let $\lambda_{C,N,k}=(\beta_{1,C,N,k}+\beta_{2,C,N,k})/2$, and let
$\rho_{C,N,k}=(\beta_{2,C,N,k}-\beta_{1,C,N,k})/(\beta_{1,C,N,k}+\beta_{2,C,N,k})$.
Since $\nabla \widehat{\mathcal L}_{b,C}(0)=0$ and $\widehat{\mathcal L}_{b,C}$ is quadratic inside the fixed-mask cell,
\[
\begin{aligned}
\nabla \widehat{\mathcal L}_{b,C}(\mathbf E_k)
&\stackrel{(\circ)}{=}
\widehat{\mathbf H}_{C,N,k}\mathbf E_k
\\
&\stackrel{(\star)}{=}
\lambda_{C,N,k}\mathbf E_k
+
(\widehat{\mathbf H}_{C,N,k}-\lambda_{C,N,k}\mathbf I)\mathbf E_k .
\end{aligned}
\]
where \((\circ)\) uses the quadratic structure inside the fixed-mask cell together with $\nabla\widehat{\mathcal L}_{b,C}(0)=0$, and \((\star)\) adds and subtracts the centered curvature $\lambda_{C,N,k}\mathbf I$ to isolate the residual.

Define $\mathbf R_k=(\widehat{\mathbf H}_{C,N,k}-\lambda_{C,N,k}\mathbf I)\mathbf E_k$.
Then
\[
\nabla \widehat{\mathcal L}_{b,C}(\mathbf E_k)
=
\lambda_{C,N,k}\mathbf E_k+\mathbf R_k .
\]
Moreover, the spectral window gives
\[
\begin{aligned}
\|\mathbf R_k\|_F
&\stackrel{(\circ)}{=}
\|(\widehat{\mathbf H}_{C,N,k}-\lambda_{C,N,k}\mathbf I)\mathbf E_k\|_F
\\
&\stackrel{(\star)}{\le}
\frac{\beta_{2,C,N,k}-\beta_{1,C,N,k}}{2}\|\mathbf E_k\|_F
\\
&\stackrel{(\bullet)}{=}
\lambda_{C,N,k}
\frac{\beta_{2,C,N,k}-\beta_{1,C,N,k}}
{\beta_{1,C,N,k}+\beta_{2,C,N,k}}
\|\mathbf E_k\|_F
\\
&\stackrel{(\diamond)}{=}
\lambda_{C,N,k}\rho_{C,N,k}\|\mathbf E_k\|_F .
\end{aligned}
\]
where \((\circ)\) uses the definition of $\mathbf R_k$; \((\star)\) uses the spectral-radius bound $\|\widehat{\mathbf H}_{C,N,k}-\lambda_{C,N,k}\mathbf I\|_{\mathrm{op}}\le(\beta_{2,C,N,k}-\beta_{1,C,N,k})/2$; \((\bullet)\) uses the definition of $\lambda_{C,N,k}$; and \((\diamond)\) uses the definition of $\rho_{C,N,k}$.

Therefore,
\[
\|\mathbf R_k\|_F
\le
\lambda_{C,N,k}\rho_{C,N,k}\|\mathbf E_k\|_F .
\]

\medskip\noindent\textbf{Step 2: Pushing the radial perturbation to the matrix representation.}

We now push the block-space decomposition obtained above to $T(\mathcal B_b)\subseteq\mathbb R^{p_T\times q_T}$. This step only uses the linearity of $T$ and its Frobenius isometry.

Write $\mathbf A_k=T(\mathbf E_k)$, $\mathbf G_k=T(\nabla \widehat{\mathcal L}_{b,C}(\mathbf E_k))$, and $\boldsymbol\Delta_k=T(\mathbf R_k)$. All normalized quantities below are defined on steps where $\mathbf A_k\ne0$; since $T$ is a Frobenius isometry, this is equivalent to $\mathbf E_k\ne0$.

By the previous step,
\[
\begin{aligned}
\mathbf G_k
&\stackrel{(\circ)}{=}
T(\lambda_{C,N,k}\mathbf E_k+\mathbf R_k)
\\
&\stackrel{(\star)}{=}
\lambda_{C,N,k}\mathbf A_k+\boldsymbol\Delta_k .
\end{aligned}
\]
where \((\circ)\) uses the gradient decomposition from the previous step, and \((\star)\) uses the linearity of $T$ together with the definitions of $\mathbf A_k$ and $\boldsymbol\Delta_k$.

Moreover,
\[
\begin{aligned}
\|\boldsymbol\Delta_k\|_F
&\stackrel{(\circ)}{=}
\|T(\mathbf R_k)\|_F
\\
&\stackrel{(\star)}{=}
\|\mathbf R_k\|_F
\\
&\stackrel{(\bullet)}{\le}
\lambda_{C,N,k}\rho_{C,N,k}\|\mathbf E_k\|_F
\\
&\stackrel{(\diamond)}{=}
\lambda_{C,N,k}\rho_{C,N,k}\|\mathbf A_k\|_F .
\end{aligned}
\]
where \((\circ)\) uses the definition of $\boldsymbol\Delta_k$; \((\star)\) uses the Frobenius isometry of $T$; \((\bullet)\) uses the residual bound from the previous step; and \((\diamond)\) uses $\|\mathbf E_k\|_F=\|\mathbf A_k\|_F$.

Thus, on this local curvature event,
\[
\mathbf G_k=\lambda_{C,N,k}\mathbf A_k+\boldsymbol\Delta_k,
\qquad
\|\boldsymbol\Delta_k\|_F
\le
\lambda_{C,N,k}\rho_{C,N,k}\|\mathbf A_k\|_F .
\]

\medskip\noindent\textbf{Step 3: Controlling the support deficit using Lemma~\ref{lem:appendix-support-deficit-radial-perturbation}.}

The preceding display has written $\mathbf G_k$ as the main radial term $\lambda_{C,N,k}\mathbf A_k$ plus a small perturbation $\boldsymbol\Delta_k$, so Lemma~\ref{lem:appendix-support-deficit-radial-perturbation} can be applied directly.

Apply Lemma~\ref{lem:appendix-support-deficit-radial-perturbation} with
\[
\mathbf A=\mathbf A_k,
\qquad
\mathbf G=\mathbf G_k,
\qquad
\lambda=\lambda_{C,N,k},
\qquad
\boldsymbol\Delta=\boldsymbol\Delta_k,
\qquad
\mathbf P=\mathbf P_k,
\]
to obtain
\[
\begin{aligned}
0\le \delta_k
&\le
\frac1{\lambda_{C,N,k}}
\left(
\|\boldsymbol\Delta_k\|_*
-
h_{\mathcal S(\mathbf A_k)}(\boldsymbol\Delta_k)
\right)
\\
&\le
\frac{2\|\boldsymbol\Delta_k\|_*}{\lambda_{C,N,k}} .
\end{aligned}
\]

Let $r_T=p_T=\min(p_T,q_T)$. Since $\operatorname{rank}(\boldsymbol\Delta_k)\le r_T$,
\[
\|\boldsymbol\Delta_k\|_*
\le
\sqrt{r_T}\|\boldsymbol\Delta_k\|_F .
\]
Here we used the standard bound $\|\mathbf M\|_*\le\sqrt{\operatorname{rank}(\mathbf M)}\|\mathbf M\|_F$ and $\operatorname{rank}(\boldsymbol\Delta_k)\le r_T$.

Combining this with the empirical residual bound from the previous step gives
\[
\begin{aligned}
\delta_k
&\stackrel{(\circ)}{\le}
\frac{2\|\boldsymbol\Delta_k\|_*}{\lambda_{C,N,k}}
\\
&\stackrel{(\star)}{\le}
\frac{2\sqrt{r_T}\|\boldsymbol\Delta_k\|_F}{\lambda_{C,N,k}}
\\
&\stackrel{(\bullet)}{\le}
2\rho_{C,N,k}\sqrt{r_T}\|\mathbf A_k\|_F .
\end{aligned}
\]
where \((\circ)\) uses the coarse upper bound from Lemma~\ref{lem:appendix-support-deficit-radial-perturbation}; \((\star)\) uses the rank-to-Frobenius nuclear-norm control; and \((\bullet)\) uses the residual bound from the previous step.

By $r_{\mathrm{eff},*}(\mathbf A_k)=(\|\mathbf A_k\|_*/\|\mathbf A_k\|_F)^2$,
\[
\|\mathbf A_k\|_*
=
\sqrt{r_{\mathrm{eff},*}(\mathbf A_k)}\|\mathbf A_k\|_F .
\]
Hence
\[
\begin{aligned}
\frac{\delta_k}{\|\mathbf A_k\|_*}
&\le
2\rho_{C,N,k}
\frac{\sqrt{r_T}\|\mathbf A_k\|_F}
{\sqrt{r_{\mathrm{eff},*}(\mathbf A_k)}\|\mathbf A_k\|_F}
\\
&=
2\rho_{C,N,k}
\left(
\frac{r_T}{r_{\mathrm{eff},*}(\mathbf A_k)}
\right)^{1/2}.
\end{aligned}
\]
This error is controlled by the effective nuclear rank.

\medskip\noindent\textbf{Step 4: Expanding the normalized energy identity.}

The support-deficit bound has completed the conversion between the gradient maximizer and the nuclear support. It remains only to expand the algebra of one update for the main identity.

From the update rule
\[
\mathbf A_{k+1}=\mathbf A_k-\eta \mathbf P_k
\]
we directly expand:
\[
\begin{aligned}
\|\mathbf A_{k+1}\|_F^2
&=
\|\mathbf A_k-\eta \mathbf P_k\|_F^2
\\
&=
\|\mathbf A_k\|_F^2
-2\eta\langle \mathbf A_k,\mathbf P_k\rangle
+\eta^2\|\mathbf P_k\|_F^2
\\
&=
\|\mathbf A_k\|_F^2
-2\eta\|\mathbf A_k\|_*
+2\eta \delta_k
+\eta^2\|\mathbf P_k\|_F^2 .
\end{aligned}
\]
The last step follows from $\delta_k=\|\mathbf A_k\|_*-\langle \mathbf A_k,\mathbf P_k\rangle$.

Therefore we obtain the exact spectral energy identity
\[
\frac{\|\mathbf A_{k+1}\|_F^2-\|\mathbf A_k\|_F^2}{\eta}
=
-2\|\mathbf A_k\|_*
+
2\delta_k
+
\eta\|\mathbf P_k\|_F^2 .
\]
Dividing this display by $\|\mathbf A_k\|_F^2$, we obtain the normalized energy dissipation:
\[
\frac{
\|\mathbf A_{k+1}\|_F^2-\|\mathbf A_k\|_F^2
}{
\eta\|\mathbf A_k\|_F^2
}
=
-2\frac{\|\mathbf A_k\|_*}{\|\mathbf A_k\|_F^2}
+
2\frac{\delta_k}{\|\mathbf A_k\|_F^2}
+
\eta\frac{\|\mathbf P_k\|_F^2}{\|\mathbf A_k\|_F^2}.
\]

Here \(-2\|\mathbf A_k\|_*/\|\mathbf A_k\|_F^2\) is the representation-dependent main dissipation term,
\(2\frac{\delta_k}{\|\mathbf A_k\|_F^2}\) is the support-deficit error, and
\(\eta\frac{\|\mathbf P_k\|_F^2}{\|\mathbf A_k\|_F^2}\) is the finite-step error.

When $\operatorname{Orth}$ returns a support maximizer satisfying $\|\mathbf P_k\|_{\mathrm{op}}\le1$, we always have \(\|\mathbf P_k\|_F^2 \le r_T.\)
Since $\|\mathbf P_k\|_{\mathrm{op}}\le1$ and the matrix has at most $r_T$ singular values, the finite-step error can be controlled by \(\eta\frac{r_T}{\|\mathbf A_k\|_F^2}.\)

Finally, we write the support-deficit estimate relative to the main term. By the support-deficit bound from Step 3,
\[
\frac{\delta_k}{\|\mathbf A_k\|_F^2}
=
\frac{\delta_k}{\|\mathbf A_k\|_*}
\frac{\|\mathbf A_k\|_*}{\|\mathbf A_k\|_F^2}
\le
2\rho_{C,N,k}
\left(
\frac{r_T}{r_{\mathrm{eff},*}(\mathbf A_k)}
\right)^{1/2}
\frac{\|\mathbf A_k\|_*}{\|\mathbf A_k\|_F^2}.
\]
Therefore, if
\[
\rho_{C,N,k}
\left(
\frac{r_T}{r_{\mathrm{eff},*}(\mathbf A_k)}
\right)^{1/2}
\xrightarrow{p}0,
\]
then the support-deficit term satisfies
\[
\frac{\delta_k}{\|\mathbf A_k\|_*}
\xrightarrow{p}0.
\]
If, in addition,
\begin{equation}
\label{theorm:condition}
\frac{\eta\|\mathbf P_k\|_F^2}
{\|\mathbf A_k\|_*}
\xrightarrow{p}0,
\end{equation}
then the exact identity gives
\[
\frac{
\|\mathbf A_{k+1}\|_F^2-\|\mathbf A_k\|_F^2
}{
\eta\|\mathbf A_k\|_*
}
\xrightarrow{p}
-2 .
\]
This gives the representation-dependent normalized energy identity under curvature collapse, with error controlled by the effective nuclear rank. Proposition~\ref{prop:appendix-curvature-collapse-normalized-spectral-energy} is proved.

\end{proof}
\end{proposition}
\begin{theorem}
\label{thm:appendix-mp-representation-ordering-early-spectral-dissipation}

Suppose Assumptions~\ref{ass:main-local-row-nondegeneracy} and~\ref{ass:main-teacher-output-normalization-moments} hold.
Let \(T:\mathcal B_b\to\mathbb R^{p_T\times q_T}\) be a linear Frobenius-isometric representation with \(p_Tq_T=d\) and \(p_T\le q_T\).
Let \(\mathbf E_0\in\mathcal U_r\), assume that the path \(\tau\mapsto \tau\mathbf E_0\) satisfies the empirical fixed-mask path condition, set \(\mathbf A_0=T(\mathbf E_0)\), and choose \(\mathbf P_0\) as in Proposition~\ref{prop:appendix-curvature-collapse-normalized-spectral-energy} at \(k=0\).
Assume, in addition, that the MP spectral profile in Assumption~\ref{ass:main-mp-spectral-profile-finite-step} hold for this \(T,\mathbf A_0,\mathbf P_0\), and let \(\mathbf A_1=\mathbf A_0-\eta\mathbf P_0\). Then
\[
\lim_{d\to\infty}
\frac{\sigma_0d^{1/4}}{\beta(\gamma_T)\gamma_T^{1/4}}
\frac{\|\mathbf A_0\|_*}{\|\mathbf A_0\|_F^2}
=1.
\]
Moreover, for every \(\varepsilon>0\),
\[
\lim_{d\to\infty}\lim_{C,N\to\infty}
\mathbb P\left(
\left|
\frac{\sigma_0d^{1/4}}{\beta(\gamma_T)\gamma_T^{1/4}}
\frac{\|\mathbf A_1\|_F^2-\|\mathbf A_0\|_F^2}
{\eta\|\mathbf A_0\|_F^2}
+2
\right|>\varepsilon
\right)=0.
\]
Therefore, the leading-order early normalized dissipation is ordered by $f(\gamma)=\beta(\gamma)\gamma^{1/4}$.

\begin{proof}

By the MP-style profile,
\[
\begin{aligned}
\frac1{p_T}\|\mathbf A_0\|_*
&=
\frac1{p_T}
\sum_i\sigma_i(\mathbf A_0)=
\sigma_0\sqrt{q_T}
\left(
\frac1{p_T}\sum_i\widetilde\sigma_i^{(T)}
\right)
=
\sigma_0\sqrt{q_T}\,\beta(\gamma_T)(1+o(1)),
\\[1mm]
\frac1{p_T}\|\mathbf A_0\|_F^2
&=
\frac1{p_T}
\sum_i\sigma_i(\mathbf A_0)^2=
\sigma_0^2 q_T
\left(
\frac1{p_T}\sum_i(\widetilde\sigma_i^{(T)})^2
\right)
=
\sigma_0^2 q_T(1+o(1)) .
\end{aligned}
\]
Here and below, $o(1)$ refers to the MP-style spectral asymptotic as $d\to\infty$, understood either at a fixed representation aspect ratio or along a specified representation sequence.

Therefore,
\[
\begin{aligned}
\frac{\|\mathbf A_0\|_*}
{\|\mathbf A_0\|_F^2}=
\frac{
p_T\sigma_0\sqrt{q_T}\,\beta(\gamma_T)(1+o(1))
}{
p_T\sigma_0^2q_T(1+o(1))
}
=
\frac{\beta(\gamma_T)}{\sigma_0\sqrt{q_T}}(1+o(1)) .
\end{aligned}
\]

The same MP-style profile gives
\[
\|\mathbf A_0\|_*
=
p_T\sigma_0\sqrt{q_T}\,\beta(\gamma_T)(1+o(1)),
\qquad
\|\mathbf A_0\|_F^2
=
\sigma_0^2p_Tq_T(1+o(1)) .
\]
Hence
\[
\begin{aligned}
r_{\mathrm{eff},*}(\mathbf A_0)
=
\left(
\frac{\|\mathbf A_0\|_*}{\|\mathbf A_0\|_F}
\right)^2
=
\left(
\frac{
p_T\sigma_0\sqrt{q_T}\,\beta(\gamma_T)(1+o(1))
}{
\sigma_0\sqrt{p_Tq_T}(1+o(1))
}
\right)^2
=
p_T\beta(\gamma_T)^2(1+o(1)) .
\end{aligned}
\]
Since $r_T=p_T$,
\[
\frac{r_T}{r_{\mathrm{eff},*}(\mathbf A_0)}
=
\frac{1+o(1)}{\beta(\gamma_T)^2}.
\]
Using the uniform lower bound on \(\beta(\gamma_T)\), this ratio is \(O(1)\). The support-deficit bound from Proposition~\ref{prop:appendix-curvature-collapse-normalized-spectral-energy} at \(k=0\) gives
\[
\frac{\delta_0}{\|\mathbf A_0\|_*}
\le
2\rho_{C,N,0}
\left(
\frac{r_T}{r_{\mathrm{eff},*}(\mathbf A_0)}
\right)^{1/2}
=
\frac{2\rho_{C,N,0}}{\beta(\gamma_T)}(1+o(1)).
\]
Therefore, for every $\varepsilon>0$,
\[
\lim_{d\to\infty}
\lim_{C,N\to\infty}
\mathbb P\left(
\frac{\delta_0}{\|\mathbf A_0\|_*}>\varepsilon
\right)
=0.
\]
Thus
\[
\lim_{d\to\infty}\lim_{C,N\to\infty}
\mathbb P\left(
\frac{\delta_0}{\|\mathbf A_0\|_*}
>\varepsilon
\right)=0
\quad\text{for every }\varepsilon>0.
\]

Moreover, since $q_T=\sqrt{d/\gamma_T}$,
\[
\frac{\beta(\gamma_T)}{\sigma_0\sqrt{q_T}}
=
\frac{
\beta(\gamma_T)\gamma_T^{1/4}
}{
\sigma_0d^{1/4}
}(1+o(1)).
\]
The effective-rank estimate above already shows that the support-deficit
term is lower order in the sequential probability limit. Moreover, since \(\|\mathbf P_0\|_F^2\le p_T\),
\[
\eta\frac{\|\mathbf P_0\|_F^2}
{\|\mathbf A_0\|_F^2}
\le
\eta\frac{p_T}
{\sigma_0^2p_Tq_T(1+o(1))}
=
\eta\frac{\gamma_T^{1/2}+o(1)}
{\sigma_0^2d^{1/2}}.
\]
Therefore, for any fixed learning rate \(\eta\), if
\(\beta(\gamma_T)\) remains bounded away from zero along the representation
sequence, then
\[
\frac{
\eta\|\mathbf P_0\|_F^2
}{
\|\mathbf A_0\|_*
}
\le
\eta\frac{(\gamma_T^{1/2}+o(1))\sigma_0d^{1/4}}
{\sigma_0^2d^{1/2} \beta(\gamma_T)\gamma_T^{1/4} (1+o(1))}
\to0.
\]
Combining these two estimates with Proposition~\ref{prop:appendix-curvature-collapse-normalized-spectral-energy} condition (\ref{theorm:condition}) gives, for every \(\varepsilon>0\),
\[
\lim_{d\to\infty}\lim_{C,N\to\infty}
\mathbb P\left(
\left|
\frac{\sigma_0d^{1/4}}{\beta(\gamma_T)\gamma_T^{1/4}}
\frac{\|\mathbf A_1\|_F^2-\|\mathbf A_0\|_F^2}
{\eta\|\mathbf A_0\|_F^2}
+2
\right|>\varepsilon
\right)=0.
\]
This also shows that the leading-order dissipation is ordered by $f(\gamma_T)=\beta(\gamma_T)\gamma_T^{1/4}$. The theorem is proved.

\end{proof}
\end{theorem}
\begin{remark}
\label{rem:appendix-how-to-read-mp-ordering}

The MP-style profile is an isotropic spectral baseline for isolating the shape dependence of the represented matrix from task-specific alignment effects.
Under the collapsed curvature window, Proposition~\ref{prop:appendix-curvature-collapse-normalized-spectral-energy} reduces the leading energy drop to the scalar gain
\[
\frac{\|\mathbf A_0\|_*}{\|\mathbf A_0\|_F^2}.
\]
The MP profile evaluates this gain as
\[
\frac{\|\mathbf A_0\|_*}{\|\mathbf A_0\|_F^2}
\sim
\frac{\beta(\gamma_T)\gamma_T^{1/4}}{\sigma_0d^{1/4}} .
\]
Thus the factor \(\beta(\gamma_T)\gamma_T^{1/4}\) is the shape-dependent part of the leading normalized dissipation.
At a fixed infinitesimal learning rate, larger values give faster early relative decrease; conversely, a representation with smaller \(\beta(\gamma_T)\gamma_T^{1/4}\) needs a larger learning rate to match the same early decrease.
\end{remark}
\begin{corollary}
\label{cor:appendix-learning-rate-calibration-representations}

Under the assumptions of Theorem~\ref{thm:appendix-mp-representation-ordering-early-spectral-dissipation}, if two representations $T_1,T_2$ satisfy $f(\gamma_{T_1})>f(\gamma_{T_2})$, then $T_1$ has larger leading-order early normalized dissipation. To match a target infinitesimal relative-drop sequence $\tau_d\to0$, namely
\[
\frac{\|\mathbf A_0\|_F^2-\|\mathbf A_1\|_F^2}{\|\mathbf A_0\|_F^2}
\sim
\tau_d,
\]
the learning rate should satisfy
\[
2\eta_T\frac{\|\mathbf A_0\|_*}{\|\mathbf A_0\|_F^2}\sim\tau_d,
\qquad
\eta_T\sim
\frac{\tau_d}{2}\frac{\sigma_0\sqrt{q_T}}{\beta(\gamma_T)}.
\]

\begin{proof}
The first statement follows directly by comparing the leading-order formula in Theorem~\ref{thm:appendix-mp-representation-ordering-early-spectral-dissipation}. The second statement follows from the leading term \(2\eta_T\|\mathbf A_0\|_*/\|\mathbf A_0\|_F^2\) of the relative drop and from \(\|\mathbf A_0\|_*/\|\mathbf A_0\|_F^2=\beta(\gamma_T)/(\sigma_0\sqrt{q_T})(1+o(1))\).

\end{proof}
\end{corollary}
\begin{remark}
\label{rem:appendix-endpoint-skinny-representations}

The square and vector endpoints show the scale separation directly. For the square endpoint, \(p_T=q_T=d^{1/2}\) and \(\gamma_T=1\), so the MP formula gives
\[
\frac{\|\mathbf A_0\|_*}{\|\mathbf A_0\|_F^2}
=
\frac{\beta(1)}{\sigma_0d^{1/4}}(1+o(1)) .
\]
For the vector endpoint, \(p_T=1\), \(q_T=d\), and \(\|\mathbf A_0\|_*=\|\mathbf A_0\|_F=\sigma_0\sqrt{d}(1+o(1))\). Hence
\[
\frac{\|\mathbf A_0\|_*}{\|\mathbf A_0\|_F^2}
=
\frac1{\|\mathbf A_0\|_F}
=
\frac1{\sigma_0\sqrt{d}}(1+o(1)).
\]
Consequently, the square endpoint has larger scalar gain by the factor
\[
\frac{\beta(1)(\sigma_0d^{1/4})^{-1}}
{(\sigma_0\sqrt d)^{-1}}
=
\beta(1)d^{1/4}(1+o(1)).
\]
Under the standard Gaussian/MP normalization, \(\beta(1)=8/(3\pi)\), so this endpoint separation is \(\Theta(d^{1/4})\).

For skinny representations, reducing the represented short side also decreases \(\gamma_T=p_T/q_T\). On any interval where \(\beta(\gamma)\gamma^{1/4}\) is increasing in \(\gamma\), this reduction decreases the leading-order scalar gain \(\|\mathbf A_0\|_*/\|\mathbf A_0\|_F^2\). Therefore, on such an interval, the fixed-\(\eta\), early-stage normalized dissipation follows the short-side ordering
\[
\text{large-short-side matrix representations}
\succ
\text{skinnier matrix representations}
\succ
\text{vector}.
\]

\end{remark}
\begin{remark}
\label{rem:appendix-finite-step-scaling}

Under the same MP scale, the finite-step term can be estimated explicitly:
\[
\begin{aligned}
\eta\frac{\|\mathbf P_0\|_F^2}{\|\mathbf A_0\|_F^2}
&\le
\eta\frac{p_T}{\sigma_0^2p_Tq_T(1+o(1))}
\\
&=
\eta\frac{1+o(1)}{\sigma_0^2q_T}
\\
&=
\eta
\frac{\gamma_T^{1/2}+o(1)}
{\sigma_0^2d^{1/2}} .
\end{aligned}
\]
Here we used \(\|\mathbf P_0\|_F^2\le p_T\) and, under the MP scale, \(\|\mathbf A_0\|_F^2=\sigma_0^2p_Tq_T(1+o(1))\). Therefore the finite-step term has its size on the \((d,\gamma_T)\) scale. If \(\eta/(\sigma_0d^{1/4})\to0\) and \(\beta(\gamma_T)\) is uniformly bounded away from zero along the representation sequence, then \(\eta\|\mathbf P_0\|_F^2/\|\mathbf A_0\|_F^2=o(\|\mathbf A_0\|_*/\|\mathbf A_0\|_F^2)\).
\end{remark}

%% file: 6_appendix.tex
\section{Experimental Details}
\label{app:experimental-details}

Here, we give the experimental settings for the LLaMA2 comparisons in Section~\ref{sec:exp} and for the controlled teacher--student diagnostics. Unless stated otherwise, each polar comparison uses the same architecture, data stream, optimizer rule, training budget, and evaluation schedule; only the matrix representation passed to the polar map is changed.

For the LLaMA2-130M comparison, we train a LLaMA-style decoder-only Transformer~\citep{Touvron2023Llama2O} with \(12\) layers, hidden width \(768\), \(12\) attention heads of dimension \(64\), MLP hidden width \(2048\), RMSNorm with \(\epsilon=10^{-6}\), tied token embeddings, no bias parameters, and no dropout. The model has \(123.59\)M trainable parameters. Training uses the FineWeb10B stream~\citep{penedo2024fineweb}, which contains \(10{,}255{,}324{,}043\) training tokens and \(100{,}000{,}000\) validation tokens, and applies the GPT-2 tokenizer~\citep{radford2019language} with vocabulary size \(50{,}304\). We use sequence length \(4096\), global batch size \(128\), \(4960\) training steps, \(496\) warmup steps, cosine learning-rate decay, weight decay \(0.1\), gradient clipping at norm \(1.0\), and bfloat16 training. These runs use four NVIDIA RTX PRO 6000 Blackwell GPUs (96GB). Muon/Muse updates the matrix modules summarized in Table~\ref{tab:llama130m_matrix_modules} with momentum \(0.95\), no Nesterov correction, \(5\) Newton--Schulz steps, and the standard Muon multiplier \(0.2\sqrt{\max(m,n)}\) on the matrix learning rate~\citep{jordan6muon,liu2025muon}. The tied token embedding/language-model head, normalization parameters, and all other non-matrix parameters are optimized by AdamW~\citep{kingma2014adam,ilya2019adamw} with \((\beta_1,\beta_2)=(0.9,0.95)\), bias correction, and \(\epsilon=10^{-8}\). Validation is run every \(100\) steps. Table~\ref{tab:llama130m_lr_sweep_details} reports the learning-rate grids and the selected learning-rate setting for each method, and Table~\ref{tab:llama130m_compute_diagnostics} reports the corresponding system diagnostics. The reported final losses and selected trajectories are averages over three random seeds; the training-loss trajectories use a \(100\)-step moving average of minibatch losses.

For the LLaMA2-600M comparison, we train a \(32\)-layer LLaMA-style decoder-only Transformer~\citep{Touvron2023Llama2O} with hidden width \(1024\), MLP hidden width \(4096\), RMSNorm, tied token embeddings, no bias parameters, and no dropout. The run uses the C4 stream with T5-base tokenization~\citep{raffel2020exploring}, sequence length \(4096\), global batch size \(320\), and \(9600\) training steps, corresponding to \(12{,}582{,}912{,}000\) training tokens. The warmup--stable--decay schedule uses \(800\) warmup steps and \(960\) decay steps. All polar runs use momentum \(0.95\), \(5\) Newton--Schulz steps, bfloat16 training, gradient clipping at norm \(1.0\), and weight decay \(0\). These runs use eight 910C NPUs (64GB). The polar optimizer is applied to the \(224\) query, key, value, attention-output, and MLP matrices summarized in Table~\ref{tab:llama0p6b_matrix_modules}, totaling \(536.87\)M matrix parameters~\citep{jordan6muon,liu2025muon}. The 600M comparison includes only polar variants. Table~\ref{tab:llama0p6b_lr_selected_details} reports the learning-rate grids, selected learning rates, and final validation losses, and Table~\ref{tab:llama0p6b_compute_diagnostics} reports the corresponding system diagnostics. The reported final losses and selected trajectories are averages over three random seeds.

For the LLaMA geometry diagnostics, given a saved momentum block \(\mathbf M_t\) and a representation \(T\), we form \(T(\mathbf M_t)\), compute its compact singular-value decomposition, and record the represented Frobenius norm, represented nuclear norm, stable rank, effective rank, polar rank, pulled-back polar direction, and number of singular channels needed to capture \(90\%\) of the nuclear mass. The native-relative diagnostics also record nuclear-norm ratios and angles between the pulled-back direction induced by \(T\) and the native pulled-back polar direction, following recent fixed-state diagnostics for Transformer optimization geometry~\citep{zhang2024whytransformersneedadam,dong2025hessianstructure,tomihari2025gradientheterogeneity}. The LLaMA2-130M diagnostics use all \(60\) polar-optimized matrix modules at \(0.52\)B, \(1.31\)B, \(1.84\)B, and \(2.60\)B training tokens. The LLaMA2-600M diagnostics use all \(224\) polar-optimized matrix modules at \(3.15\)B, \(6.29\)B, \(9.44\)B, and \(12.58\)B training tokens. The early-diagnostic proxy in Appendix~\ref{app:theory-aligned-early-diagnostics} uses the final LLaMA2-130M checkpoint and aggregates \(\langle \mathbf G_t,\mathbf D_T(\mathbf M_t)\rangle\) over the \(60\) matrix blocks for each selected learning-rate setting.

\begin{table}[!ht]
\centering
\small
\setlength{\tabcolsep}{9pt}
\renewcommand{\arraystretch}{1.08}
\begin{tabular}{lccc}
\hline
\textbf{Module group} & \textbf{Count} & \textbf{Native shape} & \textbf{Nearest-square shape} \\
\hline
\multicolumn{4}{l}{\textit{LLaMA2-130M (12 layers)}} \\
Attention QKV & \(12\) & \(2304\times768\) & \(1152\times1536\) \\
Attention output & \(12\) & \(768\times768\) & \(768\times768\) \\
MLP \(W_1\) & \(12\) & \(2048\times768\) & \(1024\times1536\) \\
MLP \(W_2\) & \(12\) & \(2048\times768\) & \(1024\times1536\) \\
MLP projection & \(12\) & \(768\times2048\) & \(1024\times1536\) \\
\hline
\multicolumn{4}{l}{\textit{LLaMA2-600M (32 layers)}} \\
Attention QKV & \(32\) & \(3072\times1024\) & \(1536\times2048\) \\
Attention output & \(32\) & \(1024\times1024\) & \(1024\times1024\) \\
MLP \(W_1\) & \(32\) & \(4096\times1024\) & \(2048\times2048\) \\
MLP \(W_2\) & \(32\) & \(4096\times1024\) & \(2048\times2048\) \\
MLP projection & \(32\) & \(1024\times4096\) & \(2048\times2048\) \\
\hline
\end{tabular}
\caption{Matrix-module shapes for the LLaMA2-130M and LLaMA2-600M experiments. The QKV row reports one layerwise group with concatenated dimensions; the \(224\)-matrix total for LLaMA2-600M counts the query, key, and value projections separately. Token embeddings, the language-model head, normalization parameters, and other non-matrix parameters are excluded from the polar update and optimized by AdamW.}
\label{tab:llama130m_matrix_modules}
\label{tab:llama0p6b_matrix_modules}
\end{table}

\begin{table}[!ht]
\centering
\small
\renewcommand{\arraystretch}{1.06}
\begin{tabular}{p{0.18\linewidth}p{0.45\linewidth}p{0.14\linewidth}p{0.13\linewidth}}
\hline
\textbf{Method} & \textbf{Learning-rate grid} & \textbf{Selected LR} & \textbf{Final val. loss} \\
\hline
AdamW & \(0.003,0.005,0.008\) & \(0.008\) & \(3.346\) \\
Native & \(0.002,0.003,0.005,0.0075,0.008,0.01,0.02\) & \(0.01\) & \(3.242\) \\
Square & \(0.002,0.003,0.005,0.0075,0.008,0.02\) & \(0.0075\) & \(3.237\) \\
Skinny-\(75\%\) & \(0.003,0.005,0.0075,0.008,0.01,0.015\) & \(0.008\) & \(3.269\) \\
Skinny-\(50\%\) & \(0.003,0.005,0.0075,0.008,0.01,0.015\) & \(0.008\) & \(3.289\) \\
Skinny-\(25\%\) & \(0.003,0.005,0.0075,0.008,0.01,0.015\) & \(0.01\) & \(3.293\) \\
Skinny-\(12.5\%\) & \(0.003,0.005,0.0075,0.008,0.01,0.015\) & \(0.015\) & \(3.307\) \\
Vector & \(0.005,0.008,0.01,0.015,0.02\) & \(0.015\) & \(3.402\) \\
\hline
\end{tabular}
\caption{Learning-rate selection for the \(2.60\)B-token LLaMA2-130M experiments. Within each method, the selected rate minimizes the mean final validation loss over three random seeds; the reported loss is the corresponding three-seed average.}
\label{tab:llama130m_lr_sweep_details}
\end{table}

\begin{table}[!ht]
\centering
\small
\renewcommand{\arraystretch}{1.06}
\begin{tabular}{p{0.18\linewidth}p{0.43\linewidth}p{0.14\linewidth}p{0.14\linewidth}}
\hline
\textbf{Method} & \textbf{Learning-rate grid} & \textbf{Selected LR} & \textbf{Final val. loss} \\
\hline
Native & \(0.001,0.0015,0.002,0.003,0.005,0.01\) & \(0.0015\) & \(2.668\) \\
Square & \(0.001,0.0015,0.002,0.003,0.005,0.01\) & \(0.0015\) & \(2.671\) \\
Skinny-\(75\%\) & \(0.001,0.002,0.005,0.01\) & \(0.001\) & \(2.685\) \\
Skinny-\(50\%\) & \(0.001,0.002,0.005,0.01\) & \(0.001\) & \(2.685\) \\
Skinny-\(25\%\) & \(0.001,0.002,0.005,0.01\) & \(0.001\) & \(2.710\) \\
Vector & \(0.0008,0.001,0.0015,0.002,0.003,0.05\) & \(0.003\) & \(3.080\) \\
\hline
\end{tabular}
\caption{Learning-rate selection for the \(12.58\)B-token LLaMA2-600M experiments. Within each polar method, the selected rate minimizes the mean final validation loss over three random seeds; the reported loss is the corresponding three-seed average.}
\label{tab:llama0p6b_lr_selected_details}
\end{table}

Tables~\ref{tab:llama130m_compute_diagnostics} and~\ref{tab:llama0p6b_compute_diagnostics} report training time, throughput, peak device memory, and a shape-dependent Newton--Schulz (NS) cost proxy for the selected learning-rate setting of each method. The proxy is summed over the polar-optimized matrix modules at the final checkpoint and normalized to Native. It therefore measures representation-dependent polar-backend cost rather than end-to-end runtime.

\begin{table}[!ht]
\centering
\small
\renewcommand{\arraystretch}{1.06}
\setlength{\tabcolsep}{8pt}
\begin{tabular}{lcccc}
\hline
\textbf{Method} & \textbf{Training time} & \textbf{Throughput} & \textbf{Peak memory} & \textbf{Relative NS cost} \\
& \textit{(h:mm:ss)} & \textit{(k tok/s)} & \textit{(alloc./res., GB)} & \textit{(Native = 1)} \\
\hline
AdamW & \(1{:}50{:}38\) & \(392.0\) & \(64.90/82.24\) & -- \\
Native & \(1{:}50{:}46\) & \(391.5\) & \(64.59/82.24\) & \(1.00\) \\
Square & \(1{:}50{:}45\) & \(391.5\) & \(64.59/82.24\) & \(1.51\) \\
Skinny-\(75\%\) & \(1{:}50{:}11\) & \(393.6\) & \(64.59/82.24\) & \(0.63\) \\
Skinny-\(50\%\) & \(1{:}50{:}05\) & \(393.9\) & \(64.59/82.24\) & \(0.44\) \\
Skinny-\(25\%\) & \(1{:}50{:}04\) & \(394.0\) & \(64.59/82.24\) & \(0.21\) \\
Skinny-\(12.5\%\) & \(1{:}49{:}39\) & \(395.5\) & \(64.59/82.24\) & \(0.10\) \\
Vector & \(1{:}48{:}10\) & \(400.9\) & \(64.59/82.24\) & \(0.0011\) \\
\hline
\end{tabular}
\caption{System diagnostics for the selected \(2.60\)B-token LLaMA2-130M runs. Time covers the full training budget, and peak memory reports allocated/reserved device memory. All polar methods use \(5\) Newton--Schulz steps; relative NS cost is normalized to Native.}
\label{tab:llama130m_compute_diagnostics}
\end{table}

\begin{table}[!ht]
\centering
\small
\renewcommand{\arraystretch}{1.06}
\setlength{\tabcolsep}{8pt}
\begin{tabular}{lcccc}
\hline
\textbf{Method} & \textbf{Time at budget} & \textbf{Throughput} & \textbf{Peak memory} & \textbf{Relative NS cost} \\
& \textit{(h)} & \textit{(k tok/s)} & \textit{(alloc./res., GB)} & \textit{(Native = 1)} \\
\hline
Native & \(19.15\) & \(182.5\) & \(53.14/56.07\) & \(1.00\) \\
Square & \(19.16\) & \(182.4\) & \(53.14/56.07\) & \(1.75\) \\
Skinny-\(75\%\) & \(19.17\) & \(182.3\) & \(53.14/56.07\) & \(0.50\) \\
Skinny-\(50\%\) & \(19.14\) & \(182.6\) & \(53.14/56.07\) & \(0.50\) \\
Skinny-\(25\%\) & \(19.18\) & \(182.2\) & \(53.14/56.21\) & \(0.25\) \\
Vector & \(18.84\) & \(185.5\) & \(53.14/56.21\) & \(0.0010\) \\
\hline
\end{tabular}
\caption{System diagnostics for the selected \(12.58\)B-token LLaMA2-600M runs. Throughput is averaged over a common late-training window, and normalized time extrapolates the corresponding mean update time to the full \(9600\)-step budget. Peak memory reports allocated/reserved device memory. All methods use \(5\) Newton--Schulz steps; relative NS cost is normalized to Native.}
\label{tab:llama0p6b_compute_diagnostics}
\end{table}

For the controlled teacher--student diagnostics, Figure~\ref{fig:teacher_student_case_study} uses a one-hidden-layer squared-logit teacher--student model with
\[
\begin{gathered}
\mathbf x\sim N(\mathbf0,\mathbf I),\\
\mathbf z=\mathbf V\operatorname{ReLU}(\mathbf W\mathbf x),\\
\mathbf z_\star=\mathbf V_\star\operatorname{ReLU}(\mathbf W_\star\mathbf x).
\end{gathered}
\]
Both \(\mathbf W\) and \(\mathbf V\) are trainable, and \(\mathbf V_\star\) is generated as a near-orthogonal output matrix with a dense perturbation. All representations share the data, initialization, mini-batches, momentum, learning-rate selection, and polar backend. The Hessian diagnostic partitions the \(\mathbf W\)-block by hidden unit and records \(R_i=\sum_{k\ne i}\|\widehat{\mathbf H}^{WW}_{ik}\|_{\mathrm{op}}\), \(\lambda_{\min}(\widehat{\mathbf H}^{WW}_{ii})\), the block lower endpoint \(\beta_1=\min_i\{\lambda_{\min}(\widehat{\mathbf H}^{WW}_{ii})-R_i\}\), and the coordinate lower envelope \(\alpha_1\). Figure~\ref{fig:teacher_student_representation_sorting} uses two trainable teacher--student protocols. The single-hidden-layer dimension sweeps use hidden widths \(m\in\{32,48,64,96\}\), input dimension \(n=2m\), class dimension \(C=8m\), \(1536\) samples, batch size \(256\), seeds \(0,1,2\), \(120\) stochastic steps, and learning rates in \(\{0.25,0.5,1,1.5,2,3,4,6,8,12,16,24,32,48,64,96\}\times10^{-3}\). The depth panels use the corresponding \(L\)-hidden-layer variant with hidden width \(24\), input dimension \(48\), \(C=192\), \(1536\) samples, batch size \(256\), seeds \(0,1,2\), and a \(5000\)-step constant-learning-rate sweep over \(\{2.5,5,10,25,50,75,100\}\times10^{-5}\). The main depth sweep reports \(L\in\{1,2,3,4,10\}\), with an additional \(10000\)-step ten-layer setting in Figure~\ref{fig:appendix_teacher_student_depth_robustness}. The early-dissipation diagnostic uses hidden width \(48\), input dimension \(192\), \(2048\) training samples, batch size \(256\), six seeds, three feature-covariance scenarios, and \(100\) stochastic steps. For each seed and scenario, the initial perturbation \(\mathbf E_0\) is reconstructed from the saved teacher--student initialization, and \(\|T(\mathbf E_0)\|_*/\|T(\mathbf E_0)\|_F^2\) is correlated with the normalized loss drops at steps \(10\) and \(20\).

The next two figures supplement the controlled representation-sorting experiment in Figure~\ref{fig:teacher_student_representation_sorting}.
They test whether the learning-rate calibration predicted by Corollary~\ref{cor:main-learning-rate-calibration} is specific to the largest single-hidden-layer setting or persists under additional dimensions and longer depth runs.
In all panels, the representation choice remains the only optimizer-side variable.

\begin{figure*}[!t]\centering
\subfloat[\(m=32,n=64,C=256\).]{\includegraphics[width=0.32\linewidth]{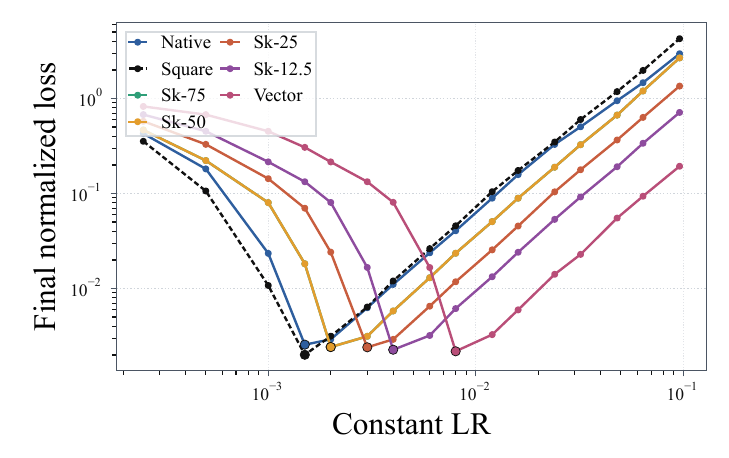}}\hfill
\subfloat[\(m=48,n=96,C=384\).]{\includegraphics[width=0.32\linewidth]{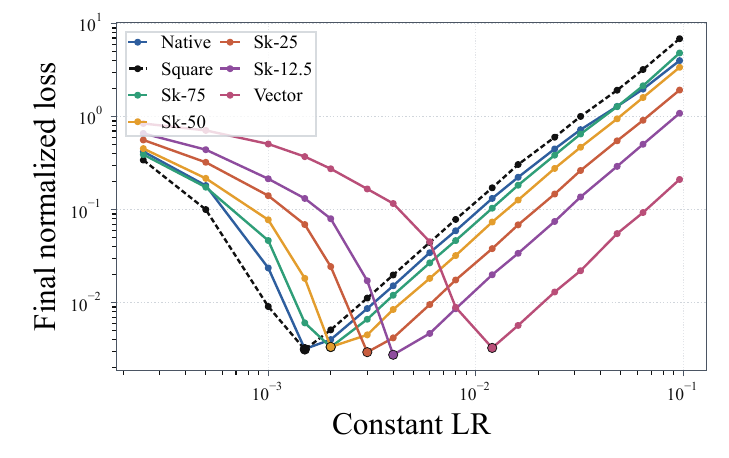}}\hfill
\subfloat[\(m=64,n=128,C=512\).]{\includegraphics[width=0.32\linewidth]{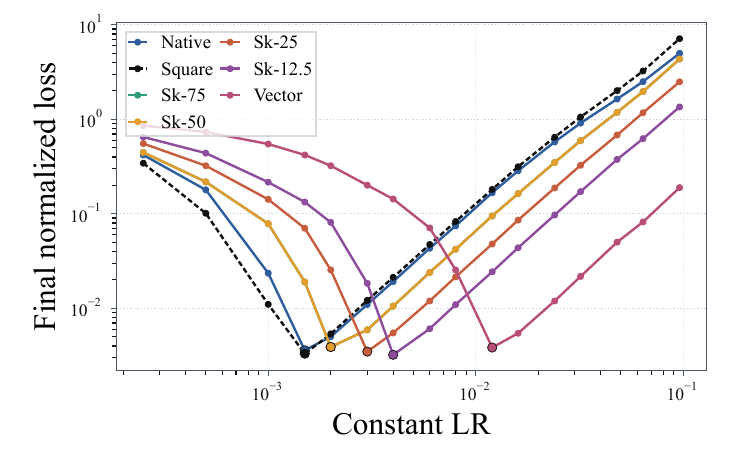}}
\caption{Additional single-hidden-layer dimension sweeps for Figure~\ref{fig:teacher_student_representation_sorting}. These panels use the same protocol as the main-text sweep and cover the remaining settings \(m\in\{32,48,64\}\), with \(n=2m\) and \(C=8m\). In these settings, the selected markers for narrower or vectorized representations occur at larger learning rates than the native and square markers.}
\label{fig:appendix_teacher_student_dimension_sweeps}
\end{figure*}

\begin{figure}[!t]\centering
\subfloat[Ten-layer \(10000\)-step curves.]{\includegraphics[width=0.34\linewidth]{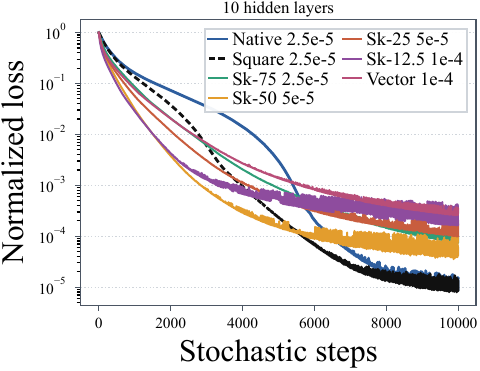}}\quad
\subfloat[Ten-layer constant-LR sweep.]{\includegraphics[width=0.34\linewidth]{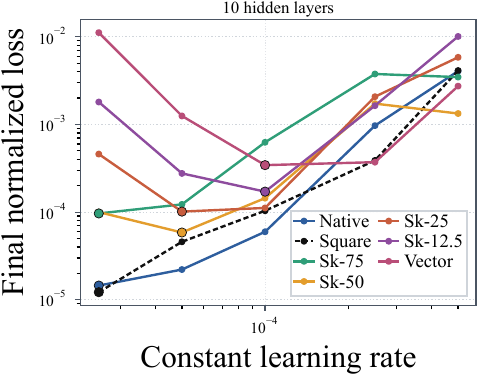}}
\caption{Ten-hidden-layer \(10000\)-step supplement to Figure~\ref{fig:teacher_student_representation_sorting}. Native and square matrixizations attain the lowest tuned final losses; narrower and vectorized representations select larger learning rates and remain at higher final loss.}
\label{fig:appendix_teacher_student_depth_robustness}
\end{figure}

Figures~\ref{fig:appendix_curvature_window_sweeps} and~\ref{fig:appendix_full_hessian_depth_maps} examine the curvature mechanism behind Theorem~\ref{thm:main-curvature-window-collapse} and Corollary~\ref{cor:main-fixed-mask-loss-equivalence}.
The finite-\((C,N)\) sweeps report the quantities that enter the curvature-window certificate: the normalized lower endpoint, the relative window width, and the off-block radius.
The depth maps provide a complementary diagnostic, showing how cross-layer and off-block curvature becomes more prominent beyond the one-hidden-layer setting used for Theorem~\ref{thm:main-curvature-window-collapse}.

\begin{figure}[!ht]
\centering
\subfloat[\(\beta_1/\beta_2\) versus \(N\).\label{fig:figS-curv-a}]{\includegraphics[width=0.34\linewidth]{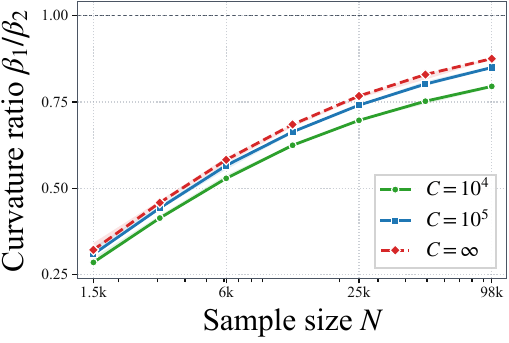}}
\quad
\subfloat[\(\rho_{C,N}\) versus \(N\).\label{fig:figS-curv-b}]{\includegraphics[width=0.34\linewidth]{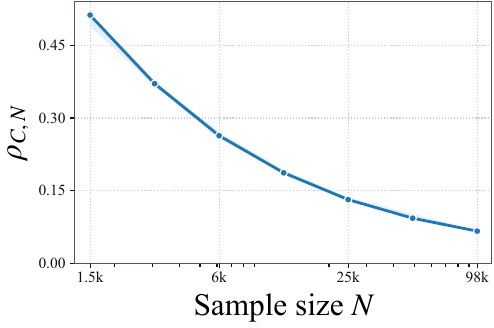}}
\\
\vspace{0.5em}
\subfloat[\(\max_i R_{C,N,i}\) versus \(C\).\label{fig:figS-curv-c}]{\includegraphics[width=0.34\linewidth]{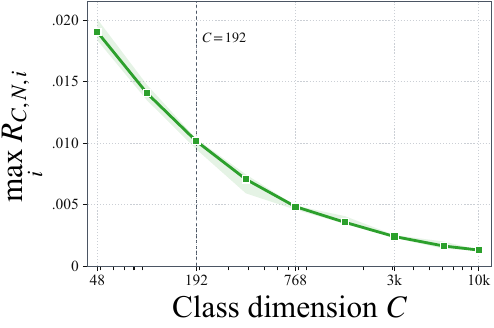}}
\quad
\subfloat[Signed \(\beta_1/\beta_2\) versus \(L\).\label{fig:figS-curv-d}]{\includegraphics[width=0.34\linewidth]{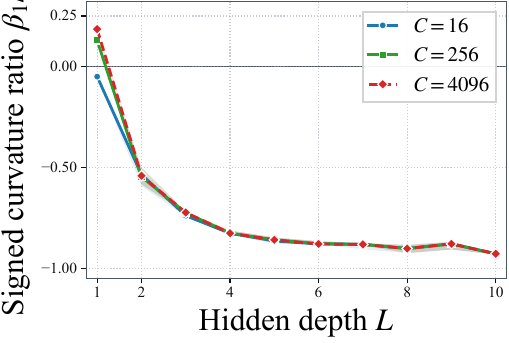}}
\caption{Finite-\((C,N)\) and depth diagnostics for the fixed-mask curvature window in Theorem~\ref{thm:main-curvature-window-collapse}. Panels (a)--(c) use the one-hidden-layer hidden-row diagnostic dimensions \((m,n)=(24,48)\): panel (a) plots the normalized lower endpoint, panel (b) plots the relative curvature-window width \(\rho_{C,N}=(\beta_{2,C,N}-\beta_{1,C,N})/(\beta_{1,C,N}+\beta_{2,C,N})\), and panel (c) plots the off-block radius \(R_{C,N,i}(\tau\mathbf E)=\sum_{k\ne i}\|\mathbf H^{WW}_{C,N,ik}(\tau\mathbf E)\|_{\mathrm{op}}\). Panel (d) uses a small fixed-mask depth diagnostic with \((m,n,N)=(8,16,256)\) and selected class dimensions \(C\in\{16,256,4096\}\).}
\label{fig:appendix_curvature_window_sweeps}
\end{figure}

\begin{figure}[!ht]
\centering
\subfloat[\(L=1\).\label{fig:figS-hess-d-L1}]{\includegraphics[width=0.35\linewidth]{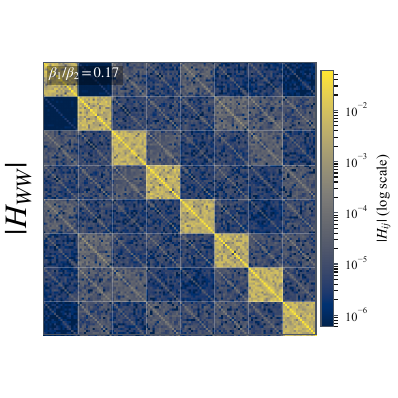}}
\quad
\subfloat[\(L=2\).\label{fig:figS-hess-d-L2}]{\includegraphics[width=0.35\linewidth]{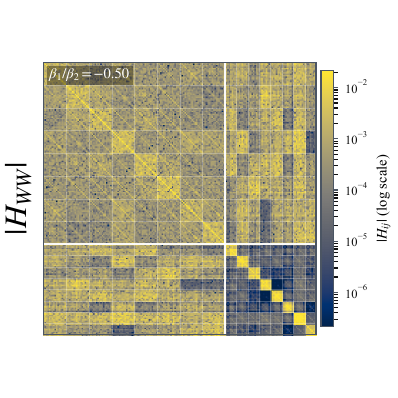}}
\\
\vspace{0.5em}
\subfloat[\(L=4\).\label{fig:figS-hess-d-L4}]{\includegraphics[width=0.35\linewidth]{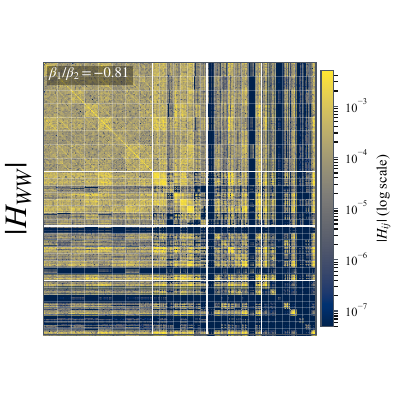}}\quad
\subfloat[\(L=10\).\label{fig:figS-hess-d-L10}]{\includegraphics[width=0.35\linewidth]{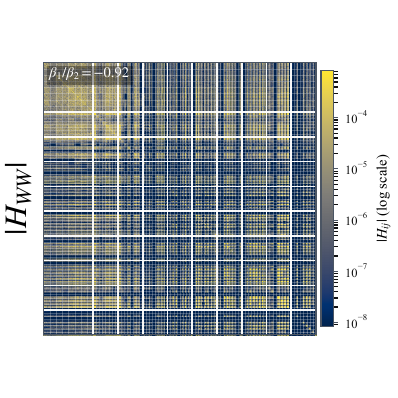}}
\caption{Full fixed-mask Hessian magnitude maps for the depth diagnostic at \(C=4096\). Each panel plots \(|H_{WW}|\) on a log color scale; white boundaries separate parameter-layer blocks, and the in-panel number reports the signed endpoint ratio \(\beta_1/\beta_2\). Increasing depth introduces dense cross-layer and off-block structure, matching the negative signed endpoints in Figure~\ref{fig:appendix_curvature_window_sweeps}(d).}
\label{fig:appendix_full_hessian_depth_maps}
\end{figure}

\section{Theory-Aligned Early Diagnostics}
\label{app:theory-aligned-early-diagnostics}

The main-text experiments rank selected final validation losses and fixed-state geometry.
This section reports early quantities aligned with the curvature-collapse dissipation mechanism in Theorem~\ref{thm:main-mp-representation-ordering}.
In the controlled teacher--student setting, reconstructing the saved initialization gives the initial normalized dissipation \(\|T(\mathbf E_0)\|_*/\|T(\mathbf E_0)\|_F^2\).
Across the three covariance scenarios, six seeds, and seven representations, this scalar strongly correlates with early normalized loss decrease: Pearson/Spearman correlations are \(0.90/0.87\) over the first \(10\) steps and \(0.92/0.95\) over the first \(20\) steps.
The \(1\)--\(100\)-step trajectories preserve the same ordering.
These measurements are the empirical analogue of the one-step dissipation term in Theorem~\ref{thm:main-mp-representation-ordering}; they therefore connect the teacher--student loss drops in Figure~\ref{fig:teacher_student_representation_sorting} to the represented nuclear-support quantity rather than only to final tuned loss.

For LLaMA2-130M, we use saved momentum states and a small-batch gradient diagnostic to evaluate the first-order proxy \(\langle \mathbf G_t,\mathbf D_T(\mathbf M_t)\rangle\).
At the final diagnostic checkpoint, this proxy, aggregated over the \(60\) matrix blocks and normalized to native, is \(1.00\) for native, \(1.12\) for Square, \(0.87\) for the selected skinny representation, and \(0.70\) for the vector endpoint.
The same ordering matches the LLaMA2-130M fixed-state geometry in Figure~\ref{fig:intro_representation_axis}(a) and the selected-run comparison in Figure~\ref{fig:llama_learning_dynamics}.

\section{Additional Diagnostics}
\label{app:additional-representation-geometry}

These trajectory-level and module-resolved diagnostics supplement the main-text representation comparison.
They use the representation-indexed polar geometry of Theorem~\ref{thm:main-fixed-representation-stationarity-nonselective}: the temporary matrixization determines the nuclear support seen by the polar map.
The LLaMA diagnostics use this represented nuclear support as a fixed-state proxy for the same representation axis.
Figure~\ref{fig:llama130m_layerwise_nuclear_norm} then tracks represented nuclear norm across training checkpoints.
Figure~\ref{fig:appendix_llama130m_module_nuclear_heatmaps} complements these layer averages with module-resolved heatmaps.
Together, these displays provide the layerwise, module-resolved, and checkpoint-resolved support profiles underlying the fixed-state diagnostics.
They complement the final-checkpoint scatter plots in Figure~\ref{fig:llama_module_geometry_scatter} by showing where the representation-dependent nuclear support appears across layers, modules, and training time.

\begin{figure}[!ht]
\centering
\subfloat[\(0.52\)B tokens.]{%
  \includegraphics[width=0.49\linewidth]{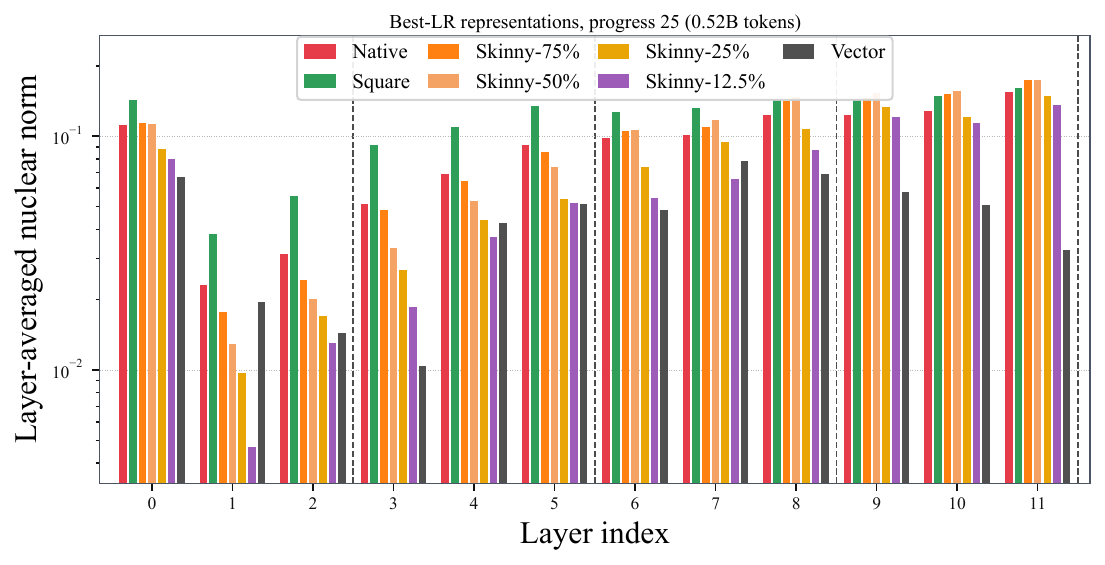}
}
\hfill
\subfloat[\(1.31\)B tokens.]{%
  \includegraphics[width=0.49\linewidth]{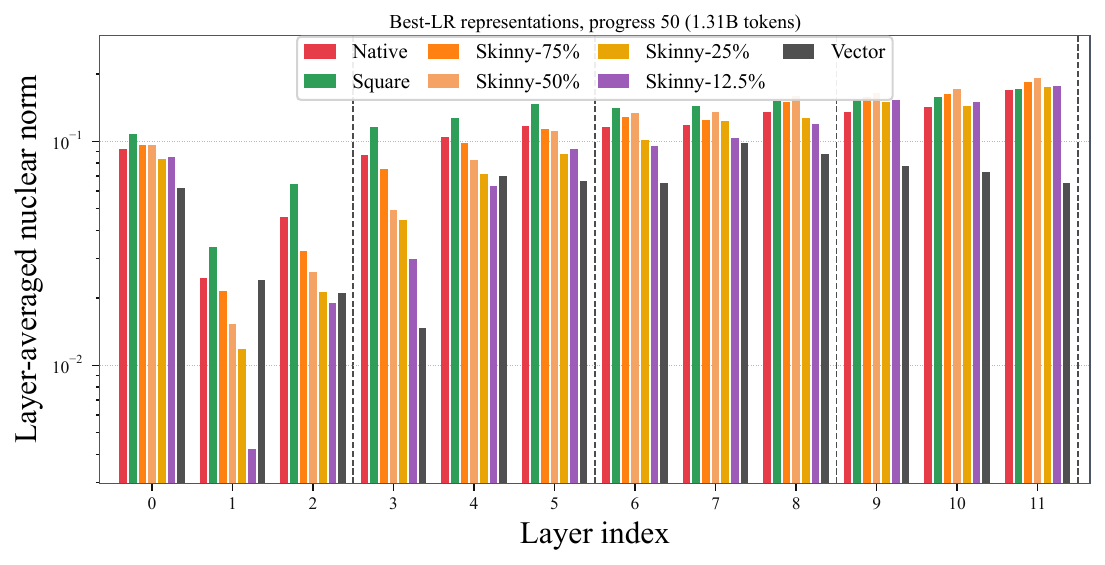}
}
\vspace{0.5em}
\subfloat[\(1.84\)B tokens.]{%
  \includegraphics[width=0.49\linewidth]{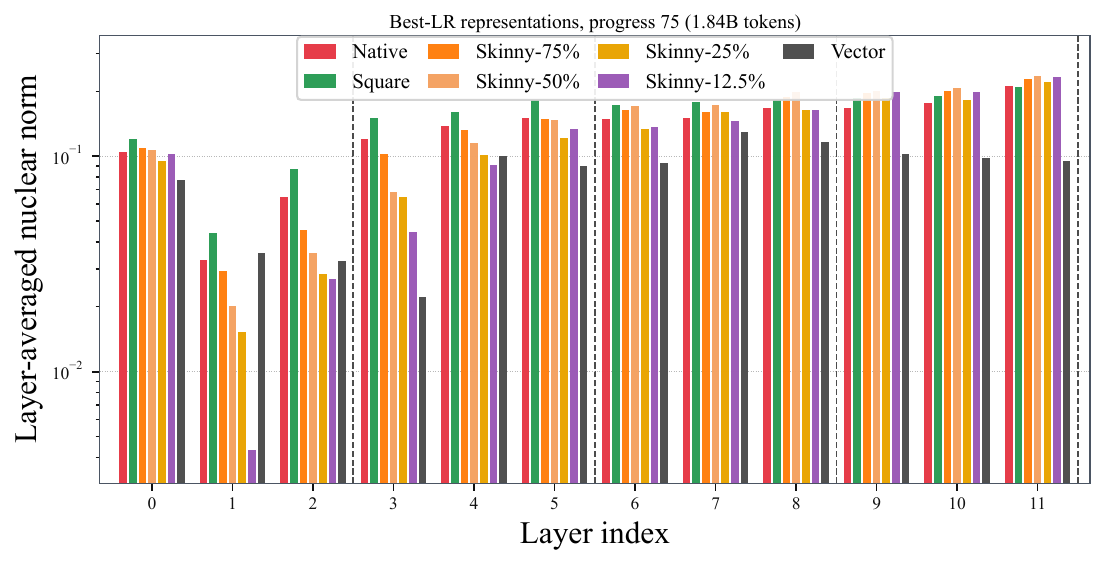}
}
\hfill
\subfloat[\(2.60\)B tokens.]{%
  \includegraphics[width=0.49\linewidth]{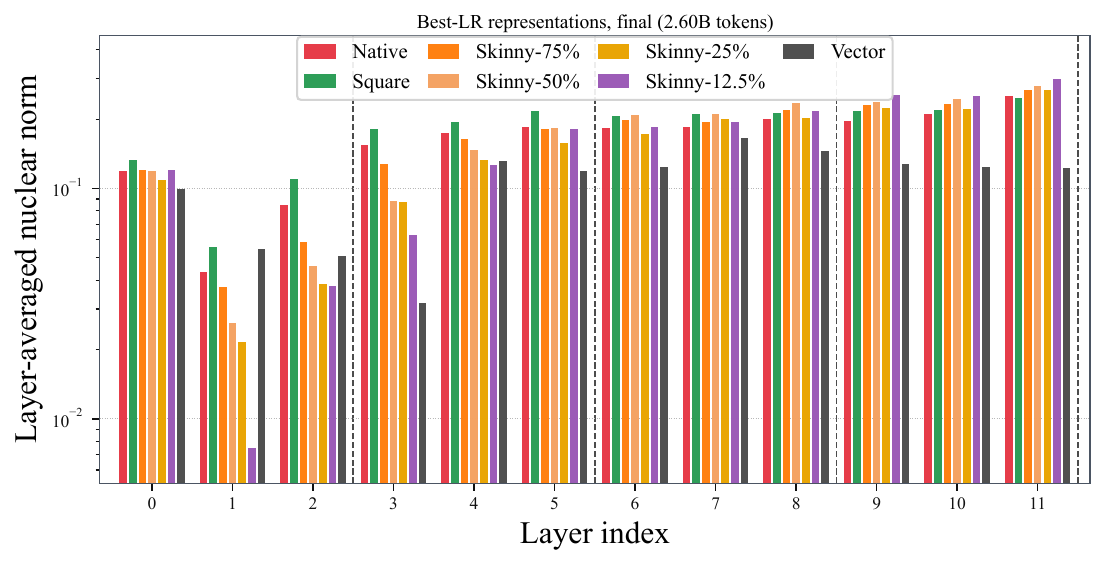}
}
\caption{Layerwise represented nuclear norm across checkpoints. Each bar is the parameter-count-weighted average over the five matrix modules in a layer for the selected learning-rate setting of the corresponding representation. All panels use a shared scale, so changes across rows and checkpoints are read as layerwise changes in represented nuclear support.}
\label{fig:llama130m_layerwise_nuclear_norm}
\end{figure}

Figure~\ref{fig:llama130m_layerwise_nuclear_norm} shows that the representation ordering is not confined to a single checkpoint.
Figure~\ref{fig:appendix_llama130m_module_nuclear_heatmaps} complements these layer averages with module-resolved heatmaps, identifying the attention and MLP blocks that contribute to the main-text module geometry in Figure~\ref{fig:llama130m_module_geometry_scatter}.
Together, the layerwise and module-resolved views support the fixed-momentum geometry interpretation used to read the training curves.

\begin{figure}[!ht]
\centering
\subfloat[\(0.52\)B tokens.]{%
  \includegraphics[width=0.49\linewidth]{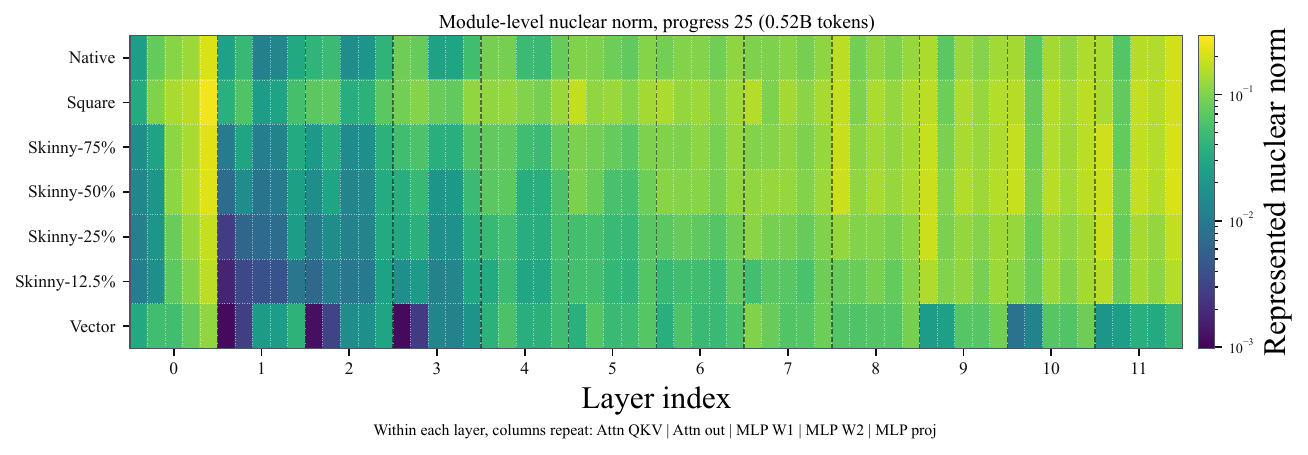}
}
\hfill
\subfloat[\(1.31\)B tokens.]{%
  \includegraphics[width=0.49\linewidth]{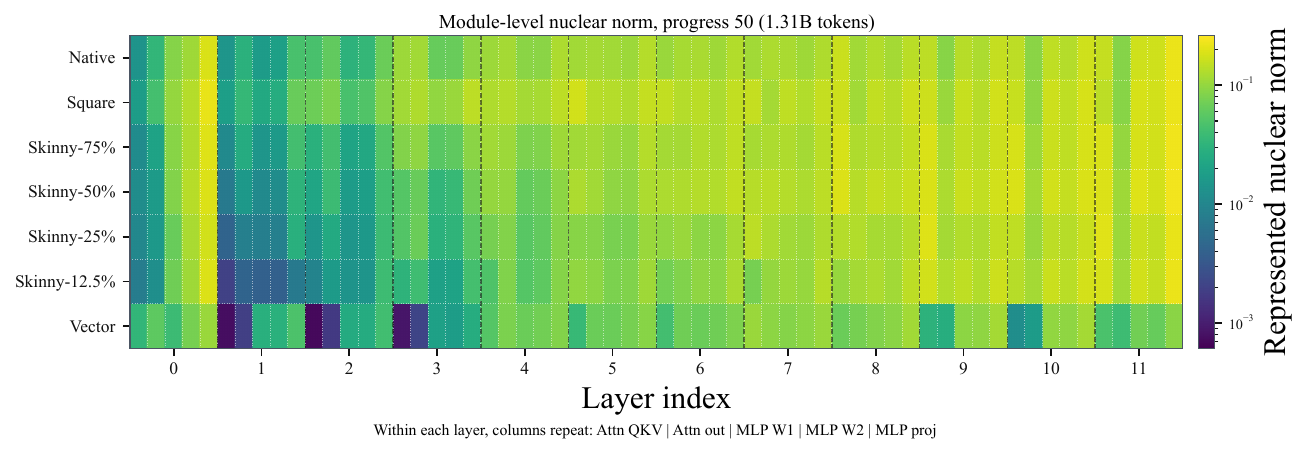}
}
\vspace{0.5em}
\subfloat[\(1.84\)B tokens.]{%
  \includegraphics[width=0.49\linewidth]{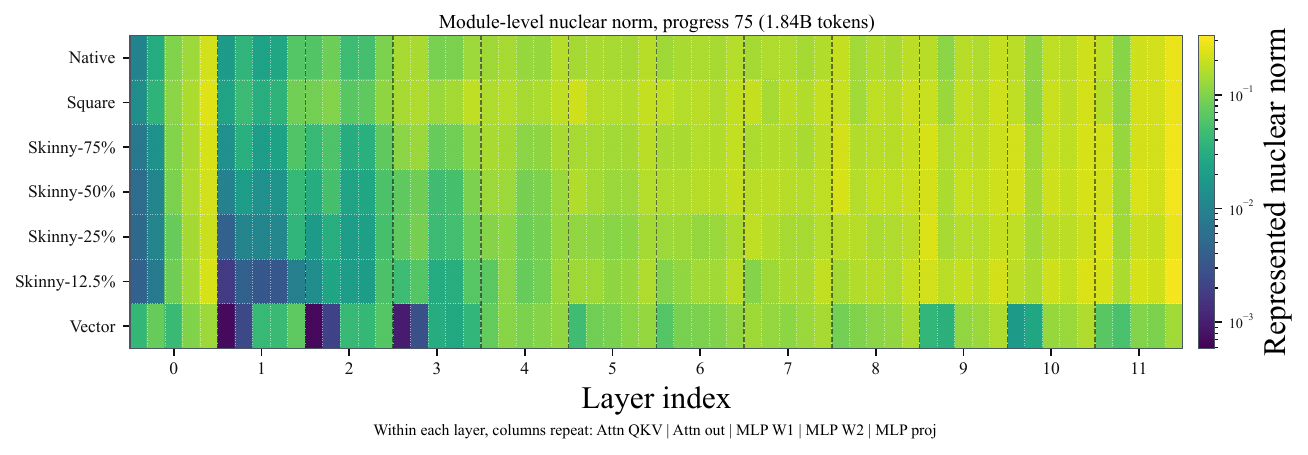}
}
\hfill
\subfloat[\(2.60\)B tokens.]{%
  \includegraphics[width=0.49\linewidth]{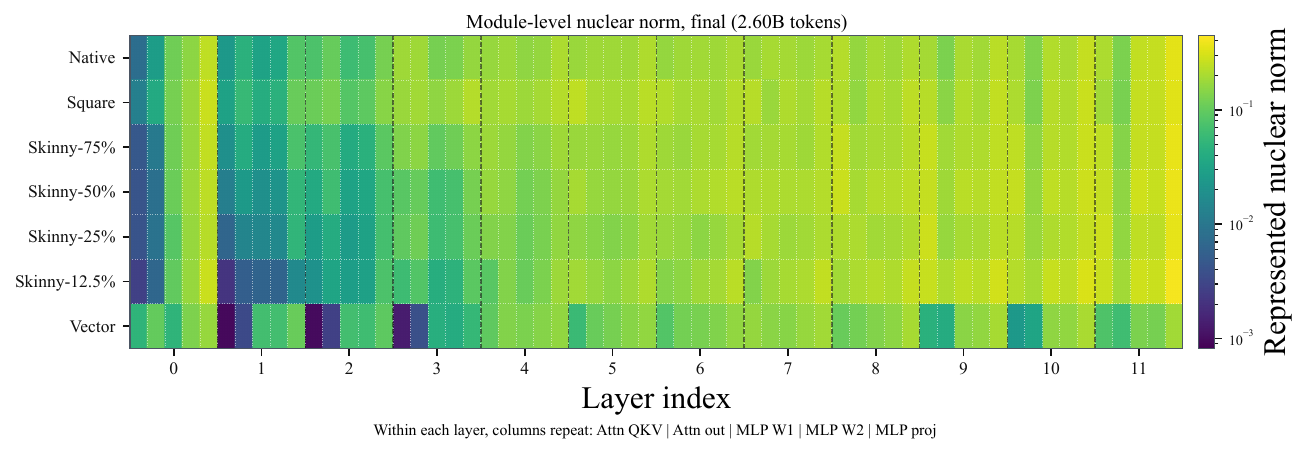}
}
\caption{Module-resolved supplement to Figure~\ref{fig:llama130m_layerwise_nuclear_norm}. Rows correspond to trained representations, and within each layer the columns repeat the five module types: attention QKV, attention output, MLP W1, MLP W2, and MLP projection. The heatmaps retain the module-level variation averaged within each layer in the main-text bars and localize the representation-dependent support changes.}
\label{fig:appendix_llama130m_module_nuclear_heatmaps}
\end{figure}

Figures~\ref{fig:appendix_llama600m_layerwise_nuclear_norm} and~\ref{fig:appendix_llama600m_module_nuclear_heatmaps} give the corresponding checkpoint-resolved layerwise and module-resolved support diagnostics at 600M.
They play the same role for the larger model as Figures~\ref{fig:llama130m_layerwise_nuclear_norm} and~\ref{fig:appendix_llama130m_module_nuclear_heatmaps} do for 130M: the selected learning-rate settings in Figure~\ref{fig:llama0p6b_learning_dynamics} are accompanied by the represented nuclear-support profiles of their optimizer states.

\begin{figure}[!ht]
\centering
\subfloat[\(3.15\)B tokens.]{%
  \includegraphics[width=0.49\linewidth]{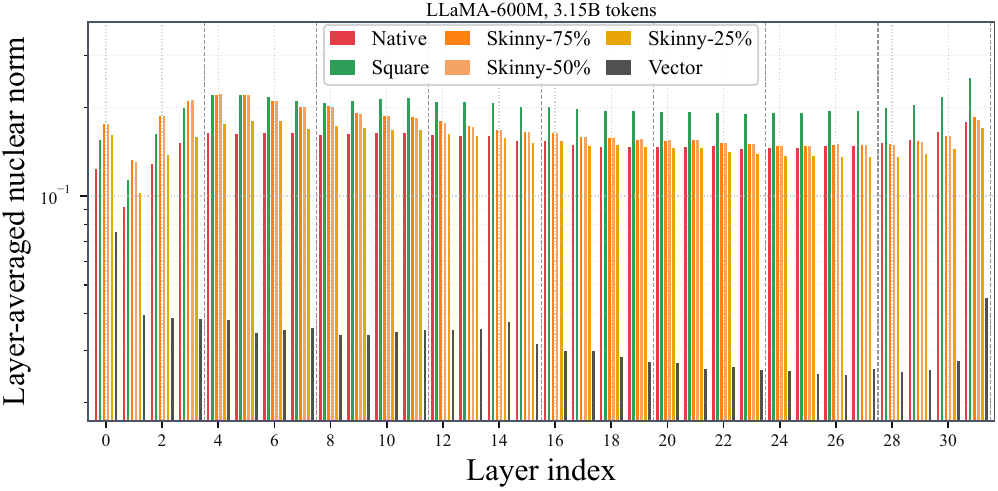}
}
\hfill
\subfloat[\(6.29\)B tokens.]{%
  \includegraphics[width=0.49\linewidth]{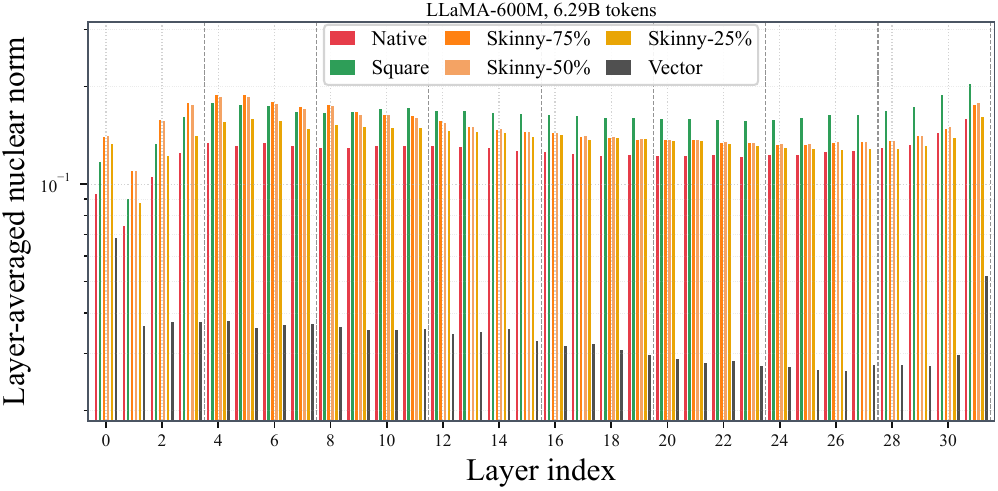}
}
\vspace{0.5em}
\subfloat[\(9.44\)B tokens.]{%
  \includegraphics[width=0.49\linewidth]{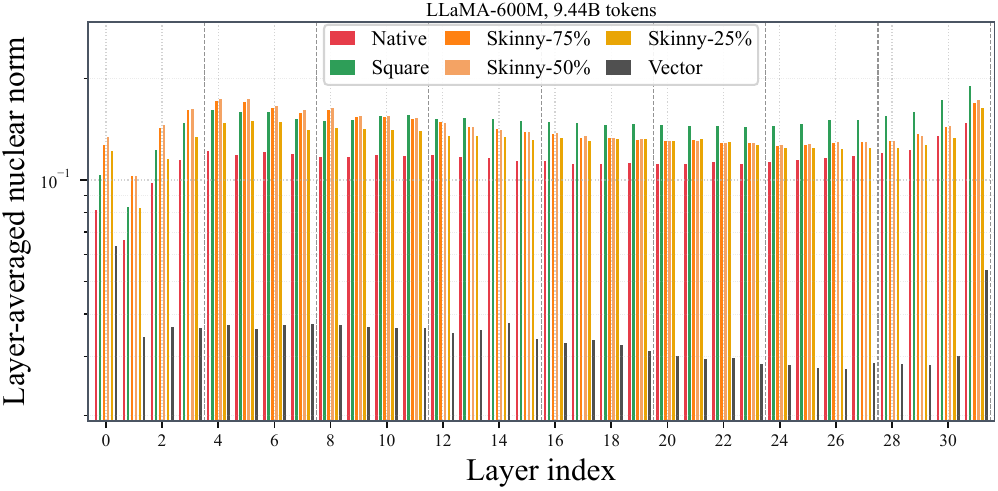}
}
\hfill
\subfloat[\(12.58\)B tokens.]{%
  \includegraphics[width=0.49\linewidth]{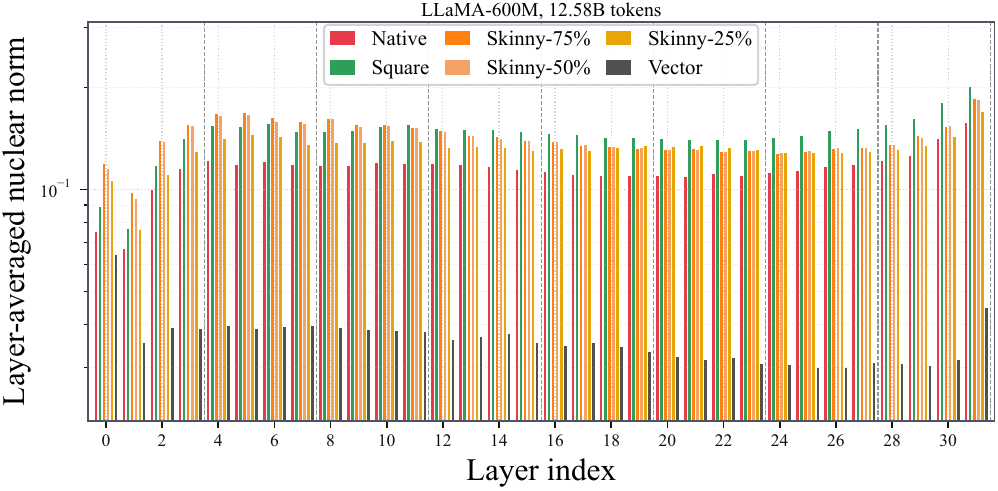}
}
\caption{Layerwise represented nuclear norm across LLaMA2-600M checkpoints. Each bar is the parameter-count-weighted average over the seven matrix module types in a layer for the selected learning-rate setting of the corresponding representation. All panels use a shared scale, so changes across rows and checkpoints are read as layerwise changes in represented nuclear support.}
\label{fig:appendix_llama600m_layerwise_nuclear_norm}
\end{figure}

\begin{figure}[!ht]
\centering
\subfloat[\(3.15\)B tokens.]{%
  \includegraphics[width=0.49\linewidth]{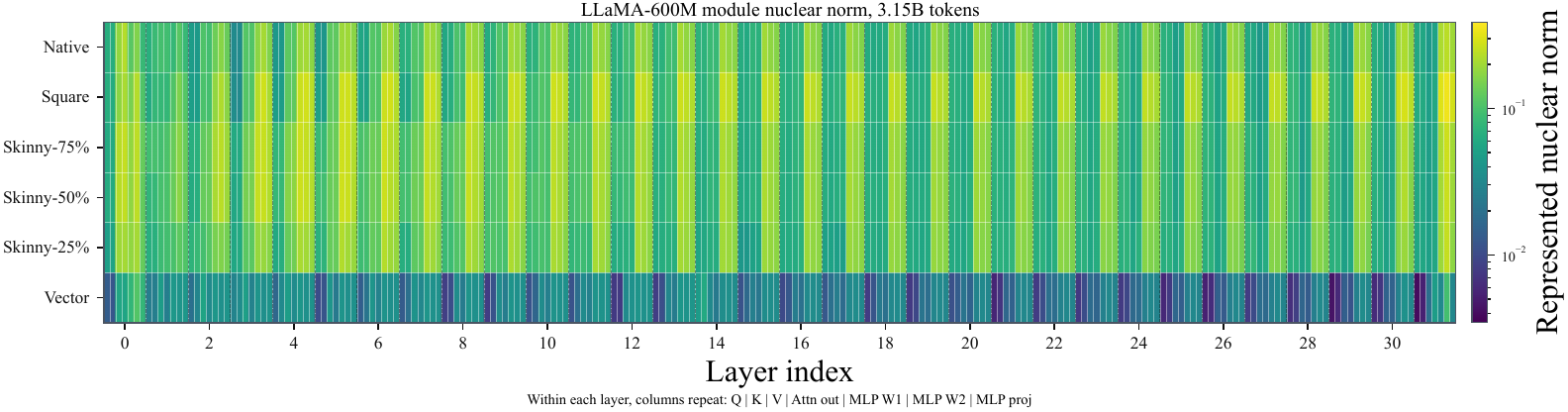}
}
\hfill
\subfloat[\(6.29\)B tokens.]{%
  \includegraphics[width=0.49\linewidth]{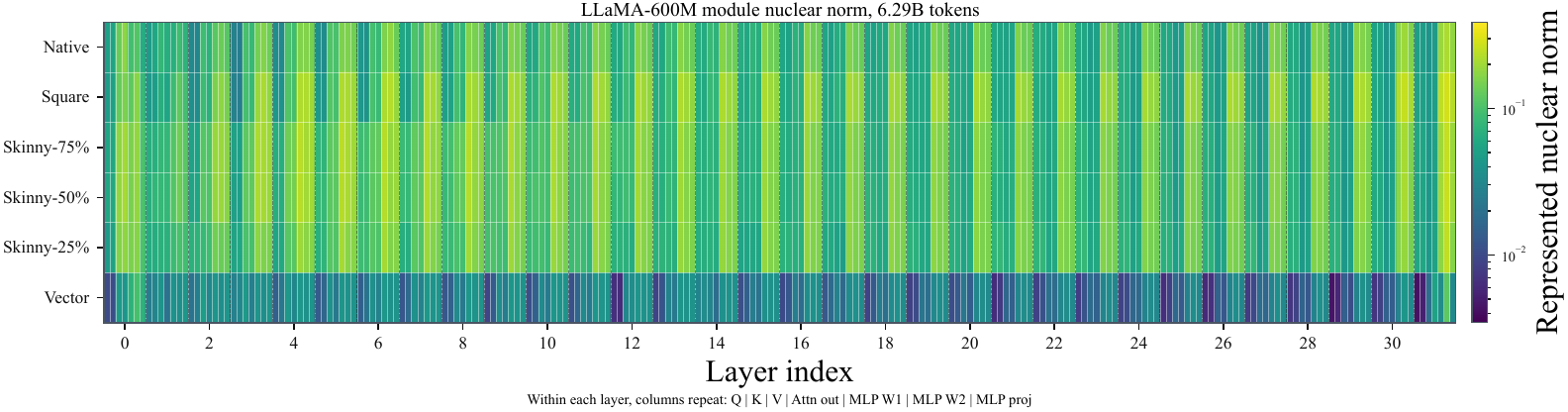}
}
\vspace{0.5em}
\subfloat[\(9.44\)B tokens.]{%
  \includegraphics[width=0.49\linewidth]{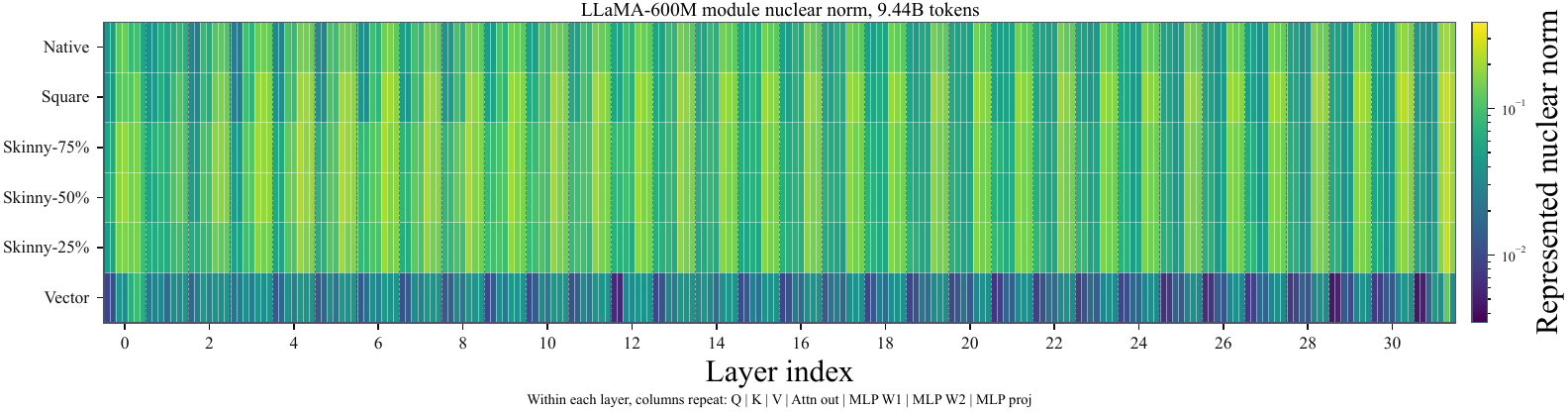}
}
\hfill
\subfloat[\(12.58\)B tokens.]{%
  \includegraphics[width=0.49\linewidth]{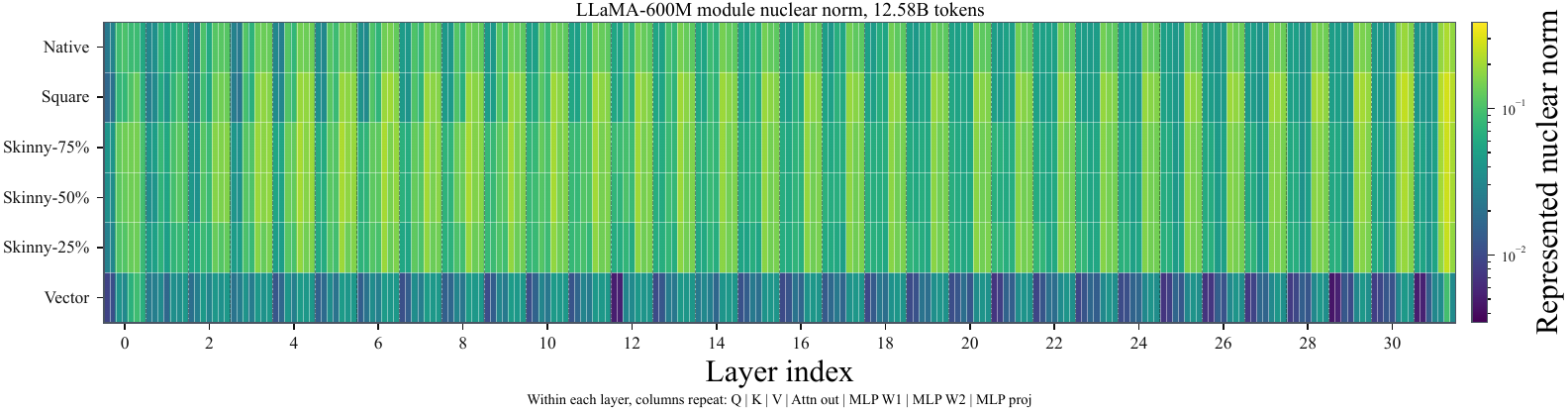}
}
\caption{Module-resolved supplement to Figure~\ref{fig:appendix_llama600m_layerwise_nuclear_norm}. Rows correspond to trained representations, and within each layer the columns repeat the seven module types: attention Q, attention K, attention V, attention output, MLP W1, MLP W2, and MLP projection. The heatmaps retain the module-level variation averaged within each layer in the appendix bars and localize the representation-dependent support changes at the 600M scale.}
\label{fig:appendix_llama600m_module_nuclear_heatmaps}
\end{figure}